\documentclass[12pt]{article}
\usepackage{amsmath}
\usepackage{amsfonts}
\usepackage{amsthm} 
\usepackage{amssymb}
\usepackage{dsfont}
\usepackage{algorithm}
\usepackage{algpseudocode}
\usepackage{mathtools}
\usepackage{natbib}
\usepackage{hyperref}
\usepackage{tabularx}
\usepackage{latexsym}
\usepackage{euscript,makeidx, color,mathrsfs}
\usepackage[margin=1in]{geometry}
\usepackage{bm}
\usepackage{booktabs}
\usepackage{array}
\usepackage{float}
\usepackage{enumitem}
\usepackage{xcolor}
\usepackage{caption}
\usepackage{subcaption}
\usepackage{indentfirst}
\usepackage{graphicx}
\usepackage{epstopdf}
\usepackage{multirow}
\usepackage{cleveref}
\usepackage[utf8]{inputenc}
\usepackage{textgreek}

\newtheorem{Theorem}{Theorem}[section]
\newtheorem{Proposition}[Theorem]{Proposition}
\newtheorem{Remark}[Theorem]{Remark}
\newtheorem{Lemma}[Theorem]{Lemma}
\newtheorem{Assumption}[Theorem]{Assumption}
\newtheorem{Corollary}[Theorem]{Corollary}
\newtheorem{Definition}[Theorem]{Definition}

\Crefname{Remark}{Remark}{Remarks}
\Crefname{Lemma}{Lemma}{Lemmas}
\crefname{Assumption}{assumption}{assumptions}
\Crefname{Assumption}{Assumption}{Assumptions}
\crefname{Corollary}{corollary}{corollaries}
\Crefname{Corollary}{Corollary}{Corollaries}

\renewcommand {\theequation}{\arabic{section}.\arabic{equation}}
\renewcommand{\theequation}{\arabic{section}.\arabic{equation}}

\algnewcommand{\algorithmicresult}{\textbf{Result:}}
\algnewcommand{\Result}{\item[\algorithmicresult]}

\newcommand{\E}{\mathbb{E}}

\newcommand{\Ptar}{P^{\mathrm{tar}}}
\newcommand{\Ppre}{P^{\mathrm{pre}}}

\newcommand{\Lip}{\mathrm{Lip}}
\newcommand{\CVaR}{\mathrm{CVaR}}
\newcommand{\VaR}{\mathrm{VaR}}
\newcommand{\VaRbar}{\overline{{\mathrm{VaR}}}}

\newcommand{\Qg}{\mathcal{Q}^g}
\newcommand{\R}{\mathbb{R}}
\newcommand{\LipKL}{D_{\mathrm{KL}}^{L}}
\newcommand{\QPtar}{(Q\|\Ptar)}
\newcommand{\Fg}{\mathcal{F}}
\newcommand{\PR}{\mathcal{P}(\R^d)}
\newcommand{\ind}{\mathds{1}}

\def\limsup{\mathop{\overline{\mathrm{lim}}}}
\def\liminf{\mathop{\underline{\mathrm {lim}}}}
\usepackage{authblk}

\usepackage{wrapfig}
\usepackage{longtable}
\usepackage{placeins}

\newcommand{\F}{\mathcal{F}}

\newcommand{\PropI}{\hyperref[P1]{\textbf{(I)}}}
\newcommand{\PropII}{\hyperref[P2]{\textbf{(II)}}}
\newcommand{\PropIII}{\hyperref[P3]{\textbf{(III)}}}

\newenvironment{restated}[1]%
  {\begin{trivlist}\item[\hskip\labelsep\bfseries #1]\itshape}%
  {\end{trivlist}}

\title{Fine-Tuning Generative Models for Extreme Events via CVaR-Penalized Wasserstein Gradient Flows
}

\author[1]{Thejani Gamage\thanks{ Email: \texttt{tgamage@umass.edu}}}
\author[1]{Hyemin Gu}
\author[1]{Zhizhen Zhang}
\author[2]{Ziyu Chen}
\author[1]{Markos A. Katsoulakis}
\author[1]{Luc Rey-Bellet}
\affil[1]{Department of Mathematics and Statistics, University of Massachusetts Amherst, Amherst, MA}
\affil[2]{School of Data and Information Sciences (SDIS) and Department of Mathematics, University of North Carolina at Chapel Hill, Chapel Hill, NC}

\date{}
\begin{document}

\maketitle
\addtocontents{toc}{\protect\setcounter{tocdepth}{-1}}
\begin{abstract}

In many high-stakes domains, extreme events carry substantial consequences, yet learning the heavy-tailed distributions that govern them from finite samples remains challenging: the quantities of interest are driven by a few extreme observations, so the tail region is under-sampled relative to its importance. Even generative models tailored for heavy tails capture the tail region inadequately in practice. We propose the Conditional Value-at-Risk (CVaR)-penalized Generative Particle Algorithm (CVaR-GPA), a tail-agnostic algorithm for fine-tuning generative models toward heavy-tailed targets, built as a time discretization of the Wasserstein gradient flow of the Lipschitz-regularized KL divergence penalized by a CVaR discrepancy term. Such a flow can be initialized from the output samples of the pre-trained model, which are then transported along the gradient descent of the loss functional, without requiring access to the pre-trained model's internal architecture. The Lipschitz-regularized KL divergence requires minimal assumptions on the target, while the CVaR penalty focuses the flow on the tail. The CVaR penalty depends on the target only through a scalar tail statistic, inducing a velocity field that remains active in the under-sampled tail region at a dimension-free estimation cost. The resulting velocity field, and hence the depth of the transport map, implicitly adapts to the target distribution, without any target-specific modifications. Across four targets, including two real-world, high-dimensional datasets (daily streamflow in the Ohio River basin ($d=64$) and the Fama–French portfolios ($d=25$) with tail indices ranging from $1.05$ to $3.34$), fine-tuning with CVaR-GPA reduces global and tail errors by geometric-mean factors of $14.0 \times$ and $9.8 \times$, respectively, across seven pre-trained models spanning GANs, diffusion models, and other generative flows, with a single set of hyperparameters.

\end{abstract}
\section{Introduction}

Heavy-tailed distributions arise in a wide range of application domains, including finance and insurance
\citep{albrecher2006ruin,embrechts1982estimates,grossi2005catastrophe},
climate science \citep{farazmand2019closed, merz2022understanding}, and biology
\citep{cirillo2020tail,lynn2024heavy}, where accurate simulations of extreme events are essential.
Under heavy tails, the quantities of interest are driven by a few extreme observations \citep{embrechts1997modelling,foss2013introduction,nair2022fundamentals}: tail events are rare but not negligible. This makes learning such distributions from finite samples statistically challenging: only a small fraction of the samples falls in the tail region, which is therefore under-sampled relative to its importance. Standard objectives for learning distributions weigh all samples equally and are thus dominated by the bulk of the distribution, making them ill-suited for heavy-tailed targets. Despite their ability to simulate complex target distributions, standard generative models, which typically consist of globally Lipschitz continuous transport maps and light-tailed source distributions, necessarily output a light-tailed distribution \citep{jaini2020tails}, and thus cannot capture heavy-tailed targets accurately. 

The necessity and the challenges of learning heavy-tailed targets have motivated a growing class of generative models designed specifically for heavy-tailed targets, including extreme-value GANs \citep{Allouche2022EV-GAN,allouche2026exceedgan,bhatia2021exgan}, Pareto GAN \citep{Chiang2021ParetoGAN}, Tail-GAN \citep{cont2026tail}, and generative flows tailored for heavy-tailed targets \citep{guan2025mirrorflowmatchingheavytailed,hickling2024flexible,laszkiewicz2022marginal,pandey2025_heavytaildiffusion}. However, the statistical challenges inherent to learning heavy-tailed targets persist in these models, which capture the tails inadequately in practice. Models with fixed, target-independent architectural depths \citep{Allouche2022EV-GAN,allouche2026exceedgan,cont2026tail} may not be deep enough to reach the extreme regions of the target distribution. Although the depth of the model proposed in \cite{gu2024lipschitz} adapts to the target by construction, due to the scarcity of samples in the tail region,  its training terminates in practice before the extreme regions are adequately captured. Moreover, many models tailored for heavy-tailed targets require prior knowledge or accurate estimation of the target tail-decay rates \citep{guan2025mirrorflowmatchingheavytailed,Chiang2021ParetoGAN, hickling2024flexible, laszkiewicz2022marginal,pandey2025_heavytaildiffusion}, which itself is a challenging finite-sample statistical problem, again due to the under-sampling of the tail region. 
More recently, \cite{chang2026extreme} incorporate extreme-event statistics absent from the training data through a Wasserstein-1 penalty toward a reference law of a scalar observable, either estimated from a separate data record or prescribed as a hypothesized extreme-value distribution. % our approach instead matches a single tail statistic estimated from target samples, without specifying a reference law.

%More recently, \cite{chang2026extreme} propose an objective for capturing heavy-tailed behavior without requiring training samples from the tail. The method relies critically on the choice of the reference measure in the Wasserstein-1 term, which in turn requires some prior knowledge about the target. 

In this paper, we propose a fine-tuning flow and develop its corresponding particle algorithm, \textit{CVaR-GPA}, the \textit{Conditional Value-at-Risk (CVaR)-penalized Generative Particle Algorithm}, which mitigates these drawbacks to better learn heavy-tailed distributions. Rather than designing and training yet another architecture, we retain the capabilities of existing pre-trained models tailored for heavy-tailed distributions, and correct their output samples where necessary to overcome their limitations. Among generative frameworks, Wasserstein gradient flows \citep{jordan1998variational,otto2001geometry}, which describe the gradient descent of a free energy functional $\F$ on the space of probability measures, are particularly suitable for fine-tuning. Since the velocity field of a Wasserstein gradient flow is determined by the variational derivative of $\F$, desired properties can be embedded into the flow through the choice of $\F$. To fine-tune existing models toward heavy-tailed targets without accessing or modifying their internal architecture, we require the fine-tuning flow to have the following properties:
\begin{itemize} \setlength{\topsep}{1pt}\setlength{\partopsep}{1pt}\setlength{\itemsep}{1pt}\setlength{\parskip}{1pt}
\item[\textbf{(I)}]\phantomsection\label{P1} The fine-tuning flow can be initialized from any empirical distribution supported on finite samples, in particular from the output samples of a pre-trained model. Thus, fine-tuning requires neither access to the pre-trained model's internal architecture nor any modifications specific to the pre-trained model.
\item[\textbf{(II)}]\phantomsection\label{P2} The velocity field of the fine-tuning flow places additional weight on the tail region and remains active there despite its under-sampling, even for high-dimensional targets. 
\item[\textbf{(III)}]\phantomsection\label{P3} The velocity field of the fine-tuning flow is target-adaptive, and hence so is the terminal time of the flow, which determines the depth of the fine-tuning transport map. Thus, the fine-tuning flow self-regulates across targets without requiring target-specific modifications.
\end{itemize}
\vspace{-4pt}

Let $\Ppre$ and $\Ptar$ denote the output distribution of a pre-trained model and the target distribution. We design $\F$ as the Lipschitz-regularized KL divergence \citep{dupuis2022formulation} penalized by a smoothed Conditional Value-at-Risk (CVaR) discrepancy term \citep{kouri2016risk} given by 
\begin{align}
\label{Eq:penalized:loss:functional}
\F(Q; \Ptar)= \LipKL  \QPtar + \lambda \, \left(\Delta C_h (Q; \Ptar)\right)^2,
\end{align}
where $\Delta C_h (Q; \Ptar):=\mathrm{CVaR}_{\alpha,h}^{\Ptar,g} -\mathrm{CVaR}_{\alpha,h}^{ Q,g}$ denotes the 
smoothed CVaR discrepancy  with smoothing parameter $h>0$  of $g(X)$ between $X \sim Q$ and $X \sim \Ptar$, at level $\alpha \in(0,1)$,  for a given risk function $g:\R^d \rightarrow [0,\infty)$, $\LipKL$ denotes the Lipschitz-regularized KL divergence with Lipschitz constant $L$, and $\lambda > 0$ is
a hyperparameter controlling the relative weight of the CVaR penalization. A divergence term is needed in $\F$ so that $\F$ identifies $\Ptar$ as its unique minimizer. We choose the Lipschitz-regularized KL divergence over standard choices such as the KL divergence or the Wasserstein-1 distance since, unlike those, both the  Lipschitz-regularized KL divergence and its variational derivative exist finitely at the population level and at the finite-sample level, with no additional assumptions on the target \citep{gu2024lipschitz,chen2025robust}. The smoothed CVaR and its variational derivative likewise exist without any density assumptions. Thus, both $\F$ and its variational derivative are well defined for empirical distributions supported on finite samples, which allows us to \textit{initialize the fine-tuning flow using samples from $\Ppre$}, satisfying property \PropI.

To place additional weight on the tail region, we use the smoothed CVaR at level $\alpha$ of the radial function $g(x)=\|x\|$, which is the average radius over the most extreme $(1-\alpha)$-fraction of the samples, and thus by definition focuses on the tail region. Moreover, the CVaR penalty depends on the target only through the single scalar statistic $\mathrm{CVaR}_{\alpha,h}^{\Ptar,g}$, whose finite-sample estimation error will be shown to be independent of the dimension $d$ of the target distribution, so the number of extreme target samples the CVaR penalty requires does not grow with $d$. We derive the velocity field of the fine-tuning flow and show that the CVaR penalty contributes a velocity component that \textit{focuses on the tail region}, mitigating its under-sampling, even for high-dimensional heavy-tailed targets, satisfying property \PropII. Furthermore, the resulting velocity field of the fine-tuning flow is target-adaptive, thus satisfying property \PropIII. In particular, the velocity component from the CVaR penalization is determined by the CVaR discrepancy $\Delta C_h (Q; \Ptar)$, and remains active whenever the discrepancy is non-zero. Thus, the magnitude of the velocity field, as well as the terminal time of the fine-tuning flow, is implicitly determined by the target.

We evaluate CVaR-GPA by fine-tuning seven pre-trained generative models: Lip-KL-GPA \citep{chen2025robust}, Tail Transform Flow (TTF) \citep{hickling2024flexible}, Tail-GAN \citep{cont2026tail}, t-EDM \citep{pandey2025_heavytaildiffusion}, Score-based Heavy-tailed Diffusion (SHD) \citep{liu2024_heavytaileddiffusion}, Denoising L\'evy Probabilistic Models (DLPM) \citep{shariatian2024denoisinglevyprobabilisticmodels}, and marginal Tail-Adaptive Flows (mTAF) \citep{laszkiewicz2022marginal}. Together they span normalizing flows, GANs, diffusion and score-based approaches as well as Wasserstein gradient flows. \Cref{fig:streamflow} shows the global and tail errors of the seven pre-trained models, before and after fine-tuning using CVaR-GPA, on the daily streamflow in the Ohio River basin ($d=64$) \citep{usgs_nwis}, an extremely heavy-tailed target whose Hill tail indices \citep{hill1975} range from $1.05$ to $3.34$, with $51$ of the $64$ marginals in the infinite-variance regime. Both errors are reduced for every model after fine-tuning, by geometric-mean factors of $9.5\times$ globally and $5.4\times$ in the tail. Across four targets and seven pre-trained models, CVaR-GPA reduces errors in all but one of the $56$ dataset--model--error metric combinations (\Cref{tab:headline:summary}), by geometric-mean factors of $14.0\times$ globally and $9.8\times$ in the tail, with a single set of hyperparameters and without any model-specific or target-specific tuning. These results empirically demonstrate how CVaR-GPA improves the accuracy of the pre-trained models and makes the fine-tuning robust across various models and datasets.
\begin{figure}[H]
\centering
\includegraphics[width=0.62\linewidth]{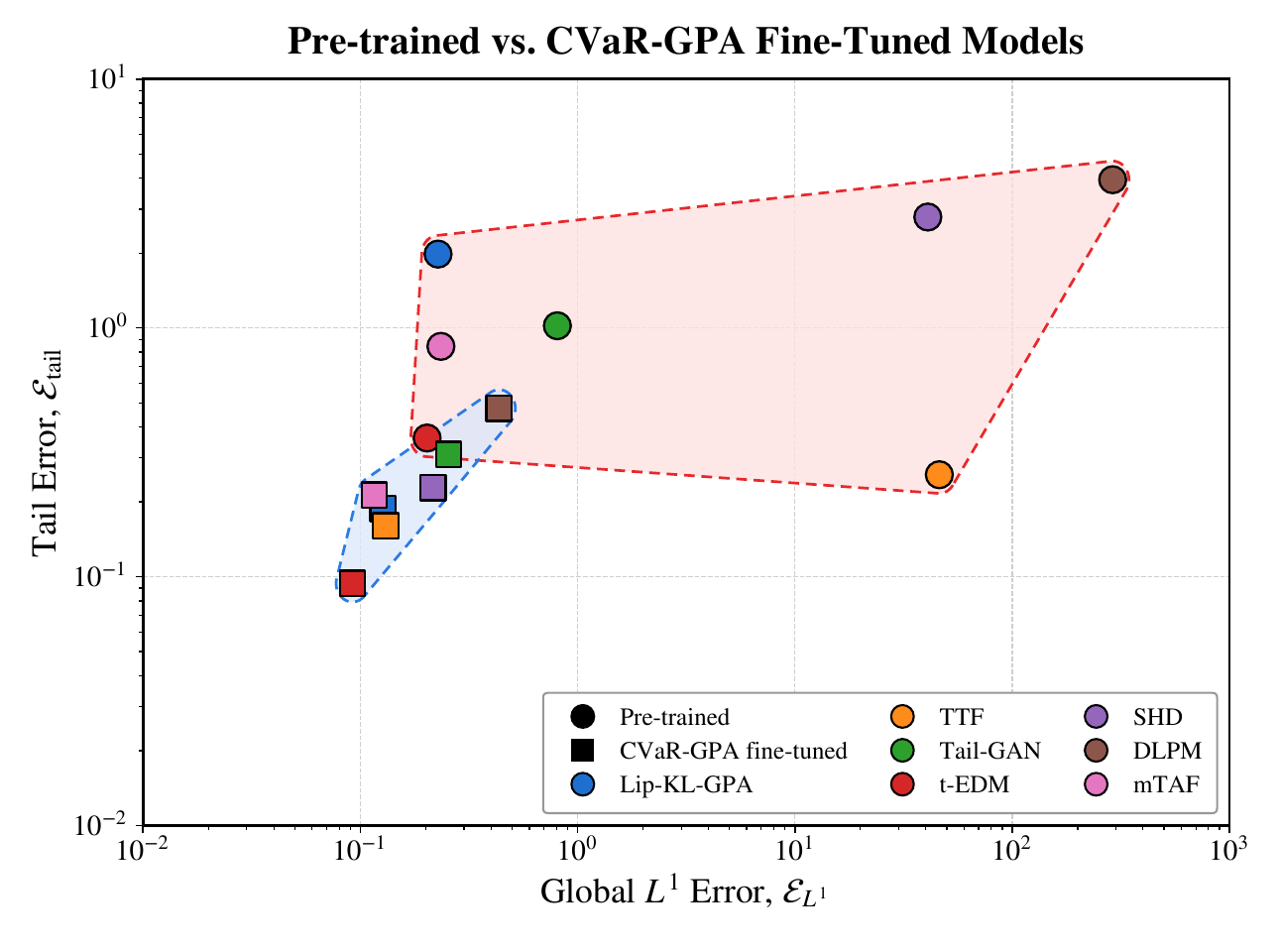}
\caption{
Global $L^1$ and tail errors of seven pre-trained models before (circles) and after (squares) fine-tuning via CVaR-GPA, on the daily streamflow in the Ohio River basin ($d=64$) \citep{usgs_nwis} on a logarithmic axis. The models and the error metrics are given in \Cref{Sec:Results}.
}
\label{fig:streamflow}
\end{figure}

CVaR is one instance of a broader family of spectral risk measures \citep{acerbi2002portfolio} that assign
greater weight to the tail region. CVaR is a canonical choice in tail-sensitive modeling, with an extensive literature
\citep{acerbi2001expected, cont2026tail, kishida2023risk}, owing largely to its
tractable variational representation via the Rockafellar--Uryasev formula \citep{Rock2002CVaR}. Tail-GAN \citep{cont2026tail}, an adversarial generator for multi-asset financial returns, augments the GAN discriminator loss with Value-at-Risk (VaR) and CVaR, so that the generator learns to
reproduce correct tail-risk statistics for benchmark financial portfolios. Unlike our method, however, its generator is a fixed-depth Lipschitz continuous transport map, which may structurally limit its ability to learn a
heavy-tailed target distribution. Furthermore, the main goal of Tail-GAN is to reproduce tail statistics for financial portfolios, as opposed to our goal, which is learning the heavy-tailed target and capturing extreme events.

A separate line of work incorporates CVaR into fine-tuning objectives for tail-sensitive reward maximization
\citep{chaudhary2024risk, NEURIPS2025_10715deb,wang2026efficient}. This line of research pursues a fundamentally
different goal from ours: it reweights the generated distribution to favor extreme rewards, whereas we aim to
learn heavy-tailed distributions and capture extreme events. The framework of \cite{wang2026efficient} is
closest to ours, where CVaR, its Rockafellar--Uryasev representation, and its first variational derivative are
employed for tail-sensitive reward maximization. However, the variational derivative of CVaR utilized in
\cite{NEURIPS2025_10715deb, wang2026efficient} is computed under density assumptions that do not hold for
general distributions such as empirical measures supported on finite samples. This necessitates either the derivation
of the first variational subgradients of $\mathrm{CVaR}_\alpha^{ Q,g}$, or {the use of a} smoothed CVaR, whose variational derivative exists for  empirical measures supported on finite samples.

The rest of the paper is organized as follows. 
\Cref{Sec:CVaR:Flow} introduces the CVaR-penalized loss functional and discusses how the induced Wasserstein gradient flow and its velocity field realize \PropI--\PropIII. \Cref{Sec:CGPA} develops the corresponding particle algorithm CVaR-GPA, and \Cref{Sec:Results} evaluates CVaR-GPA across four heavy-tailed targets (synthetic and real-world) and seven pre-trained generative models. \Cref{Sec:Conclusions} concludes the paper.

\section{Target-adaptive fine-tuning flow}
\label{Sec:CVaR:Flow}
The loss functional $\F$ in \eqref{Eq:penalized:loss:functional} comprises two main components: the Lipschitz-regularized KL divergence and the smoothed CVaR. 
\paragraph{Lipschitz-regularized KL divergence} Introduced in \cite{dupuis2022formulation} and subsequently
developed in \cite{Birrell2022fGamma} as a computational tool for generative modeling, the Lipschitz-regularized KL divergence is defined via the infimal convolution:
\begin{align}
\label{Eq:Lip:KL:Def}
\LipKL  \QPtar = \inf_{\gamma\in \mathcal{P}(\mathbb{R}^d)} \left\{ D_{KL}(\gamma \| \Ptar) + L \cdot \mathcal{W}_1(Q, \gamma) \right\},
\end{align}

where $D_{KL}$ and $\mathcal{W}_1$ denote the KL divergence and the Wasserstein-$1$ metric, respectively. By
definition, the minimizer $\gamma^*$ of \eqref{Eq:Lip:KL:Def} can be viewed as a reweighting of the target
measure $\Ptar$ through the KL divergence, while the Wasserstein-$1$ metric measures the cost of transporting
mass from the intermediate measure $\gamma^*$ to $Q$ (see \cite[Section 3]{Birrell2022fGamma} for a detailed discussion of the mass
redistribution/transport interpretation).  Furthermore, from \cite{dupuis2022formulation}, the Lipschitz-regularized KL divergence  defined by \eqref{Eq:Lip:KL:Def} has a variational representation:
\begin{align}
\label{equation:Variational_KL_Lip}
 D_{\mathrm{KL}}^L (Q \| P^{\mathrm{tar}})   = \sup_{\phi \in \Gamma_L} \{\mathbb{E}_Q[\phi] - \log(\mathbb{E}_{\Ptar}[e^{\phi}]) \},
\end{align}
where $\Gamma_L$ is the space of $L$-Lipschitz functions.
Let $\mathcal{S}:= \mathrm{supp}(Q)\cup\mathrm{supp} (\Ptar)$, where $\mathrm{supp}(Q)$ and $\mathrm{supp} (\Ptar)$ denote the supports of $Q$ and $\Ptar$, respectively. From \cite[Theorem 3]{chen2025robust}, the optimizer of \eqref{equation:Variational_KL_Lip} exists and is unique (up to a constant) on $\mathcal{S}$; we denote it by $\phi^*_{\mathcal{S}}:\mathcal{S}\to\R$. Extending $\phi^*_{\mathcal{S}}$  to $\R^d$ via $\phi^*(x):=\inf_{y\in \mathcal{S}}\big\{\phi^*_{\mathcal{S}}(y)+L\|x-y\|\big\}$ for $x\in\R^d$, we obtain $\phi^* \in \Gamma_L$ with $\phi^*|_{\mathcal{S}}=\phi^*_{\mathcal{S}}$, and for all $x \in \R^d$, 
\begin{eqnarray*}
  \frac{\delta \LipKL(Q\|\Ptar)}{\delta Q}(x)
=\phi^*(x).
\end{eqnarray*} For empirical measures supported on finite samples, both $D^{L}_{\mathrm{KL}}
(Q\|\Ptar)$ and its first variational derivative exist finitely (see
\cite[Theorem 2, Theorem 3]{chen2025robust}). Due to the scarcity of samples in the tail region, the velocity field induced by the Lipschitz-regularized KL divergence alone, often vanishes before the tail region is adequately learned. Thus, we penalize the Lipschitz-regularized KL divergence via CVaR to place additional weight on the tail.

\paragraph{Conditional Value-at-Risk} Let $g: \R^d  \rightarrow [0,\infty)$ be a risk function.  For $Q \in \mathcal{Q}^g:=\{Q \in \PR: \E_Q[g] < \infty\}$, the CVaR of $g$ at level $\alpha\in(0,1)$ is the average of $g$ over the extreme $(1-\alpha)$-fraction of outcomes: $\CVaR^{Q,g}_\alpha=\E_Q[g\,|\,g\ge q^{Q,g}_\alpha]$ when the law $g_{\#}Q$ of $g(X)$, $X\sim Q$, is continuous, with $q^{Q,g}_\alpha$ the $\alpha$-quantile of $g_{\#}Q$. Therefore, CVaR, by definition, focuses on the tail region of a given distribution, and hence is a natural choice for learning heavy-tailed targets. The Rockafellar--Uryasev formulation \citep{Rock2002CVaR} is a variational formulation of CVaR, which defines for every $Q\in\Qg$:
\begin{align}
\mathrm{CVaR}^{Q,g}_{\alpha}
&= \inf_{y\in \mathbb{R}}
F^{Q,g}_{\alpha}(y):=  \inf_{y\in \mathbb{R}} \left\{  y + \frac{1}{1-\alpha}  \mathbb{E}_{Q} [(g(\cdot)-y)^+] \right\},
\label{eq:cvar_definition_APP}
\end{align}
 where $u^+ = \max\{0,u\}$ for $u \in \R$. As shown in Theorem \ref{Thm:VD:CVaR} in Appendix \ref{Sec:CVaR:Loss:Clarke}, CVaR, in general, does not have a variational derivative, which is required to define Wasserstein gradient flows. Thus, we instead consider the smoothed CVaR introduced in \cite{kouri2016risk}.  We refer to  Appendix \ref{Sec:CVaR:Loss:Clarke} for detailed discussions on penalization via (unsmoothed) CVaR, using Clarke's generalized subgradients, and how it compares to the smoothed CVaR penalization.   For a smoothing parameter $h>0$, the smoothed CVaR is defined by
 \begin{align}
\label{Eq:ramp:RU}
\CVaR^{Q,g}_{\alpha,h}=\inf_{y\in\R}F^{Q,g}_{\alpha,h}(y):= \inf_{y\in\R} \left\{y+\frac{1}{1-\alpha}\,\E_Q\big[r_h(g-y)\big] \right\},
\end{align}
where $r_h(u):=u^{+}+h\log\big(1+e^{-|u|/h}\big)$
is the softplus smoothing of $u^+$. Throughout this work, CVaR-penalization refers to penalization via the smoothed CVaR.

Variational derivatives and subgradients (for Clarke's generalized subgradients discussed in Appendix \ref{Sec:CVaR:Loss:Clarke}) are taken with respect to perturbations of $Q$ of the following form.
\begin{Definition}[Admissible perturbations]
\label{def:admissible}
Let $Q\in\Qg$. A signed measure $\rho$ on $\R^d$ is a \emph{right-admissible} perturbation at $Q$ if
\begin{align*}
\int_{\R^d}d\rho=0,\qquad \|\rho\|:=\int_{\R^d}d|\rho|<\infty,\qquad \int_{\R^d}\|x\|\,d|\rho|<\infty,\qquad \int_{\R^d}g\,d|\rho|<\infty.
\end{align*}
 A signed measure $\rho$ is \emph{left-admissible} if $-\rho$ is right-admissible, and \emph{admissible} if it is both right- and left-admissible.
\end{Definition}
 Write $Q^\epsilon:=Q+\epsilon\rho\in\PR$ for all sufficiently small $\epsilon>0$. For every right- (respectively left-) admissible $\rho$ at $Q$ and every $\epsilon$ with $Q^\epsilon\in\PR$,
\begin{align}
\label{Eq:g:integrability:h}
\E_{Q^\epsilon}[g]\;\le\;\E_Q[g]+|\epsilon|\int g\,d|\rho|\;<\;\infty,
\end{align}
so that $Q^\epsilon\in\Qg$. The condition $\int\|x\|\,d|\rho|<\infty$ ensures that $Q^\epsilon$ has a finite first moment whenever $Q$ does. We have the following proposition that derives the variational derivative of the smoothed CVaR.
\begin{Proposition}[Properties of the  smoothed CVaR]
\label{prop:ramp}
Let $h>0$, let $g:\R^d\to[0,\infty)$ be measurable, and let $Q\in\Qg$. Then the function $F^{Q,g}_{\alpha,h}(\cdot): \R \rightarrow \R$ is in $C^{\infty}$ and strictly convex. Thus, it has a unique minimizer $q_{\alpha,h}^{Q,g}$, characterized by
\begin{align}
\label{Eq:ramp:root}
\E_Q\Big[\sigma\Big(\frac{g-q_{\alpha,h}^{Q,g}}{h}\Big)\Big]=1-\alpha, \quad  \sigma(t)=\frac{1}{1+e^{-t}}.
\end{align}
Furthermore, the variational derivative of the  smoothed CVaR, with respect to admissible perturbations (Definition \ref{def:admissible}), is given by
\begin{align}
\label{Eq:CVARD}
\frac{\delta}{\delta Q}\CVaR_{\alpha,h}^{Q,g}=\frac{1}{1-\alpha} r_h\left(g-q_{\alpha,h}^{Q,g} \right).
\end{align}
\end{Proposition}

\begin{proof}
Let $h>0$ be arbitrary.
Since $0\le r_h(g-y)\le g+|y|+h\log2$ and $\E_Q[g]<\infty$, we have $\E_Q[r_h(g-y)] < \infty$ for all $y \in \R$. Thus it follows that $F^{Q,g}_{\alpha,h}(y)$ is finite for all $y \in \R$. Furthermore, for all $y \in \R$ we have $|\partial_y r_h(g-y)|=\sigma\left(\frac{g-y}{h}\right)\le1$. Thus, by the dominated convergence theorem, we
differentiate under the expectation in \eqref{Eq:ramp:RU} to obtain
\begin{align}
\label{Eq:F:derivative}
\frac{d}{dy}F^{Q,g}_{\alpha,h}(y)=1-\frac{S_h^{Q}(y)}{1-\alpha},
\end{align}
where $S_h^{Q}(y):=\E_Q\Big[\sigma\Big(\frac{g-y}{h}\Big)\Big]$. Since $\sigma(\cdot) \in C^{\infty}$, and $\E_Q\Big[\sigma\Big(\frac{g-y}{h}\Big)\Big]$ is finite for all $y \in \R$, we have that $F^{Q,g}_{\alpha,h}(\cdot) \in C^{\infty}$.  Furthermore, {since $\sigma'>0$ and $|\sigma'|\le\frac14$, differentiating \eqref{Eq:F:derivative} once more under the expectation gives $\frac{d^2}{dy^2}F^{Q,g}_{\alpha,h}(y)=\frac{1}{(1-\alpha)h}\E_Q\big[\sigma'\big(\frac{g-y}{h}\big)\big]>0$}, and thus  $F^{Q,g}_{\alpha,h}(y)$ is strictly convex, and hence there is a unique minimizer $q_{\alpha,h}^{Q,g}$ such that
\begin{align}
\frac{d}{dy}F^{Q,g}_{\alpha,h}(y)\Big|_{y=q_{\alpha,h}^{Q,g}}=1-\frac{S_h^{Q}(q_{\alpha,h}^{Q,g})}{1-\alpha}=0.
\end{align}
That is
\begin{align}
\E_Q\Big[\sigma\Big(\frac{g-q_{\alpha,h}^{Q,g}}{h}\Big)\Big] = S_h^{Q}(q_{\alpha,h}^{Q,g}) = 1-\alpha.
\end{align}
To derive the variational derivative of $\CVaR^{Q,g}_{\alpha,h}$, we consider admissible perturbations $\rho$ at $Q$ in the sense of Definition \ref{def:admissible}, and write $Q^\epsilon:=Q+\epsilon\rho$. By \eqref{Eq:g:integrability:h}, $Q^\epsilon\in\Qg$ for all $\epsilon$ with $Q^\epsilon\in\PR$, so that $\CVaR^{Q^\epsilon,g}_{\alpha,h}$ is well defined. Let $\rho$ be an admissible perturbation at $Q$, set
$q_0:=q_{\alpha,h}^{Q,g}$ and $q_\epsilon:=q_{{\alpha,}h}^{Q^\epsilon,g}$.  Define
\begin{align}
\label{Eq:ell:def}
R(y):=\frac{1}{1-\alpha}\int r_h(g-y)\,d\rho.
\end{align}

{Since $0\le r_h(g-y)\le g+|y|+h\log2$ and $\int g\,d|\rho|<\infty$ by admissibility, $R(y)$ is finite for every $y\in\R$.} Since $r_h'(u)=\sigma(u/h) \leq 1$, $r_h$ is $1$-Lipschitz continuous, and $R(\cdot)$
is $\frac{\|\rho\|}{1-\alpha}$-Lipschitz in $y${.} Since
$F^{Q^\epsilon,g}_{\alpha,h}=F^{Q,g}_{\alpha,h}+\epsilon\,R$ by linearity of
the expectation in $Q$, evaluating the infimum at $q_0$ and at $q_\epsilon$
respectively gives
\begin{align*}
\CVaR_{\alpha,h}^{Q^\epsilon,g}&\le F^{Q^\epsilon,g}_{\alpha,h}(q_0)=\CVaR_{\alpha,h}^{Q,g}+\epsilon\,R(q_0),\\
\CVaR_{\alpha,h}^{Q^\epsilon,g}&=F^{Q,g}_{\alpha,h}(q_\epsilon)+\epsilon\,R(q_\epsilon)\ge\CVaR_{\alpha,h}^{Q,g}+\epsilon\,R (q_\epsilon).
\end{align*}
Hence, for $\epsilon>0$, we have
\begin{align*}
R(q_\epsilon) \leq \frac{1}{\epsilon} \left( \CVaR_{\alpha,h}^{Q^\epsilon,g}-\CVaR_{\alpha,h}^{Q,g}\right)\leq R(q_0)
\end{align*}
Similarly for $\epsilon<0$, we have
\begin{align*}
R(q_0) \leq \frac{1}{\epsilon} \left( \CVaR_{\alpha,h}^{Q^\epsilon,g}-\CVaR_{\alpha,h}^{Q,g}\right)\leq R(q_\epsilon)
\end{align*}

Fix $\eta>0$ and set
\begin{align}
\kappa:=\min\big\{S_h^{Q}(q_0-\eta)-(1-\alpha),\ (1-\alpha)-S_h^{Q}(q_0+\eta)\big\}.
\end{align}
Since the mapping $S_h^{Q}(y)$ is strictly decreasing and
equals $1-\alpha$ exactly at $q_0$, $\kappa$ is positive.

Next, $Q^\epsilon=Q+\epsilon\rho$ gives, for every $y\in\R$,
\begin{align}
S_h^{Q^\epsilon}(y)-S_h^{Q}(y)
=\epsilon\int\sigma\Big(\frac{g-y}{h}\Big)\,d\rho ,
\end{align}
and since $0<\sigma<1$,
\begin{align}
\label{Eq:S:perturbation}
\big|S_h^{Q^\epsilon}(y)-S_h^{Q}(y)\big|
\le|\epsilon|\int\Big|\sigma\Big(\frac{g-y}{h}\Big)\Big|\,d|\rho|
\le|\epsilon|\int d|\rho|=|\epsilon|\,\|\rho\| ,
\end{align}
uniformly in $y$.  Let now $|\epsilon|\,\|\rho\|<\kappa$. We have
\begin{align*}
S_h^{Q^\epsilon}(q_0-\eta)
&\ge S_h^{Q}(q_0-\eta)-|\epsilon|\,\|\rho\|\\
&>S_h^{Q}(q_0-\eta)-\kappa\\
&\ge S_h^{Q}(q_0-\eta)-\big(S_h^{Q}(q_0-\eta)-(1-\alpha)\big)=1-\alpha .
\end{align*}
Similarly we have
\begin{align*}
S_h^{Q^\epsilon}(q_0+\eta)
&\le S_h^{Q}(q_0+\eta)+|\epsilon|\,\|\rho\| \nonumber \\
&<S_h^{Q}(q_0+\eta)+\kappa \nonumber\\
&\le S_h^{Q}(q_0+\eta)+\big((1-\alpha)-S_h^{Q}(q_0+\eta)\big)=1-\alpha .
\end{align*}
Thus we have
\begin{align}
S_h^{Q^\epsilon}(q_0-\eta)>1-\alpha>S_h^{Q^\epsilon}(q_0+\eta).
\end{align}

The minimizer $q_\epsilon$ is
characterized by $S_h^{Q^\epsilon}(q_\epsilon)=1-\alpha$, and
$S_h^{Q^\epsilon}$ is strictly decreasing.  If $q_\epsilon\le q_0-\eta$, then
\begin{align}
1-\alpha=S_h^{Q^\epsilon}(q_\epsilon)\ \ge\ S_h^{Q^\epsilon}(q_0-\eta)\ >\ 1-\alpha,
\end{align}
a contradiction; and if $q_\epsilon\ge q_0+\eta$, then
\begin{align}
1-\alpha=S_h^{Q^\epsilon}(q_\epsilon)\ \le\ S_h^{Q^\epsilon}(q_0+\eta)\ <\ 1-\alpha,
\end{align}
again a {contradiction. Hence} $|q_\epsilon-q_0|<\eta$ whenever
$|\epsilon|\,\|\rho\|<\kappa$.

Since $\eta>0$ was arbitrary and $\kappa=\kappa(\eta)>0$, this says precisely
that for every $\eta>0$ there is $\epsilon_\eta>0$ such that
$|q_\epsilon-q_0|<\eta$ for all $|\epsilon|<\epsilon_\eta$; that is,
$q_\epsilon\to q_0$ as $\epsilon\to0$. Since $R$ is Lipschitz,
$R(q_\epsilon)\to R(q_0)$ as $\epsilon\to0$, thus we have
\begin{align}
\lim_{\epsilon\to0}\frac{\CVaR_{\alpha,h}^{Q^\epsilon,g}-\CVaR_{\alpha,h}^{Q,g}}{\epsilon}
=R(q_0)=\frac{1}{1-\alpha}\int r_h\big(g-q_{\alpha,h}^{Q,g}\big)\,d\rho .
\end{align}
The limit exists for every admissible $\rho$ and is linear in $\rho$, so
$\CVaR_{\alpha,h}^{\cdot,g}$ is Gateaux differentiable at $Q$ with
\begin{align}
\frac{\delta}{\delta Q}\CVaR_{\alpha,h}^{Q,g}
=\frac{r_h\big(g-q_{\alpha,h}^{Q,g}\big)}{1-\alpha}.
\end{align}
\end{proof}
The following assumption ensures that $\LipKL\QPtar$, the CVaR discrepancies, and the loss functional  $\F$ are finite.
\begin{Assumption}
\label{Assump}
The target distribution satisfies $\Ptar\in\Qg$, and the generated distribution $Q\in\Qg$ has a finite first moment.
\end{Assumption}
The condition $\Ptar\in\Qg$ holds for every empirical target measure $\widehat{\Ptar}$, which is the measure used by CVaR-GPA.
The condition $\Ptar\in\Qg$ holds for every empirical target measure $\widehat{\Ptar}$, which is the measure used by CVaR-GPA. Thus, CVaR-GPA can be applied even on extremely heavy-tailed targets with no finite first moments. 
\subsubsection{The divergence property}
\label{app:proof:DIV}
The following proposition establishes that the functional $\F(Q; \Ptar)$ is a divergence, so that the
target distribution $\Ptar$ is its unique global minimizer.
\begin{Proposition}[Divergence property]
\label{Prop:Div:Property}
Let Assumption \ref{Assump} hold. Then, the loss functional $\F$ defined in \eqref{Eq:penalized:loss:functional}
satisfies $\F(Q; \Ptar) \geq 0$ for all $Q \in \mathcal{Q}^g$ and we have $\F(Q; \Ptar) = 0$ if and only if
$Q=\Ptar$. That is, $\Ptar$ is the unique global minimizer of the optimization problem
$$\inf_{Q \in \mathcal{Q}^g}\F (Q; \Ptar).$$
\end{Proposition}
\begin{proof}
The Lipschitz-regularized KL divergence $D_{KL}^L(Q \| \Ptar)$ satisfies the divergence property,
i.e., $D_{KL}^L(Q \| \Ptar) \geq 0$ for all $Q \in \mathcal{P}(\R^d)$ and $D_{KL}^L(Q \| \Ptar) = 0$ if and
only if $Q=\Ptar$ (see \cite[Theorem 8]{Birrell2022fGamma} for the proof).  This, combined with the fact that
$ (\Delta C_h (Q; \Ptar))^2 \geq 0$ for all $Q \in \mathcal{Q}^g$, proves that  $\F(Q; \Ptar) \geq 0$  for all
$Q \in \mathcal{Q}^g$ and $\F(Q; \Ptar) = 0$  if and only if $Q=\Ptar$. That is, $\Ptar$ is the unique global
minimizer of the optimization problem $\inf_{Q \in \mathcal{Q}^g}\F (Q; \Ptar)$.
\end{proof}

Defining CVaR penalization through a scalar risk function $g$ is a crucial design feature: the CVaR penalization depends only on the target distribution via the one-dimensional pushforward measure $g_{\#}\Ptar$, through the scalar tail statistic $\CVaR_{\alpha,h}^{\Ptar,g}$, whose estimation error is independent of $d$ and $h$, as shown in the following lemma.

\begin{Lemma}[Dimension-free estimation of CVaR-penalization]
\label{Cor:CVaR:Estimation:App}
Let $\{X^{(j)}\}_{j=1}^{N}$ be i.i.d.\ samples of $\Ptar$, and let
$\widehat{\Ptar}:=\frac{1}{N}\sum_{j=1}^{N}\delta_{X^{(j)}}$ denote the corresponding empirical measure. Assume that $\E_{\Ptar}[g^{q}]<\infty$ for some $q>1$. Then there exists a constant $C${, depending only on $q$ and $\E_{\Ptar}[g^{q}]$ and independent of the dimension $d$ and the smoothing parameter $h>0$, such} that
\begin{align}
\label{Eq:CVaR:Estimation:Smo}
  \E_{(\Ptar)^{\otimes N}} \Big|\CVaR_{\alpha,h}^{\widehat{\Ptar},g}-\CVaR_{\alpha,h}^{\Ptar,g}\Big|
 \le \frac{C}{1-\alpha}\Big(N^{-\frac{1}{2}}+N^{-\frac{q-1}{q}}\Big).
\end{align} If $q=2$, the bound holds with $q$ replaced by any $q'\in(1,2)$.
\end{Lemma}
\begin{proof}
Let {$Q_1,Q_2\in\Qg$ and} $\mu_i:=g_{\#}Q_i$ for $i=1,2$. For every Borel measurable function
$\psi:\R\to\R$ such that $\E_{Q_i}[\psi\circ g]$ is finite, we have from the change of
variables formula
\begin{align}
\int_{\R^d}\psi\big(g(x)\big)\,dQ_i(x)=\int_{\R}\psi(z)\,d\mu_i(z).
\end{align}
Taking $\psi(z)=|z|$ and using $Q_i\in\Qg$ shows that $\mu_i$ has a finite
first moment and thus $\mathcal{W}_1(\mu_1,\mu_2)<\infty$.  Taking
$\psi(z)=r_h(z-y)$, and recalling that $0\le r_h(g-y)\le g+|y|+h\log2$ and $\E_{{Q_i}}[g]<\infty$ imply $\E_{Q_i}[r_h(g-y)] = \E_{Q_i}[\psi\circ g] < \infty$, we may rewrite the smoothed
Rockafellar--Uryasev objective as
\begin{align}
\label{Eq:ramp:RU:pushforward}
F^{Q_i,g}_{\alpha,h}(y)=y+\frac{1}{1-\alpha}\int_{\R}r_h(z-y)\,d\mu_i(z).
\end{align}

 For each fixed $y\in\R$, the map $z\mapsto r_h(z-y)$ is $1$-Lipschitz. Hence we have
\begin{align}
\label{Eq:ramp:F:d,ifference}
\big|F^{Q_1,g}_{\alpha,h}(y)-F^{Q_2,g}_{\alpha,h}(y)\big|
&=\frac{1}{1-\alpha}\left|\int_{\R}r_h(z-y)\,d\big[\mu_1-\mu_2\big](z)\right|
\;\leq\;\frac{1}{1-\alpha}\,\mathcal{W}_1(\mu_1,\mu_2),
\end{align}
for every $y\in\R$, where the last inequality follows from the Kantorovich--Rubinstein duality \citep{kantorovich1958space}.

Furthermore, we have
\begin{align}
\label{Eq_App_CB_1}
\CVaR_{\alpha,h}^{Q_1,g}
&\leq F^{Q_1,g}_{\alpha,h}\big(q_{\alpha,h}^{Q_2,g}\big) \nonumber\\
&\leq F^{Q_2,g}_{\alpha,h}\big(q_{\alpha,h}^{Q_2,g}\big)
   +\big|F^{Q_1,g}_{\alpha,h}\big(q_{\alpha,h}^{Q_2,g}\big)-F^{Q_2,g}_{\alpha,h}\big(q_{\alpha,h}^{Q_2,g}\big)\big|\nonumber\\
&\leq F^{Q_2,g}_{\alpha,h}\big(q_{\alpha,h}^{Q_2,g}\big)
   +\frac{1}{1-\alpha}\,\mathcal{W}_1(\mu_1,\mu_2)\nonumber\\
&=\CVaR_{\alpha,h}^{Q_2,g}+\frac{1}{1-\alpha}\,\mathcal{W}_1(\mu_1,\mu_2).
\end{align}

Similarly, we have
\begin{align}
\label{Eq_App_CB_2}
\CVaR_{\alpha,h}^{Q_2,g}
&\leq \CVaR_{\alpha,h}^{Q_1,g}+\frac{1}{1-\alpha}\,\mathcal{W}_1(\mu_1,\mu_2).
\end{align}
Combining \eqref{Eq_App_CB_1} and \eqref{Eq_App_CB_2}, we obtain
\begin{align}
\big|\CVaR_{\alpha,h}^{Q_1,g}-\CVaR_{\alpha,h}^{Q_2,g}\big|
\leq\frac{1}{1-\alpha}\,\mathcal{W}_1(\mu_1,\mu_2) = \frac{1}{1-\alpha}\,\mathcal{W}_1(g_{\#}Q_1,g_{\#}Q_2).
\end{align}

Thus we have
\begin{align}
\label{Eq_App_CB_3}
\Big|\CVaR_{\alpha,h}^{\widehat{\Ptar},g}-\CVaR_{\alpha,h}^{\Ptar,g}\Big| \leq\frac{1}{1-\alpha}\,\mathcal{W}_1(g_{\#}\Ptar,g_{\#}\widehat{\Ptar})
\end{align}

From \cite[Theorem 1]{fournier2015rate}, we have that, for a measure $\mu \in \mathcal{P}(\R)$, if there exists $q>1$ and $q \neq 2$ such that $\mu$ has finite $q^{\mathrm{th}}$ moment,
\begin{align}
\label{Eq:fournier}
  \E_{(\mu)^{\otimes N}} \Big|\mathcal{W}_1(\mu,\widehat{\mu}_N)\Big|
 \le C (N^{-\frac{1}{2}}+N^{-\frac{(q-1)}{q}}),
\end{align}
where $C$ is a coefficient that depends only on $q$ and the $q^{\mathrm{th}}$ moment of $\mu$. Combining \eqref{Eq_App_CB_3} and \eqref{Eq:fournier}, we obtain
\begin{align}
\label{Eq:CVaR:Estimation}
  \E_{(\Ptar)^{\otimes N}} \Big|\CVaR_{\alpha,h}^{\widehat{\Ptar},g}-\CVaR_{\alpha,h}^{\Ptar,g}\Big|
 \le \frac{C}{1-\alpha} (N^{-\frac{1}{2}}+N^{-\frac{(q-1)}{q}})
\end{align}
where $C$ depends only on $q$ and the $q^{\mathrm{th}}$ moment of $g_{\#}\Ptar$. If $q=2$, then $\E_{\Ptar}[g^{q'}]<\infty$ for every $q'\in(1,2)$, and the above argument applies with $q'$ in place of $q$. 
\end{proof}

 Consequently, the number of samples needed to estimate the target CVaR to a prescribed accuracy does not grow with $d$, making fine-tuning via CVaR-penalization scalable to high dimensions. CVaR penalization in the loss functional $\F$ thus places an additional weight on the tail region, at an estimation cost that does not grow with $d$.

On the other hand, the Lipschitz-regularized KL divergence $\LipKL\QPtar$ is equally essential to the CVaR-penalized loss functional.  By Proposition \ref{Prop:Div:Property}, under Assumption \ref{Assump}, $\F(Q;\Ptar)$ satisfies the divergence property, i.e., $\F(Q;\Ptar)\ge 0$, and the equality holds if and only if $Q=\Ptar$, making $\Ptar$ the unique global minimizer of $\F(\cdot\,;\Ptar)$. The term $(\Delta C_h(Q;\Ptar))^2$ alone does not satisfy this property, since infinitely many distributions $Q$ share the same smoothed CVaR as $\Ptar$. It is $\LipKL\QPtar$ that endows $\F$ with the divergence property, and thereby ensures that its minimizer uniquely identifies the target distribution. Moreover, since it compares $Q$ and $\Ptar$ directly as a divergence, $\LipKL\QPtar$ also captures structural properties of $\Ptar$ that the scalar statistic $\CVaR_{\alpha,h}^{Q,g}$ cannot reflect. Thus, the two components of the loss functional play distinct and complementary roles.

We have the following result.
 
\begin{Theorem}
\label{thm:vel:h}
Let $g\in C^1$  and $Q\in\Qg$, and $h>0$. The Wasserstein gradient flow of $\F(Q;\Ptar)$  is governed by the transport PDE:
\begin{align}
\label{Eq:Gradient:Flow_1}
&    \partial_t Q_t + \nabla\cdot(Q_t v_t(x) ) = 0, \quad v_t(x)=-\nabla_x\Big(
    \frac{\delta \mathcal{F}(Q;\Ptar)}{\delta Q}\Big|_{Q=Q_t}
    \Big),  \quad t>0, \quad Q_0 =\Ppre,
\end{align}
where  $Q_t$ denotes the evolving distribution at time $t$ and the variational derivative of $\F$ is given by
\begin{align}
\label{Eq:VD:F}
\frac{\delta\F(Q;\Ptar)}{\delta Q}=\phi^*-\frac{2\lambda}{1-\alpha}\,\Delta C_h(Q;\Ptar)\,r_h\big(g-q_{\alpha,h}^{Q,g}\big).
\end{align}
\end{Theorem}

\begin{proof}
Let $\rho$ be an admissible perturbation at $Q$ and write $Q^\epsilon=Q+\epsilon\rho$. Since
$\Delta C_h(Q;\Ptar)=\CVaR_{\alpha,h}^{\Ptar,g}-\CVaR_{\alpha,h}^{Q,g}$ and the first term does not depend on $Q$,
\begin{align*}
&\lim_{\epsilon\to0}\frac{\big(\Delta C_h(Q^\epsilon;\Ptar)\big)^2-\big(\Delta C_h(Q;\Ptar)\big)^2}{\epsilon}\\
=&\lim_{\epsilon\to0}\Big[\Delta C_h(Q^\epsilon;\Ptar)+\Delta C_h(Q;\Ptar)\Big]\,
\frac{\CVaR_{\alpha,h}^{Q,g}-\CVaR_{\alpha,h}^{Q^\epsilon,g}}{\epsilon}\\
=&-2\,\Delta C_h(Q;\Ptar)\,\frac{1}{1-\alpha}\int r_h\big(g-q_{\alpha,h}^{Q,g}\big)\,d\rho,
\end{align*}
where the last equality follows from \eqref{Eq:CVARD}.
Hence
\begin{equation*}
\frac{\delta}{\delta Q}\big(\Delta C_h(Q;\Ptar)\big)^2
=-\frac{2}{1-\alpha}\,\Delta C_h(Q;\Ptar)\,r_h\big(g-q_{\alpha,h}^{Q,g}\big).
\end{equation*}
Combining this with $\frac{\delta}{\delta Q}\LipKL(Q\|\Ptar)=\phi^*$, and with the definition \eqref{Eq:penalized:loss:functional} of $\F$, the linearity of the first variation gives
\begin{equation*}
\frac{\delta\F(Q;\Ptar)}{\delta Q}
=\phi^*-\frac{2\lambda}{1-\alpha}\,\Delta C_h(Q;\Ptar)\,r_h\big(g-q_{\alpha,h}^{Q,g}\big),
\end{equation*}
which is the asserted first variation. Both terms are finite under Assumption \ref{Assump}: $\LipKL(Q\|\Ptar)$ and $\phi^*$ by \cite[Theorem 2, Theorem 3]{chen2025robust}, and the CVaR term by Proposition \ref{prop:ramp}.

The {smoothed CVaR-penalized Wasserstein gradient} flow is \eqref{Eq:Gradient:Flow_1} with $\mathcal F=\F$ and {$Q_0=\Ppre$}, with velocity field $v_{Q,h}=-\nabla_x\frac{\delta\F(Q;\Ptar)}{\delta Q}$. Since $r_h'(u)=\sigma(u/h)$ and $g\in C^1$, the chain rule gives, for {every $x\in\R^d$},
\begin{equation*}
\nabla_x\, r_h\big(g(x)-q_{\alpha,h}^{Q,g}\big)=\sigma\Big(\frac{g(x)-q_{\alpha,h}^{Q,g}}{h}\Big)\nabla_x g(x),
\end{equation*}
and thus we obtain \eqref{Eq:VD:F}.
\end{proof}
 The smoothed CVaR-penalized Wasserstein gradient flow is thus governed by \eqref{Eq:Gradient:Flow_1}, with the velocity field 
\begin{align}
\label{Eq:ramp:vel}
v_{Q,h}(x) = \underbrace{-\nabla_x\phi^*(x)}_{\text{from }\LipKL(Q\|\Ptar)}+ \lambda \underbrace{\frac{2}{1-\alpha}\,\Delta C_h(Q;\Ptar)\,\sigma\Big(\frac{g(x)-q_{\alpha,h}^{Q,g}}{h}\Big)\nabla_x g(x)}_{\hat{v}_{Q,h}(x): \text{ velocity component from CVaR-penalization}},
\end{align}
where $\sigma(t)=\frac{1}{1+e^{-t}}$.  The logistic factor $\sigma\big(\frac{g(x)-q_{\alpha,h}^{Q,g}}{h}\big)$ is a smooth surrogate for the indicator function of $\{g(x)>q_{\alpha,h}^{Q,g}\}$, and thus localizes the velocity component $\hat{v}_{Q,h}$ to the tail region. In what follows, we fix the risk function as $g(x)=\|x\|$, which captures extremes in every direction through a scalar, and hence is a natural choice for $g$. The resulting  velocity field is given by \eqref{Eq:ramp:vel} with $\nabla_x g(x)=\frac{x}{\|x\|}$, adopting the convention $\frac{x}{\|x\|}:=0$ and $v_{Q,h}(x) =-\nabla_x\phi^*(x)$ at $x=0$.

In CVaR-GPA, the velocity component from the Lipschitz-regularized KL divergence is approximated via neural network approximations, similar to \cite{chen2025robust,gu2024lipschitz}, and learns the joint structure of the target $\Ptar$ due to its divergence property, which the scalar CVaR discrepancy alone cannot capture. 
For every $Q \in \Qg$ with a finite first moment, including empirical measures supported on finite samples, $\F(Q;\Ptar)$ and its variational derivative \eqref{Eq:VD:F} are well defined, allowing us to initialize the fine-tuning flow \eqref{Eq:Gradient:Flow_1} with samples of the pre-trained model $\Ppre$, satisfying property \PropI. The number of samples required to estimate the velocity component $\hat{v}_{Q,h}$ does not scale with $d$, since $\hat{v}_{Q,h}$ depends on $\Ptar$ only through the scalar $\CVaR_{\alpha,h}^{\Ptar,g}$,   which, as established in Lemma \ref{Cor:CVaR:Estimation:App}, can be estimated with a finite-sample error bound independent of the target dimension $d$.  Since $\sigma(t)>0$, $\hat{v}_{Q,h}(x)\neq0$ whenever $\Delta C_h(Q;\Ptar)\neq0$ and $x\neq0$, irrespective of how many particles lie in the tail region: the CVaR component remains active in the under-sampled tail. Thus, the CVaR-penalization contributes a \textit{dimension-free} velocity component $\hat{v}_{Q,h}$ that \textit{remains active in the tail region}, satisfying property \PropII. 

The velocity field ${v}_{Q,h}$ is \textit{target-adaptive} via the CVaR discrepancy $\Delta C_h(Q;\Ptar)$, satisfying property \PropIII. The magnitude of $\hat{v}_{Q,h}$  is  proportional to $|\Delta C_h(Q;\Ptar)|$. That is, the CVaR penalization supplies a self-regulating velocity field, driven by the  CVaR discrepancy. We further note that the velocity component $-\nabla_x\phi^*(\cdot)$ is merely bounded but not necessarily Lipschitz continuous, while the velocity component $\hat{v}_{Q,h}$ is locally Lipschitz on $\R^d\setminus\{0\}$ but admits no global Lipschitz bound. Thus, the velocity field $v_{Q,h}$ does not adhere to the global Lipschitz constraint of \cite{jaini2020tails}. 

\section{CVaR-penalized Generative Particle Algorithm}
\label{Sec:CGPA}
The smoothed CVaR-penalized Wasserstein gradient flow describes the time-evolving distribution $Q_t$ of a particle system. Equivalently, the dynamics of the particle system can be described by the Lagrangian formulation of  the transport/variational PDE \eqref{Eq:Gradient:Flow_1} given by 
\begin{align}
\label{Eq:Lagrangian}
\frac{d}{dt}Y_t=v_{Q_t,h}(Y_t) , \qquad Y_0\sim\Ppre,\qquad \text{where } Y_t\sim Q_t,
\end{align}
with the velocity field $v_{Q_t,h}$ given by \eqref{Eq:ramp:vel}. This motivates a neural particle algorithm, analogous to the Generative Particle Algorithm (GPA) developed in \cite{gu2024lipschitz}, but with an additional CVaR penalty.
We build the CVaR-penalized Generative Particle Algorithm (CVaR-GPA) via an Euler discretization of \eqref{Eq:Lagrangian}, by replacing the probability measures with their empirical counterparts. 
\paragraph{The update scheme of CVaR-GPA}  We denote by $\widehat{Q}_k$ the empirical measure of the generated distribution $\{Y_k^{(i)}\}_{i=1}^M$ at the $k^{\mathrm{th}}$ iteration, and by $\widehat{\Ptar}$ the empirical measure of the target samples $\{X^{(j)}\}_{j=1}^N$. Let $\Gamma^{\mathrm{NN}}_L$ be a neural network approximation, with spectral normalization \citep{miyato2018spectral}, of the function space $\Gamma_L$ of $L$-Lipschitz functions. Because $\widehat{Q}_k$ evolves over iterations, at each step the function $\phi^{*}_k$ is re-obtained by solving the variational problem \eqref{equation:Variational_KL_Lip} over the approximated function space $\Gamma^{\mathrm{NN}}_L$.  The velocity contribution $-\nabla_x \phi_k^*(Y_k^{(i)})$ from the Lipschitz-regularized KL divergence is evaluated by automatic differentiation of the network at the particle positions. Thus, for a chosen step size $\Delta t$, the  $k^{\mathrm{th}}$ iteration of the update scheme is
\begin{align}
& Y^{(i)}_{k+1} =Y^{(i)}_{k} +  \Delta t\,  v^{(i)}_{k} , \quad v^{(i)}_{k}:= -
\nabla_x\phi^{*}_k(Y^{(i)}_{k})+ \lambda  \,b^{(i)}_{k}, \quad Y_0^{(i)} \sim \Ppre, \quad i=1,\ldots,M,\nonumber\\
&  \phi^{*}_k = \arg\max_{\phi\in\Gamma^{\mathrm{NN}}_L}
    \Big\{\frac{1}{M}\sum_{i=1}^{M}\phi(Y^{(i)}_{k})
      -\log \Big( \frac{1}{N}\sum_{j=1}^{N} e^{\phi(X^{(j)})} \Big)\Big\},\label{Eq:Update:Scheme}\\
&  b^{(i)}_{k}=
\frac{2\,\Delta_k}{1-\alpha}\,w^{(i)}_{k}\,\frac{Y^{(i)}_{k}}{\|Y^{(i)}_{k}\|},
\quad \Delta_k:=\CVaR_{\alpha,h}^{\widehat{\Ptar},g}-\CVaR_{\alpha,h}^{\widehat{Q}_k,g},
\quad w^{(i)}_{k}:=\sigma\Big(\frac{\|Y^{(i)}_{k}\|-q_{\alpha,h}^{\widehat{Q}_k,g}}{h}\Big). \nonumber
\end{align}
The necessary tail statistics for the $k^{\mathrm{th}}$  update, namely  $q_{\alpha,h}^{\widehat{Q}_k,g}$, the weights $w_k^{(i)}$, and the empirical CVaR discrepancy $\Delta_k$, are computed via Algorithm \ref{alg:tail}. In contrast to $\phi_k^*$, whose approximation uses a neural network, these empirical tail statistics are computed directly from the particle samples and require no neural network approximation. The fine-tuning flow terminates when its velocity field dies out, i.e., when $v_{Q_{T^*},h}=0$ for some terminal time $T^*>0$. Hence, one can use the empirical estimate of the velocity field, or the empirical estimate of the total kinetic energy of the generative particles at the $k^{\mathrm{th}}$ iteration, given by $\mathcal K_k:=\frac{1}{2M}\sum_{i=1}^{M}\|v^{(i)}_{k}\|_2^2$, to define a stopping criterion for CVaR-GPA. We terminate the algorithm when $\mathcal{K}_k< \epsilon$ for a given threshold $\epsilon$.

\subsection{Algorithms}
\Cref{alg:tail} computes the smoothed empirical tail statistics used by CVaR-GPA, and \Cref{alg:gpa} states CVaR-GPA in full.
\begin{algorithm}[htbp]
\caption{Smoothed empirical tail statistics}
\label{alg:tail}
\begin{algorithmic}[1]
\Require samples $\{Z_i\}_{i=1}^{m}$, $g(\cdot)=\|\cdot\|$: the risk function, $\alpha$:  level, $h$: smoothing parameter, $\tau$: root-finding tolerance
\State $g_i\gets g(Z_i)$, \quad $i=1,\dots,m$
\State $q_h\gets$ root of $q\mapsto\frac1m\sum_{i=1}^{m}\sigma\big((g_i-q)/h\big)-(1-\alpha)$, to tolerance $\tau$ via bisection method
\State $\CVaR_h\gets q_h+\dfrac{1}{(1-\alpha)\,m}\displaystyle\sum_{i=1}^{m}r_h(g_i-q_h)$ \Comment{empirical form of \eqref{Eq:ramp:RU}}
\State \Return $(q_h,\ \CVaR_h)$ \Comment{$q_h=q_h^{\widehat{Q},g}$ and $\CVaR_h=\CVaR_{\alpha,h}^{\widehat{Q},g}$ for the empirical measure $\widehat{Q}$ of $\{Z_i\}$}
\end{algorithmic}
\end{algorithm}

The update scheme \eqref{Eq:Update:Scheme}, together with the stopping criterion, describes CVaR-GPA, which is stated in full as \Cref{alg:gpa}. CVaR-GPA inherits the properties \PropI--\PropIII{} of the underlying flow at the particle level. The update scheme \eqref{Eq:Update:Scheme} is initialized with samples from $\Ppre$, without requiring access to the architecture of pre-trained models. The CVaR velocity component $b^{(i)}_k$ places additional weight on the tail: since $w^{(i)}_k>0$, every tail particle receives a nonzero displacement whenever $\Delta_k\neq0$, however few particles lie in the tail. The CVaR velocity component $b^{(i)}_k$   adapts to the target via the empirical CVaR discrepancy $\Delta_k$, and the stopping criterion results in an algorithm whose architecture depth is implicitly determined by $\Ptar$.  Thus, CVaR-GPA does not require the tail index or an estimate of the tail index as an input, making CVaR-GPA a tail-agnostic algorithm. Together, these properties improve the accuracy in learning heavy-tailed targets and make CVaR-GPA robust across various pre-trained models and datasets, as we demonstrate in \Cref{Sec:Results}. 

The source code for the algorithm CVaR-GPA can be found at 
\url{https://github.com/zhizhen-z/CVaR-GPA}.

\begin{algorithm}[htbp]
\caption{CVaR-GPA (smoothed CVaR)}
\label{alg:gpa}
\begin{algorithmic}[1]
\Require {Hyperparameters $\alpha$, $\lambda$, $h$, $L$, $\Delta t$; risk function $g(\cdot)=\|\cdot\|$; number of pre-trained particles $M$; number of target particles $N$; kinetic energy threshold $\epsilon$; neural network approximation $\Gamma^{\mathrm{NN}}_L$; number of network updates per iteration $N_\phi$.} 
\Result Fine-tuned particles $\{Y_{{k}}^{(i)}\}_{i=1}^{M}$ at termination.
\State Sample $\{X^{(j)}\}_{j=1}^{N} \sim \Ptar$, a batch of samples from the target distribution.
\State Sample $\{Y_0^{(i)}\}_{i=1}^{M} \sim \Ppre$, a batch of samples from the pre-trained distribution.
\State Compute the target tail statistic $\CVaR_{\alpha,h}^{\widehat{\Ptar},g}$ via \Cref{alg:tail} using  $\{X^{(j)}\}_{j=1}^{N}$.
\State $k\gets 0$, \ $\mathcal K\gets 2\epsilon$
\While{$\mathcal K>\epsilon$}
  \State $\phi^*_k\gets\displaystyle\arg\max_{\phi\in\Gamma_L^{\mathrm{NN}}}\Big\{\tfrac1M\textstyle\sum_i\phi(Y_k^{(i)})-\log\big(\tfrac1N\sum_j e^{\phi(X^{(j)})}\big)\Big\}$ \Comment{$N_\phi$ steps}
  \State $(q_{\alpha,h}^{\widehat{Q}_k,g},\CVaR_{\alpha,h}^{\widehat{Q}_k,g})$ via \Cref{alg:tail} using  $\{Y^{(i)}_k\}_{i=1}^{M}$.
   \State $\Delta_k\gets \CVaR_{\alpha,h}^{\widehat{\Ptar},g}-\CVaR_{\alpha,h}^{\widehat{Q}_k,g}$
  \For{$i=1,\dots,M$}
     \State  $v_k^{(i)}\gets -\nabla_x\phi^*_k(Y_k^{(i)})+\dfrac{2\lambda}{1-\alpha}\,\Delta_k\,\sigma\Big(\dfrac{\|Y_k^{(i)}\|-q_{\alpha,h}^{\widehat{Q}_k,g}}{h}\Big)\,\dfrac{Y_k^{(i)}}{\|Y_k^{(i)}\|}$ \Comment{\eqref{Eq:Update:Scheme}}
  \EndFor
  \State $Y_{k+1}^{(i)}\gets Y_k^{(i)}+\Delta t\,v_k^{(i)}$;\quad $\mathcal K\gets\tfrac1{2M}\sum_i\|v_k^{(i)}\|^2$
  \State $k\gets k+1$ 
\EndWhile
\State \Return $\{Y_{{k}}^{(i)}\}_{i=1}^M$
\end{algorithmic}
\end{algorithm}
\section{Numerical Experiments}
\label{Sec:Results}
 We evaluate CVaR-GPA on both synthetic and real-world datasets that assess its robustness across varying tail indices and targets with correlated marginals, and its scalability to high dimensions. The experiments demonstrate how the properties \PropI--\PropIII{} embedded in the fine-tuning flow improve the accuracy of the pre-trained models and make the fine-tuning robust across pre-trained models of different architectures and datasets with heterogeneous tail behaviors. 

\paragraph{Datasets}
We first assess the tail-agnostic performance of CVaR-GPA on a \textit{2-dimensional isotropic Student-$t$} distribution with tail indices $\nu\in\{1,1.2,1.5,1.8\}$, where a smaller $\nu$ corresponds to a heavier power-law tail. CVaR-GPA is applied to all four targets using identical hyperparameters, without prior knowledge of $\nu$. We next consider  \textit{Neal's funnel distribution} \citep{neal2003slice}, a standard benchmark for heavy-tailed generative models \citep{jaini2020tails,pandey2025_heavytaildiffusion}, characterized by strongly heterogeneous and dependent marginal behavior. Both synthetic experiments use $N=5000$ target samples. We then evaluate CVaR-GPA on two real-world, high-dimensional datasets with heterogeneous heavy tails, drawn from domains in which accurate modeling of extremes is consequential. The first is the \textit{Fama--French 25 monthly portfolios} \citep{FAMA19933} ($d=25$, $N=1182$ months), a canonical benchmark in empirical asset pricing, with marginal Hill tail-index estimates \citep{hill1975} ranging from $2.2$ to $3.2$. The second is \textit{daily streamflow} \citep{usgs_nwis} from the $64$ gauges with the smallest drainage areas among the $312$ Ohio River basin gauges having complete records over 2005--2024 ($d=64$, $N=7305$ days). Streamflow extremes directly inform flood-risk assessment and hydraulic design, making accurate modeling of tail behavior particularly consequential. Its estimated marginal tail indices range from $1.05$ to $3.34$, with $51$ of the $64$ gauges having tail indices below $2$, and therefore lying in the infinite-variance regime. Streamflow thus provides a particularly demanding high-dimensional real-world benchmark, with extremely heavy-tailed marginals where the tail is especially difficult to learn from finite samples.

\paragraph{Pre-trained models}
For each dataset, we fine-tune seven pre-trained generative models: Lip-KL-GPA \citep{chen2025robust}, TTF \citep{hickling2024flexible}, Tail-GAN \citep{cont2026tail}, t-EDM \citep{pandey2025_heavytaildiffusion}, SHD \citep{liu2024_heavytaileddiffusion}, DLPM \citep{shariatian2024denoisinglevyprobabilisticmodels}, and mTAF \citep{laszkiewicz2022marginal}. All models are trained on the same target samples using the settings reported by their respective authors.
We note here that Tail-GAN is designed mainly to reproduce tail risk statistics of financial portfolios with heavy-tailed distributions rather than to learn the underlying distribution. We nonetheless include it because its training objective also incorporates CVaR.

\paragraph{Parameter selection}
 The number of generated particles is set to be equal to the number of target samples, i.e., $M=N$. The same hyperparameters $(\lambda, \alpha, h, L, \Delta t)= (4 (1-\alpha)   \times10^{-3},\, 1-\tfrac{3}{N},\, 0.5,\, 0.125,\, 25)$ are used across all four datasets and all seven pre-trained models. The order of $\lambda$ is chosen at the level of $1-\alpha$, so that the coefficient of the CVaR-induced velocity component in \eqref{Eq:ramp:vel} is independent of $N$. The choice of $\alpha$ is discussed in detail in Appendix~\ref{App:Alpha}.  No hyperparameter is adjusted to the target distribution or the model being fine-tuned; in particular, CVaR-GPA does not use estimates of the target tail-decay rate. CVaR-GPA is terminated when the empirical kinetic energy falls below the prescribed threshold $\epsilon = 10^{-10}$, as described in \Cref{alg:gpa}, with a maximum of $20{,}000$ iterations imposed as a computational safeguard.

\paragraph{Error metrics}
Standard error metrics weight all regions of the domain equally and can therefore understate tail error, since the extreme tail carries only a small fraction of the total probability mass.
 A model may therefore achieve low global error while still failing to capture the tail accurately. Thus, we use both standard and tail-sensitive metrics for evaluation. The \textbf{global $L^1$ error}, $\mathcal{E}_{L^1}$, is the $L^1$ distance between the complementary CDFs (CCDFs) of the risk values under the generated and target empirical distributions, $\widehat{Q}$ and $\widehat{\Ptar}$. The \textbf{tail error}, $\mathcal{E}_{\mathrm{tail}}$, is the mean log-CCDF discrepancy between $\widehat{Q}$ and $\widehat{\Ptar}$ over the tail region beyond the $0.95$ quantile of $\widehat{\Ptar}$.

For an empirical measure $\widehat{Q}_N$ with samples $\{Z^{(i)}\}_{i=1}^{N}$, the CDF and CCDF (complementary CDF) of the risk function $g$ are given by $\Psi^{\widehat{Q}_N,g}(r)=\frac1N\sum_{i=1}^{N}\ind_{\{g(Z^{(i)})\le r\}}$ and $\bar{\Psi}^{\widehat{Q}_N,g}=1-\Psi^{\widehat{Q}_N,g}$, respectively. The \textbf{global $L^1$ error} is
\begin{align}
\label{Eq:Global:Error}
\mathcal{E}_{L^1}(\widehat{Q}_N; \widehat{\Ptar}_N) =
\int_{0}^{r_{\max}} \big| \bar{\Psi}^{\widehat{Q}_N,g}(r) -  \bar{\Psi}^{\widehat{\Ptar}_N,g}(r)\big|\,dr,
\end{align}
where $r_{\max}$ is the largest risk value in either sample and the integral is evaluated by the trapezoidal rule on a uniform $2000$-point grid of $[0,r_{\max}]$. Since $\bar{\Psi}$ is itself small in the tail, absolute CCDF differences are uninformative there, so we measure the \textbf{tail error} by the log-CCDF discrepancy over the tail quantiles,
\begin{align}
\label{Eq:Tail:Error}
\mathcal{E}_{\text{tail}}(\widehat{Q}_N; \widehat{\Ptar}_N)=
\frac{1}{20}\sum_{j=1}^{20}
\Big|\log\tfrac{\bar{\Psi}^{\widehat{Q}_N,g}(t_j)+\varepsilon}{\bar{\Psi}^{\widehat{\Ptar}_N,g}(t_j)+\varepsilon}\Big|,
\qquad t_j=(\Psi^{\widehat{\Ptar}_N,g})^{-1}(p_j),
\end{align}
where $p_1,\dots,p_{20}$ are equi-spaced in $[0.95,0.999]$ and $\varepsilon=10^{-10}$ is added for numerical stability.

\paragraph{Error reduction after fine-tuning}
\Cref{tab:headline:summary} reports the global and tail error-reduction factors across the seven pre-trained models for four representative targets: the Cauchy distribution ($\nu=1$) from the Student-$t$ family, Neal's funnel, Fama--French 25, and daily streamflow. For the $2$-dimensional synthetic targets we report the error reduction factors on their joint distribution, and for the high-dimensional Fama--French 25 and streamflow targets, the factor reported for each model is itself the geometric mean of the per-marginal error-reduction factors over the $d$ marginals. Because the error-reduction factors are multiplicative ratios and span several orders of magnitude, we summarize them across models using the geometric mean.  The corresponding error values are reported in Appendix~\ref{App_FULLtable}. Across these $56$ dataset--model--error metric combinations, the global and tail geometric mean reductions are $14.0 \times$ and $9.8 \times$, respectively. Crucially, across these 56 instances, CVaR-GPA reduces errors in all but one case,  without any model-specific modifications, highlighting the \textit{architecture-agnostic} nature of the GPA framework.  The one case where the error increases occurs for the tail error of TTF on the Cauchy target, where the pre-trained model already achieves a relatively small tail error of $0.138$, leaving little room for further improvement; consequently, fine-tuning provides no additional benefit and instead leads to a marginal increase of $0.01$. 

\begin{table}
\vspace{-1.1\baselineskip}
\centering
\caption{
Error-reduction factors (pre-trained error divided by fine-tuned error) for seven models after CVaR-GPA fine-tuning on the four targets: the 2-dimensional Cauchy distribution, Neal's funnel, Fama--French 25, and streamflow. The only increase in error ($1$ out of $56$ cases) is shown in gray. The last row reports the geometric mean across the seven models for each target. The corresponding error values across the seven pre-trained models are reported in Appendix~\ref{App_FULLtable}.
}
\label{tab:headline:summary}
\begin{tabular}{@{}l cc cc cc cc@{}}\toprule
pre-trained & \multicolumn{2}{c}{Cauchy} & \multicolumn{2}{c}{Neal} & \multicolumn{2}{c}{Fama--French 25} & \multicolumn{2}{c}{Streamflow}\\
\cmidrule(lr){2-3}\cmidrule(lr){4-5}\cmidrule(lr){6-7}\cmidrule(lr){8-9}
model & $\mathcal{E}_{L^1}$ & $\mathcal{E}_{\mathrm{tail}}$ & $\mathcal{E}_{L^1}$ & $\mathcal{E}_{\mathrm{tail}}$ & $\mathcal{E}_{L^1}$ & $\mathcal{E}_{\mathrm{tail}}$ & $\mathcal{E}_{L^1}$ & $\mathcal{E}_{\mathrm{tail}}$\\
\midrule
Lip-KL-GPA & 6.6$\times$ & 67.3$\times$ & 4.5$\times$ & 4.6$\times$ & 4.2$\times$ & {36.1}$\times$ & {1.7}$\times$ & {11.6}$\times$\\
TTF & 2.5$\times$ & \textcolor{gray}{0.9$\times$} & 13.6$\times$ & 16.0$\times$ & {32.7}$\times$ & {6.2}$\times$ & {4.9}$\times$ & 1.6$\times$\\
Tail-GAN & 33.0$\times$ & 130.6$\times$ & 33.6$\times$ & 77.5$\times$ & {9.6}$\times$ & {20.1}$\times$ & {2.9}$\times$ & {4.6}$\times$\\
t-EDM & 5.1$\times$ & 5.7$\times$ & 7.8$\times$ & 6.2$\times$ & 3.5$\times$ & {4.1}$\times$ & 2.2$\times$ & {3.0}$\times$\\
SHD & 5.4$\times$ & 2.7$\times$ & 73.9$\times$ & 19.0$\times$ & {34.0}$\times$ & {20.7}$\times$ & {95.9}$\times$ & {13.5}$\times$\\
DLPM & 35.0$\times$ & 1.6$\times$ & 267.4$\times$ & 10.5$\times$ & {468.2}$\times$ & {25.0}$\times$ & 667.7$\times$ & {10.6}$\times$\\
mTAF & 4.8$\times$ & 5.2$\times$ & 24.8$\times$ & 20.7$\times$ & {11.9}$\times$ & {28.8}$\times$ & 2.0$\times$ & {3.9}$\times$\\
\midrule
\textbf{geom.\ mean} & \textbf{8.2}$\times$ & \textbf{7.2}$\times$ & \textbf{25.9}$\times$ & \textbf{14.7}$\times$ & \textbf{{18.9}}$\times$ & \textbf{{16.1}}$\times$ & \textbf{{9.5}}$\times$ & \textbf{{5.4}}$\times$\\
\bottomrule\end{tabular}
\end{table}

\begin{table}
\vspace{-1.1\baselineskip}
\centering
\caption{Geometric mean, over the seven pre-trained models, of the error-reduction factors after CVaR-GPA fine-tuning on the 2-dimensional isotropic Student-$t$ family, for each tail index $\nu$.}
\label{tab:studentt:factors:pernu}
\begin{tabular}{@{}l cc@{}}\toprule
tail index & $\mathcal{E}_{L^1}$ & $\mathcal{E}_{\mathrm{tail}}$\\ \midrule
$\nu=1$ (Cauchy) & 8.2$\times$ & 7.2$\times$\\
$\nu=1.2$ & 20.6$\times$ & 11.2$\times$\\
$\nu=1.5$ & 24.1$\times$ & 13.4$\times$\\
$\nu=1.8$ & 18.5$\times$ & 10.2$\times$\\
\midrule
\textbf{geom.\ mean} & \textbf{16.6}$\times$ & \textbf{10.3}$\times$\\
\bottomrule\end{tabular}
\end{table}
 The empirical results validate features that CVaR-GPA inherits from its underlying mathematical framework.  \Cref{fig:ER_student_t_nu1} presents the
global and tail errors of the seven pre-trained models, before and after
fine-tuning, on the \textbf{2-dimensional Student-$t$} target with tail index
$\nu=1$ (Cauchy), the heaviest-tailed member of the family. Both errors are
reduced for every pre-trained model, with the single exception of the tail
error of TTF, whose pre-trained tail error is already small. 
\Cref{tab:studentt:factors:pernu} reports the global and tail error-reduction factors after CVaR-GPA fine-tuning of the seven pre-trained models on the \textbf{2-dimensional Student-$t$ family}, across four tail indices, $\nu\in\{1,1.2,1.5,1.8\}$. Crucially, these error reductions are obtained with the same hyperparameters across the entire family: CVaR-GPA requires no tail-index-specific adaptation, since the velocity field \eqref{Eq:ramp:vel}, driving CVaR-GPA, implicitly adapts to the target.

\begin{figure}[htbp]
\centering
\includegraphics[width=0.9\linewidth]{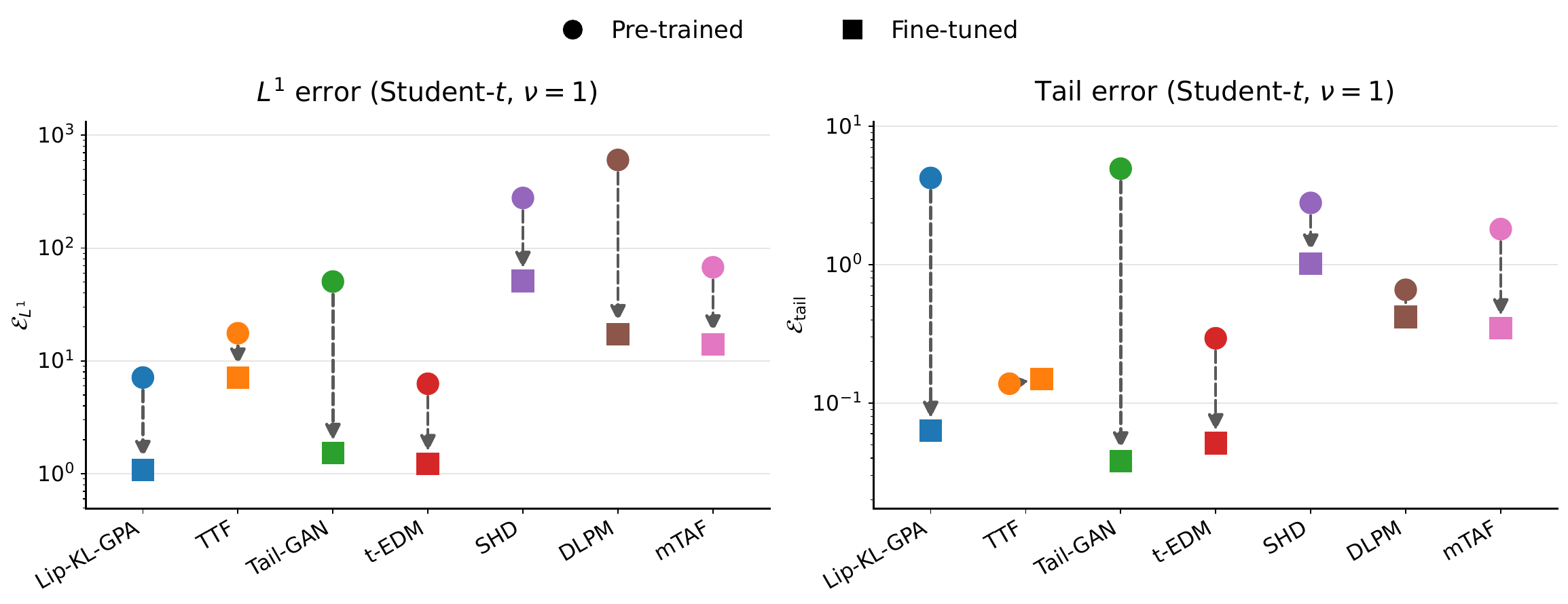}
\caption{2-dimensional Student-$t$ target with $\nu=1$ (Cauchy): global $L^1$
error $\mathcal{E}_{L^1}$ (left) and tail error $\mathcal{E}_{\mathrm{tail}}$
(right) of the seven pre-trained models before (circles) and after (squares)
fine-tuning via CVaR-GPA, on a logarithmic axis. The models and the error
metrics are given in \Cref{Sec:Results}, and the underlying error values in
\Cref{tab:studentt}. The only increase in error
occurs for the tail error of TTF, whose pre-trained tail error is already
small.}
\label{fig:ER_student_t_nu1}
\end{figure}

\Cref{fig:neal_pairs} presents samples from the seven pre-trained models for
\textbf{Neal's funnel distribution}, before and after fine-tuning. In each case,
CVaR-GPA recovers the characteristic funnel geometry, a key feature underlying
Neal's funnel's use as a challenging benchmark, from pre-trained samples that fail
to accurately reproduce this structure.  CVaR-GPA operates directly on the joint
target distribution, through the Lipschitz-regularized KL divergence component of
its loss functional, as opposed to learning each marginal separately, and is thus
expected to perform well on anisotropic targets with correlated marginal
distributions, as demonstrated empirically by \Cref{fig:neal_pairs}. The corresponding global and tail errors are shown in \Cref{fig:ER_neal}. Both
errors are reduced for every pre-trained model. While the pre-trained errors
span two orders of magnitude across the seven models, reflecting how differently
the models capture the funnel geometry, the fine-tuned errors concentrate within
a factor of less than two of one another. Per-dimension
error reductions are in Appendix~\ref{APP_NEAL}.

\begin{figure}[htbp]
\centering\scriptsize\setlength{\tabcolsep}{1.5pt}
\begin{tabular}{@{}cccc@{}}
pre-trained & fine-tuned & pre-trained & fine-tuned\\
\includegraphics[width=0.24\linewidth]{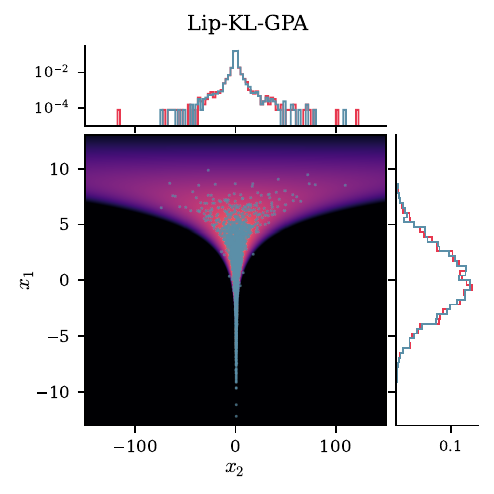} &
\includegraphics[width=0.24\linewidth]{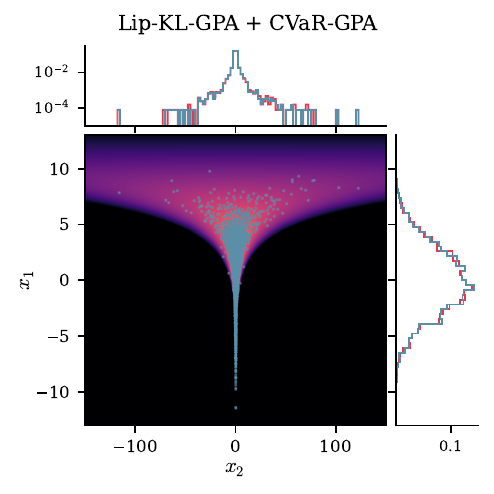} &
\includegraphics[width=0.24\linewidth]{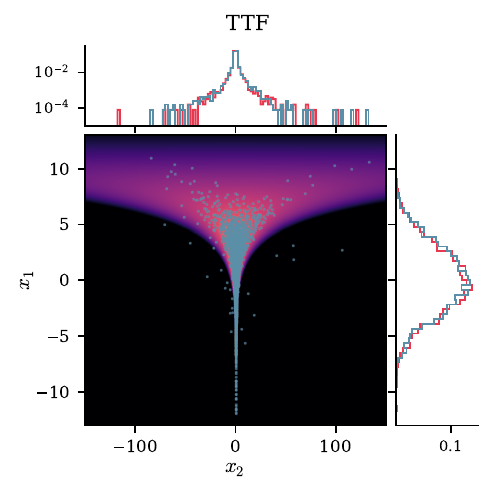} &
\includegraphics[width=0.24\linewidth]{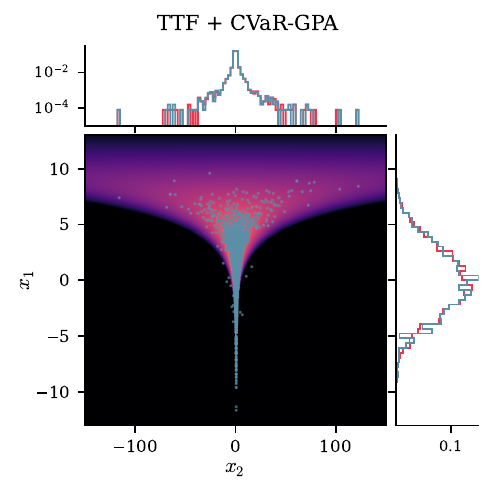}\\[4pt]
pre-trained & fine-tuned & pre-trained & fine-tuned\\
\includegraphics[width=0.24\linewidth]{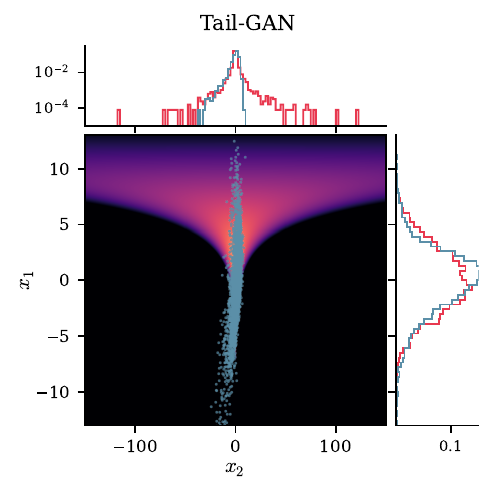} &
\includegraphics[width=0.24\linewidth]{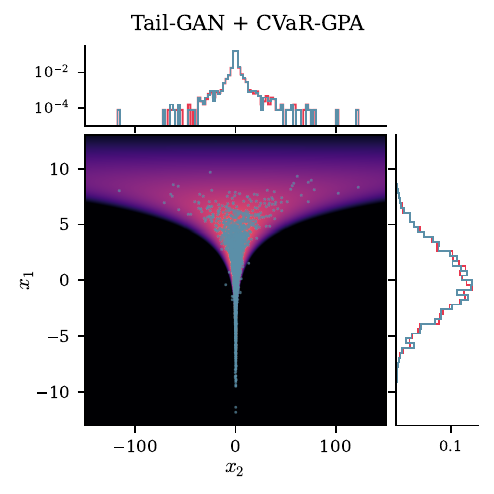} &
\includegraphics[width=0.24\linewidth]{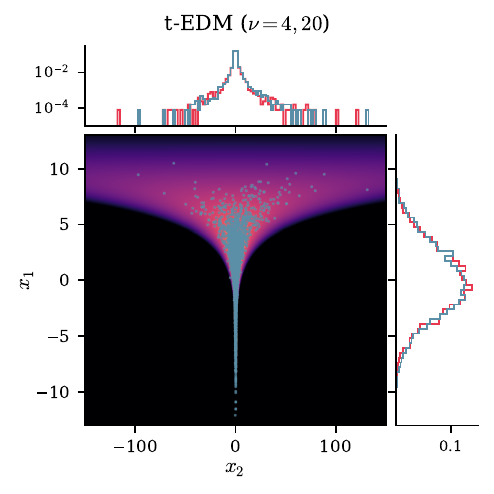} &
\includegraphics[width=0.24\linewidth]{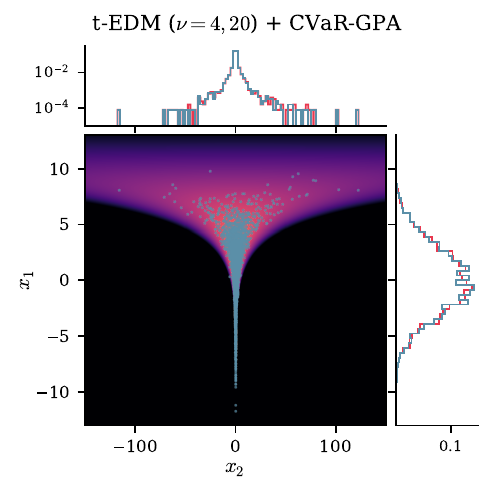}\\[4pt]
pre-trained & fine-tuned & pre-trained & fine-tuned\\
\includegraphics[width=0.24\linewidth]{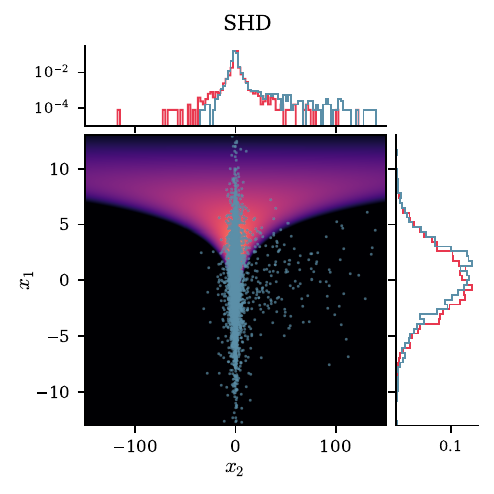} &
\includegraphics[width=0.24\linewidth]{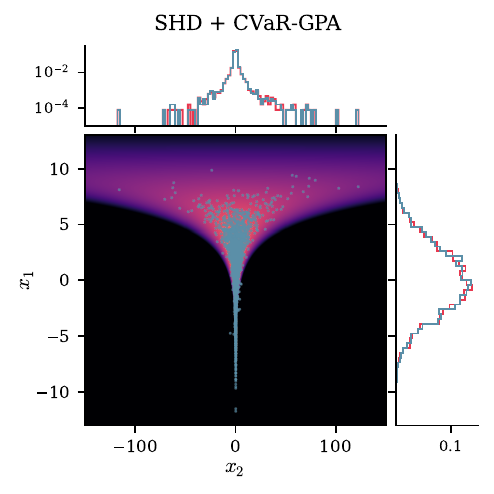} &
\includegraphics[width=0.24\linewidth]{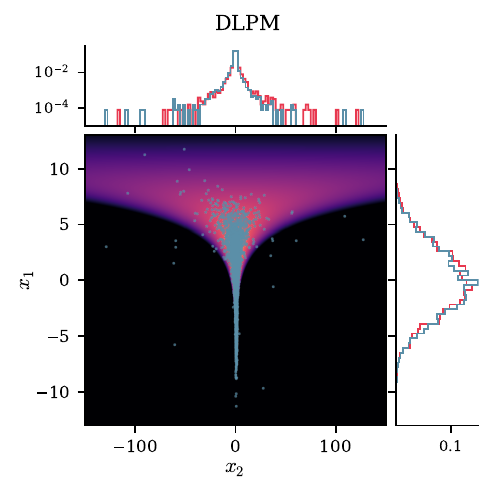} &
\includegraphics[width=0.24\linewidth]{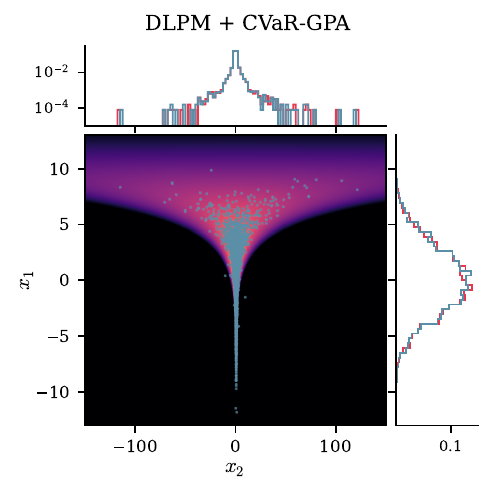}\\[4pt]
pre-trained & fine-tuned & \multicolumn{2}{c}{}\\
\includegraphics[width=0.24\linewidth]{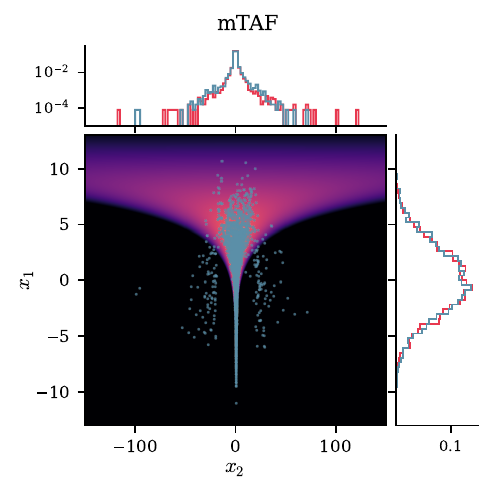} &
\includegraphics[width=0.24\linewidth]{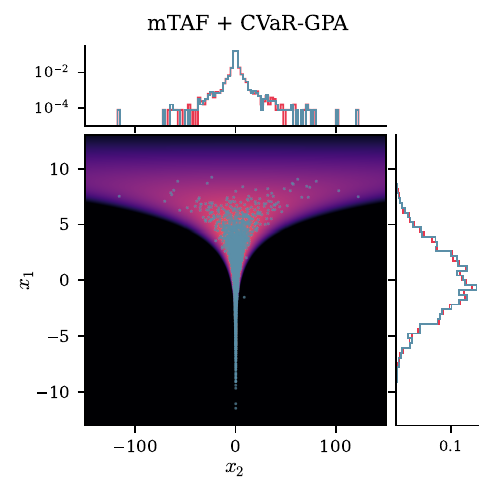} &
\multicolumn{2}{c}{\includegraphics[width=0.28\linewidth]{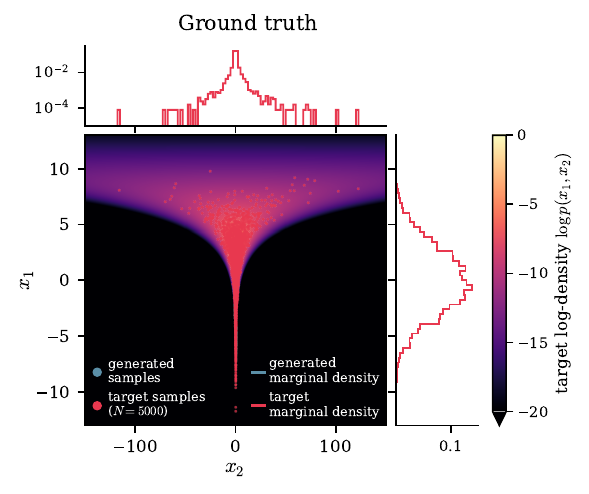}}\\
\end{tabular}
\caption{Neal’s funnel distribution: pre-trained and CVaR-GPA fine-tuned samples, given by blue particles, for the seven pre-trained models.  The target log-density is shown in the background with the shared color scale given by the target log density legend (best viewed when zoomed in).}
\label{fig:neal_pairs}
\end{figure}

\begin{figure}[htbp]
\centering
\includegraphics[width=0.9\linewidth]{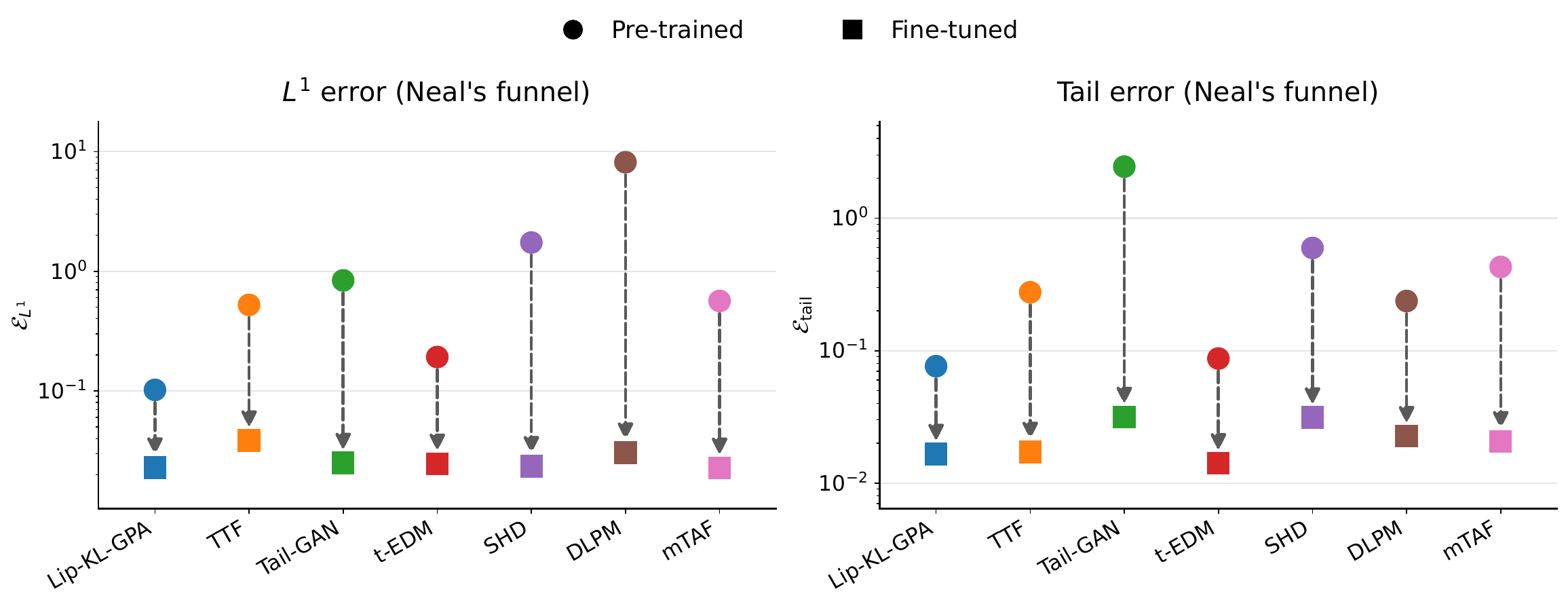}
\caption{Neal's funnel distribution: global $L^1$ error $\mathcal{E}_{L^1}$
(left) and tail error $\mathcal{E}_{\mathrm{tail}}$ (right) of the seven
pre-trained models before (circles) and after (squares) fine-tuning via
CVaR-GPA, on a logarithmic axis. Both errors are reduced for every pre-trained
model, and the fine-tuned errors are of comparable magnitude across the seven
models, despite pre-trained errors spanning more than two orders of magnitude.
The underlying error values are reported in \Cref{tab:headline}.}
\label{fig:ER_neal}
\end{figure}

 On the two \textbf{high-dimensional real-world targets}, fine-tuning improves
both global and tail error of every pre-trained model
(\Cref{tab:headline:summary}). \Cref{fig:ER_ff25,fig:ER_streamflow} show the geometric means of both global and tail errors (taken over the marginals), before and after fine-tuning, on
Fama--French 25 and on the daily streamflow, respectively. On Fama--French 25, the geometric-mean reductions are $18.9\times$ in global error and $16.1\times$ in tail error; on streamflow, they are $9.5\times$ and $5.4\times$, respectively. These datasets span markedly different heavy-tail regimes: the Fama--French marginals have estimated tail indices between $2.2$ and $3.2$ \citep{hill1975}, whereas streamflow extends down to $1.05$, with $51$ of $64$ marginals in the infinite-variance regime. Thus, the improvements persist from moderately heavy-tailed to extremely heavy-tailed targets. Per-marginal results for Fama--French 25 and streamflow are presented in Appendices~\ref{App:FF25} and~\ref{App:SF}. The tail statistic driving the CVaR velocity component is a scalar, so by Lemma \ref{Cor:CVaR:Estimation:App} its estimation does not degrade with $d$; this is what allows CVaR-GPA to be applied unchanged at $d=25$ and $d=64$, making the method scalable to high dimensions.

\begin{figure}[htbp]
\centering
\includegraphics[width=0.9\linewidth]{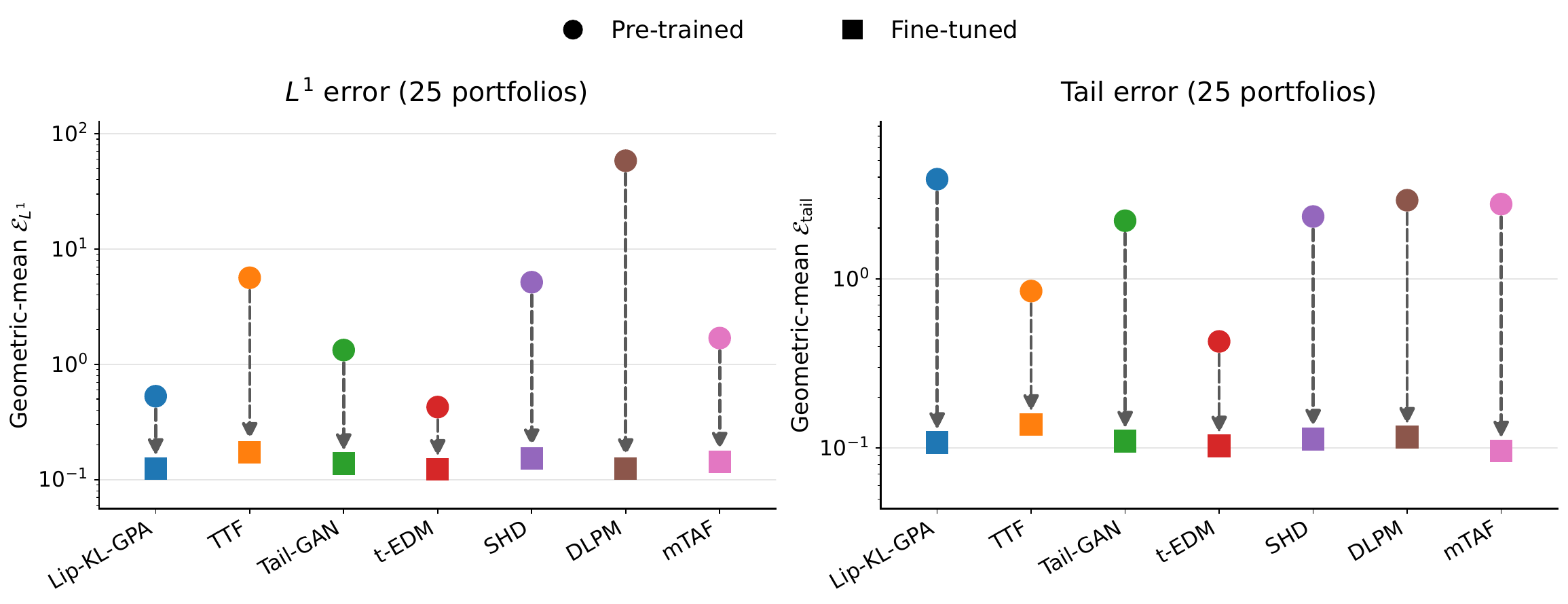}
\caption{Fama--French 25 portfolios ($d=25$): geometric means over the $25$
marginals of the global $L^1$ error $\mathcal{E}_{L^1}$ (left) and of the tail
error $\mathcal{E}_{\mathrm{tail}}$ (right), for the seven pre-trained models
before (circles) and after (squares) fine-tuning via CVaR-GPA, on a logarithmic
axis. Both errors are reduced for every pre-trained model. The underlying error
values are reported in \Cref{tab:headline}, and the per-portfolio errors in
\Cref{tab:ff25_l1,tab:ff25_tail}.}
\label{fig:ER_ff25}
\end{figure}

\begin{figure}[htbp]
\centering
\includegraphics[width=0.9\linewidth]{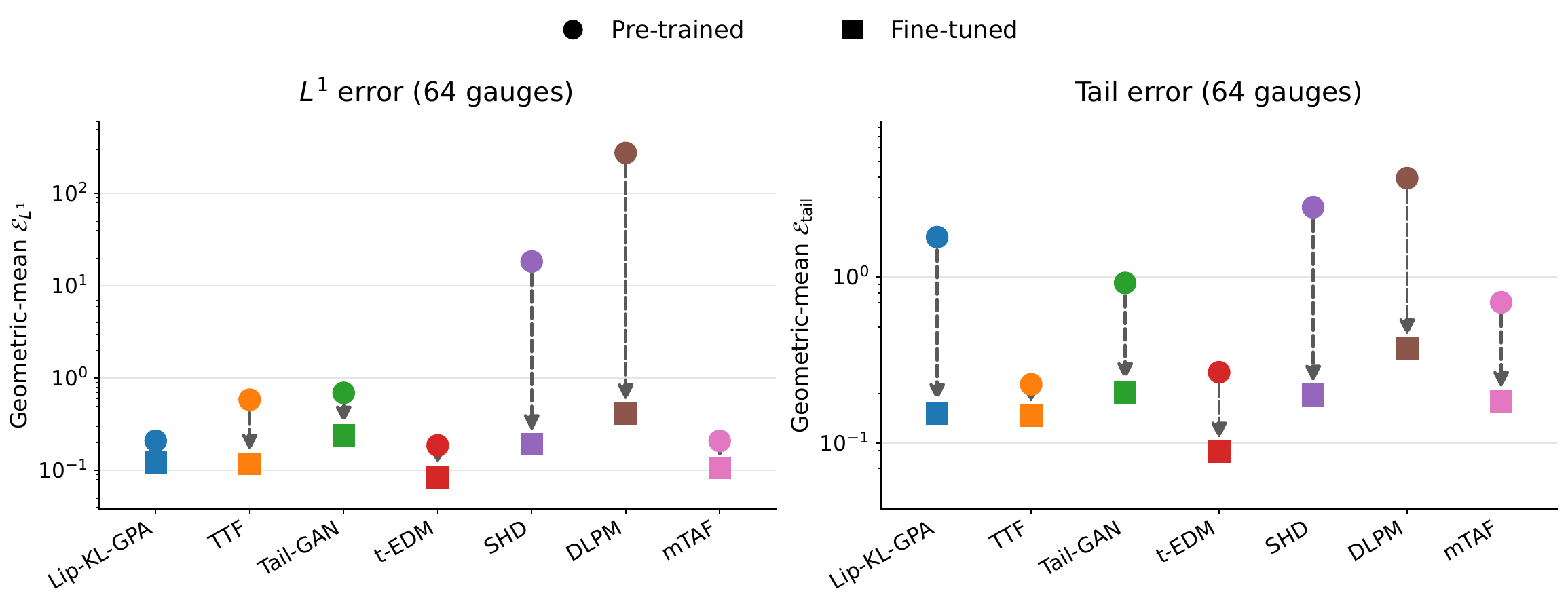}
\caption{Daily streamflow in the Ohio River basin ($d=64$) \citep{usgs_nwis}:
geometric means over the $64$ marginals of the global $L^1$ error
$\mathcal{E}_{L^1}$ (left) and of the tail error $\mathcal{E}_{\mathrm{tail}}$
(right), for the seven pre-trained models before (circles) and after (squares)
fine-tuning via CVaR-GPA, on a logarithmic axis. Both errors are reduced for
every pre-trained model, by geometric-mean factors of $9.5\times$ globally and
$5.4\times$ in the tail. The underlying error values are reported in
\Cref{tab:headline}, and the per-gauge errors in \Cref{tab:sf_l1,tab:sf_tail}.}
\label{fig:ER_streamflow}
\end{figure}
\section{Conclusion}
\label{Sec:Conclusions}
We develop CVaR-GPA, a tail-agnostic generative particle algorithm for fine-tuning pre-trained models toward heavy-tailed targets. It is built on the Wasserstein gradient flow of the Lipschitz-regularized KL divergence penalized by a smoothed Conditional Value-at-Risk discrepancy. The loss functional is designed so that the resulting flow satisfies properties \PropI--\PropIII. With hyperparameters fixed across all experiments and not tuned to particular targets or pre-trained models, CVaR-GPA improves the accuracy of pre-trained models on heavy-tailed targets with heterogeneous tail behaviors in $55$ out of $56$ dataset--model--error metric combinations. The global and tail geometric-mean error reductions are $14.0\times$ and $9.8\times$, respectively. Each of the three properties \PropI--\PropIII{} is reflected in these results.
Since the flow is initialized directly from pre-trained samples \PropI, CVaR-GPA is architecture-agnostic, which explains its robustness across all seven pre-trained models. The CVaR penalty keeps the velocity field active in the tail, even for high-dimensional target distributions \PropII, allowing CVaR-GPA to scale unchanged to the high-dimensional real-world targets with $d=25$ and $d=64$. The target-adaptive velocity field \eqref{Eq:ramp:vel} implicitly determines, for each target, the depth of the transport map required to learn it \PropIII. It thereby removes the need for tail-index-specific tuning across the four datasets with heterogeneous tail behaviors. 

While the radial risk function and the smoothed CVaR are natural choices for learning heavy-tailed targets, the same design principle, embedding desired properties into the flow through the choice of the loss functional, applies to other couplings of (possibly anisotropic) risk functions $g$ with risk measures (in particular spectral risk measures \citep{acerbi2002portfolio}, whose weight functions can be tailored to achieve task-specific goals) for specific generative tasks, such as learning multi-modal target distributions. Moreover, the cost of re-running the particle flow for new targets can be amortized by adapting to the CVaR-penalized flow the force-matching strategies developed in \cite{cheng2025profet} to learn continuous gene expression dynamics from GPA trajectories.

\section*{Funding}
This work was supported in part by the Air Force Office of Scientific
Research (AFOSR grant FA9550-21-1-0354) (T.G., H.G., Z.C., M.K., L.R.-B.) and by
the National Science Foundation (NSF grants DMS-2307115 and DMS-2606221 (M.K.,
L.R.-B.); NSF grant DMS-2606084 (Z.C.)).

\bibliographystyle{abbrv}
\bibliography{references}

@inproceedings{jaini2020tails,
  title={Tails of {L}ipschitz triangular flows},
  author={Jaini, Priyank and Kobyzev, Ivan and Yu, Yaoliang and Brubaker, Marcus},
  booktitle={International Conference on Machine Learning},
  pages={4673--4681},
  year={2020},
  publisher={PMLR}
}

@article{Rock2002CVaR,
  title={Conditional value-at-risk for general loss distributions},
  author={Rockafellar, R Tyrrell and Uryasev, Stanislav},
  journal={Journal of banking \& Finance},
  volume={26},
  number={7},
  pages={1443--1471},
  year={2002},
  publisher={Elsevier}
}

@article{chen2025robust,
  title={Robust generative learning with {L}ipschitz-regularized $\alpha$-divergences allows minimal assumptions on target distributions},
  author={Chen, Ziyu and Gu, Hyemin and Katsoulakis, Markos A and Rey-Bellet, Luc and Zhu, Wei},
  journal={Information and Inference: A Journal of the IMA},
  volume={14},
  number={4},
  pages={iaaf028},
  year={2025},
  publisher={Oxford University Press}
}

@article{Allouche2022EV-GAN,
  title={{EV}-{GAN}: Simulation of extreme events with {ReLU} neural networks},
  author={Allouche, Micha{\"e}l and Girard, St{\'e}phane and Gobet, Emmanuel},
  journal={Journal of Machine Learning Research},
  volume={23},
  number={150},
  pages={1--39},
  year={2022}
}

@inproceedings{
pandey2025_heavytaildiffusion,
title={Heavy-Tailed Diffusion Models},
author={Kushagra Pandey and Jaideep Pathak and Yilun Xu and Stephan Mandt and Michael Pritchard and Arash Vahdat and Morteza Mardani},
booktitle={International Conference on Learning Representations},
year={2025}
}

@inproceedings{Chiang2021ParetoGAN,
  title={{P}areto {GAN}: Extending the representational power of gans to heavy-tailed distributions},
  author={Huster, Todd and Cohen, Jeremy and Lin, Zinan and Chan, Kevin and Kamhoua, Charles and Leslie, Nandi O and Chiang, Cho-Yu Jason and Sekar, Vyas},
  booktitle={International Conference on Machine Learning},
  pages={4523--4532},
  year={2021},
  publisher={PMLR}
}

@inproceedings{guan2025mirrorflowmatchingheavytailed,
  title={Mirror flow matching with heavy-tailed priors for generative modeling on convex domains},
  author={Guan, Yunrui and Balasubramanian, Krishna and Ma, Shiqian},
  booktitle={International Conference on Learning Representations},
  year={2026}
}

@article{Birrell2022fGamma,
  author  = {Jeremiah Birrell and Paul Dupuis and Markos A. Katsoulakis and Yannis Pantazis and Luc Rey-Bellet},
  title   = {(f,{G}amma)-Divergences: Interpolating between f-Divergences and Integral Probability Metrics},
  journal = {Journal of Machine Learning Research},
  year    = {2022},
  volume  = {23},
  number  = {39},
  pages   = {1--70}
}

@article{gu2024lipschitz,
  title={{L}ipschitz-regularized gradient flows and generative particle algorithms for high-dimensional scarce data},
  author={Gu, Hyemin and Birmpa, Panagiota and Pantazis, Yannis and Rey-Bellet, Luc and Katsoulakis, Markos A},
  journal={SIAM Journal on Mathematics of Data Science},
  volume={6},
  number={4},
  pages={1205--1235},
  year={2024},
  publisher={SIAM}
}

@article{albrecher2006ruin,
  title={Ruin probabilities and aggregrate claims distributions for shot noise {C}ox processes},
  author={Albrecher, Hansj{\"o}rg and Asmussen, S{\o}ren},
  journal={Scandinavian Actuarial Journal},
  volume={2006},
  number={2},
  pages={86--110},
  year={2006},
  publisher={Taylor \& Francis}
}

@article{embrechts1982estimates,
  title={Estimates for the probability of ruin with special emphasis on the possibility of large claims},
  author={Embrechts, Paul and Veraverbeke, No{\"e}l},
  journal={Insurance: Mathematics and Economics},
  volume={1},
  number={1},
  pages={55--72},
  year={1982},
  publisher={Elsevier}
}

@book{grossi2005catastrophe,
  title={Catastrophe modeling: a new approach to managing risk},
  author={Grossi, Patricia and Kunreuther, Howard},
  volume={25},
  year={2005},
  publisher={Springer Science \& Business Media}
}

@article{neal2003slice,
  title={Slice sampling},
  author={Neal, Radford M},
  journal={The Annals of Statistics},
  volume={31},
  number={3},
  pages={705--767},
  year={2003},
  publisher={Institute of Mathematical Statistics}
}

@article{cirillo2020tail,
  title={Tail risk of contagious diseases},
  author={Cirillo, Pasquale and Taleb, Nassim Nicholas},
  journal={Nature Physics},
  volume={16},
  number={6},
  pages={606--613},
  year={2020},
  publisher={Nature Publishing Group UK London}
}

@InProceedings{hickling2024flexible,
  title = 	 {Flexible Tails for Normalizing Flows},
  author =       {Hickling, Tennessee and Prangle, Dennis},
  booktitle = 	 {International Conference on Machine Learning},
  year = 	 {2025},
  series = 	 {Proceedings of Machine Learning Research},
  publisher =    {PMLR},
}

@article{allouche2026exceedgan,
  title={{ExceedGAN}: simulation above extreme thresholds using Generative Adversarial Networks},
  author={Allouche, Micha{\"e}l and Girard, St{\'e}phane and Gobet, Emmanuel},
  journal={Extremes},
  pages={1--23},
  year={2026},
  publisher={Springer}
}

@inproceedings{bhatia2021exgan,
  title={{ExGAN}: Adversarial generation of extreme samples},
  author={Bhatia, Siddharth and Jain, Arjit and Hooi, Bryan},
  booktitle={Proceedings of the AAAI Conference on Artificial Intelligence},
  volume={35},
  number={8},
  pages={6750--6758},
  year={2021}
}

@article{cont2026tail,
  title={{Tail-GAN}: Learning to simulate tail risk scenarios},
  author={Cont, Rama and Cucuringu, Mihai and Xu, Renyuan and Zhang, Chao},
  journal={Management Science},
  year={2025},
  publisher={INFORMS}
}

@article{acerbi2001expected,
  title={Expected shortfall as a tool for financial risk management},
  author={Acerbi, Carlo and Nordio, Claudio and Sirtori, Carlo},
  journal={arXiv preprint cond-mat/0102304},
  year={2001}
}

@article{kishida2023risk,
  title={Risk-aware stability, ultimate boundedness, and positive invariance},
  author={Kishida, Masako},
  journal={IEEE Transactions on Automatic Control},
  volume={69},
  number={1},
  pages={681--688},
  year={2024},
  publisher={IEEE}
}

@article{acerbi2002portfolio,
  title={Portfolio optimization with spectral measures of risk},
  author={Acerbi, Carlo and Simonetti, Prospero},
  journal={arXiv preprint cond-mat/0203607},
  year={2002}
}

@article{dupuis2022formulation,
  title={Formulation and properties of a divergence used to compare probability measures without absolute continuity},
  author={Dupuis, Paul and Mao, Yixiang},
  journal={ESAIM: Control, Optimisation and Calculus of Variations},
  volume={28},
  pages={10},
  year={2022},
  publisher={EDP Sciences}
}

@article{wang2026efficient,
  title={Efficient Tail-Aware Generative Optimization via Flow Model Fine-Tuning},
  author={Wang, Zifan and De Santi, Riccardo and Mo, Xiaoyu and Zavlanos, Michael M and Krause, Andreas and Johansson, Karl H},
  journal={arXiv preprint arXiv:2602.16796},
  year={2026}
}

@inproceedings{NEURIPS2025_10715deb,
      title={Flow Density Control: Generative Optimization Beyond Entropy-Regularized Fine-Tuning}, 
      author={Riccardo De Santi and Marin Vlastelica and Ya-Ping Hsieh and Zebang Shen and Niao He and Andreas Krause},
      year={2025},
      booktitle={Advances in Neural Information Processing Systems}
}

@inproceedings{
chaudhary2024risk,
title={Risk-Averse Fine-tuning of Large Language Models},
author={Sapana Chaudhary and Ujwal Dinesha and Dileep Kalathil and Srinivas Shakkottai},
booktitle={Advances in Neural Information Processing Systems},
 year={2024},
}

@article{FAMA19933,
title = {Common risk factors in the returns on stocks and bonds},
journal = {Journal of Financial Economics},
volume = {33},
number = {1},
pages = {3-56},
year = {1993},
doi = {10.1016/0304-405X(93)90023-5},
author = {Eugene F. Fama and Kenneth R. French}
}

@inproceedings{miyato2018spectral,
  title={Spectral normalization for generative adversarial networks},
  author={Takeru Miyato and Toshiki Kataoka and Masanori Koyama and Yuichi Yoshida},
  booktitle={International Conference on Learning Representations},
  year={2018}
}

@book{clarke1990optimization,
  title={Optimization and nonsmooth analysis},
  author={Clarke, Frank H},
  year={1990},
  publisher={Society for industrial and Applied Mathematics}
}

@article{jordan1998variational,
  title={The variational formulation of the {F}okker--{P}lanck equation},
  author={Jordan, Richard and Kinderlehrer, David and Otto, Felix},
  journal={SIAM Journal on Mathematical Analysis},
  volume={29},
  number={1},
  pages={1--17},
  year={1998},
  publisher={SIAM}
}

@article{otto2001geometry,
author = {Felix Otto},
title={The geometry of dissipative evolution equations: the porous medium equation},
journal = {Communications in Partial Differential Equations},
volume = {26},
number = {1--2},
pages = {101--174},
year = {2001},
publisher = {Taylor \& Francis},
doi = {10.1081/PDE-100002243},
}

@article{hill1975,
author = {Bruce M. Hill},
title = {A Simple General Approach to Inference About the Tail of a Distribution},
volume = {3},
journal = {The Annals of Statistics},
number = {5},
publisher = {Institute of Mathematical Statistics},
pages = {1163--1174},
year = {1975},
doi = {10.1214/aos/1176343247},
}

@article{kantorovich1958space,
  title={On a space of totally additive functions},
  author = {Kantorovich, Leonid and Rubinstein, Gennady S.},
  journal={Vestnik Leningrad. Univ},
  volume={13},
  number={7},
  pages={52--59},
  year={1958},
  publisher={Allerton Press, Inc.}
}

@article{liu2024_heavytaileddiffusion,
  title={Learning to simulate from heavy-tailed distribution via diffusion model},
  author={Liu, Haoyu and Zhu, Tingyu and Jia, Nanshan and He, Jinghai and Zheng, Zeyu},
 journal = {Operations Research},
year = {2026},
doi = {10.1287/opre.2024.1450},
}

@article{kouri2016risk,
  title={Risk-averse {PDE}-constrained optimization using the conditional value-at-risk},
  author={Kouri, Drew P and Surowiec, Thomas M},
  journal={SIAM Journal on Optimization},
  volume={26},
  number={1},
  pages={365--396},
  year={2016},
  publisher={SIAM}
}

@article{fournier2015rate,
  title={On the rate of convergence in {W}asserstein distance of the empirical measure},
  author={Fournier, Nicolas and Guillin, Arnaud},
  journal={Probability Theory and Related Fields},
  volume={162},
  number={3},
  pages={707--738},
  year={2015},
  publisher={Springer}
}

@inproceedings{
shariatian2024denoisinglevyprobabilisticmodels,
title={Denoising {L\'{e}vy} Probabilistic Models},
author={Dario Shariatian and Umut Simsekli and Alain Oliviero Durmus},
booktitle={International Conference on Learning Representations},
year={2025},
}

@inproceedings{laszkiewicz2022marginal, title={Marginal Tail-Adaptive Normalizing Flows}, author={Laszkiewicz, Mike and Lederer, Johannes and Fischer, Asja}, booktitle={International Conference on Machine Learning}, pages = 	 {12020--12048}, year={2022}, organization={PMLR} }

@misc{usgs_nwis,
  author = {{U.S. Geological Survey}},
  title = {{USGS Water Data for the Nation}: {U.S. Geological Survey National Water Information System} database},
  year = {2026},
  note = {Accessed September, 2026, at \url{https://doi.org/10.5066/F7P55KJN}}
}

@article{merz2022understanding,
  title={Understanding heavy tails of flood peak distributions},
  author={Merz, Bruno and Basso, Stefano and Fischer, Svenja and Lun, David and Bl{\"o}schl, G{\"u}nter and Merz, Ralf and Guse, Bj{\"o}rn and Viglione, Alberto and Vorogushyn, Sergiy and Macdonald, Elena and others},
  journal={Water Resources Research},
  volume={58},
  number={6},
  pages={e2021WR030506},
  year={2022},
  publisher={Wiley Online Library}
}

@article{farazmand2019closed,
  title={Closed-loop adaptive control of extreme events in a turbulent flow},
  author={Farazmand, Mohammad and Sapsis, Themistoklis P},
  journal={Physical Review E},
  volume={100},
  number={3},
  pages={033110},
  year={2019},
  publisher={APS}
}

@article{lynn2024heavy,
  title={Heavy-tailed neuronal connectivity arises from {H}ebbian self-organization},
  author={Lynn, Christopher W and Holmes, Caroline M and Palmer, Stephanie E},
  journal={Nature Physics},
  volume={20},
  number={3},
  pages={484--491},
  year={2024},
  publisher={Nature Publishing Group UK London}
}

@article{cheng2025profet,
  title   = {{PROFET} Predicts Continuous Gene Expression Dynamics from {scRNA-seq} Data to Elucidate Heterogeneity of Cancer Treatment Responses},
  author  = {Cheng, Yu-Chen and others},
  journal = {bioRxiv},
  year    = {2025},
  doi     = {10.1101/2025.06.27.662030}
}

@article{chang2026extreme,
  title={Extreme event aware ($\eta$-) learning},
  author={Chang, Kai and Sapsis, Themistoklis P},
  journal={Nature Communications},
  year={2026},
  publisher={Nature Publishing Group UK London}
}

@book{embrechts1997modelling,
  author    = {Embrechts, Paul and Kl{\"u}ppelberg, Claudia and Mikosch, Thomas},
  title     = {Modelling Extremal Events for Insurance and Finance},
  series    = {Applications of Mathematics},
  volume    = {33},
  publisher = {Springer},
  address   = {Berlin},
  year      = {1997}
}

@book{foss2013introduction,
  author    = {Foss, Sergey and Korshunov, Dmitry and Zachary, Stan},
  title     = {An Introduction to Heavy-Tailed and Subexponential Distributions},
  edition   = {2},
  series    = {Springer Series in Operations Research and Financial Engineering},
  publisher = {Springer},
  address   = {New York},
  year      = {2013},
  doi       = {10.1007/978-1-4614-7101-1}
}

@book{nair2022fundamentals,
  author    = {Nair, Jayakrishnan and Wierman, Adam and Zwart, Bert},
  title     = {The Fundamentals of Heavy Tails: Properties, Emergence, and Estimation},
  series    = {Cambridge Series in Statistical and Probabilistic Mathematics},
  volume    = {53},
  publisher = {Cambridge University Press},
  address   = {Cambridge},
  year      = {2022}
}
\clearpage
\appendix
\addtocontents{toc}{\protect\setcounter{tocdepth}{2}}
\begin{center}
  {\LARGE\bfseries APPENDIX}
\end{center}
\vspace{1.5em}
\renewcommand{\contentsname}{\normalsize\textsc{Contents}}
\tableofcontents
\renewcommand{\theequation}{\thesection\arabic{equation}}
\setcounter{equation}{0}
\makeatletter
\renewcommand{\@seccntformat}[1]{Appendix \csname the#1\endcsname: }
\makeatother

\section{Numerical Experiments}
\label{App:Numerical}

\subsection{Choice of the level $\alpha$}
\label{App:Alpha}
The velocity component $\hat v_{Q,h}$ from the CVaR penalty term in \eqref{Eq:ramp:vel} acts, through the weight $\sigma\big((\|x\|-q_{\alpha,h}^{Q,g})/h\big)$, essentially only on the particles with $\|x\|>q_{\alpha,h}^{Q,g}$, since $\sigma\big((\|x\|-q_{\alpha,h}^{Q,g})/h\big)$ is the smooth surrogate of the indicator function $\mathds{1}_{\{{\|x\|>}q_{\alpha,h}^{Q,g}\}}$. The level $\alpha$ needs to be extreme for two reasons. 

First, there are infinitely many distributions $Q$ for which $\CVaR^{Q,g}_{\alpha,h}=\CVaR^{\Ptar,g}_{\alpha,h}$. As we observe in \Cref{fig:alpha}, for moderate levels of $\alpha$ since the domain for CVaR is larger than that considered by an extreme level $\alpha$, the CVaR discrepancy $\Delta C_h(Q;\Ptar)$ nears $0$, while $Q$ is not close to $\Ptar$ in distribution. As $\alpha\uparrow1$, the averaging region shrinks to the extreme tail, and the constraint $\Delta C_h(Q;\Ptar)=0$ becomes a better constraint in capturing the distribution $\Ptar$. Second, the goal of the velocity component $\hat v_{Q,h}$ is to place additional weight on the under-sampled tail region, to compensate for the vanishing velocity  $-\nabla_x\phi^*$ from the Lipschitz-regularized KL divergence in that region. For a moderate level $\alpha$, $\hat v_{Q,h}$ displaces a large number of moderate-radius particles by comparable amounts, and thereby perturbs the part of the distribution that $\LipKL$ already fits. For an extreme $\alpha$, it acts only on the few particles that constitute the extreme tail, and leaves the bulk of the distribution to be learned by the Lipschitz-regularized KL divergence.

\begin{figure}[htbp]
\centering
\includegraphics[width=\linewidth]{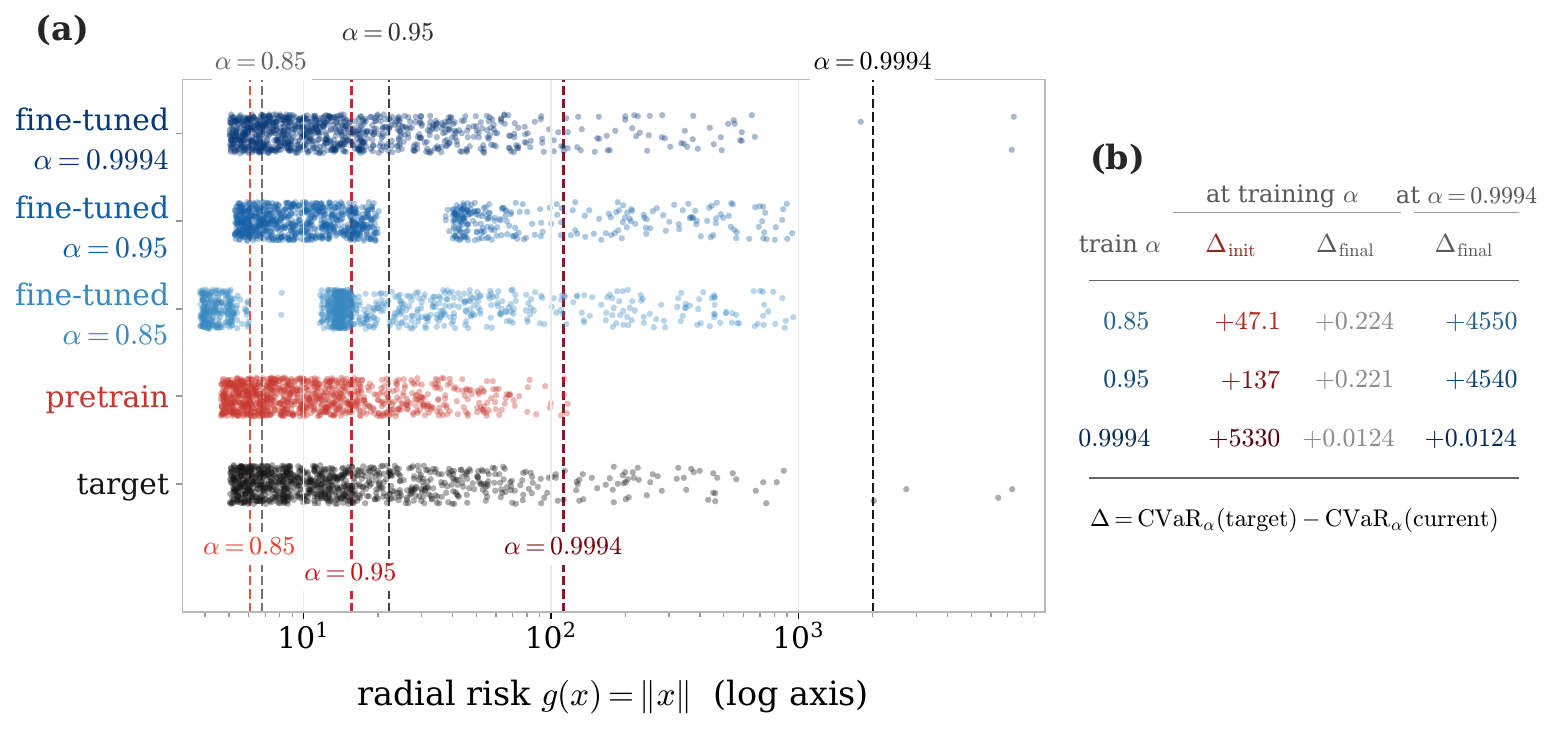}
\captionsetup{font=footnotesize}
\caption{2-dimensional Cauchy target. (a) The $1000$ largest radii $g(x)=\|x\|$ of the target (black), of the pre-trained model (red), and of the fine-tuned models at three levels (blue). Dashed lines: $\VaR_\alpha$ ($\alpha$-th quantile) of the target (labels on top) and of the pre-trained model (labels below). (b) Smoothed CVaR discrepancy $\Delta C_h(Q;\Ptar)=\CVaR^{\Ptar,g}_{\alpha,h}-\CVaR^{Q,g}_{\alpha,h}$ before and after fine-tuning at the training level $\alpha$, and after fine-tuning evaluated at $\alpha=0.9994$.}
\label{fig:alpha}
\end{figure}
\begin{table}[htbp]
\centering\small
\captionsetup{font=footnotesize}
\caption{2-dimensional Cauchy target, Lip-KL-GPA pre-trained model: joint errors for three levels $\alpha$, with $k=(1-\alpha)M$ the number of  generated particles above the level. }
\label{tab:alpha}
\begin{tabular}{lcccc}\toprule
 & \multicolumn{2}{c}{$\mathcal{E}_{L^1}$} & \multicolumn{2}{c}{$\mathcal{E}_{\mathrm{tail}}$}\\ \cmidrule(lr){2-3}\cmidrule(lr){4-5}
$\alpha$ & pre-trained & fine-tuned & pre-trained & fine-tuned\\ \midrule
$0.85$ ($k=750$) & 7.13 & {6.11} & 4.24 & {0.298}\\
$0.95$ ($k=250$) & 7.13 & {6.01} & 4.24 & {0.331}\\
$0.9994$ ($k=3$) & 7.13 & {1.08} & 4.24 & {0.0628}\\
\bottomrule\end{tabular}
\end{table}
\Cref{tab:alpha} and \Cref{fig:alpha} illustrate this phenomenon on the 2-dimensional Cauchy target, observed {empirically} when fine-tuning the Lip-KL-GPA pre-trained model at $\alpha\in\{0.85,0.95,0.9994\}$, that is, with $(1-\alpha)M\in\{750,250,3\}$ generated particles above the quantile, and with all other settings as in \Cref{Sec:Results}. At the two moderate levels, fine-tuning reduces the global error $\mathcal{E}_{L^1}$ only from $7.13$ to $6.11$ and $6.01$, respectively, whereas at $\alpha=0.9994$ it reduces it to $1.08$. The tail error $\mathcal{E}_{\mathrm{tail}}$ decreases at all three levels, but the extreme level attains $0.0628$, roughly five times smaller than the values $0.298$ and $0.331$ obtained at the moderate levels.

\subsection{Error values after fine-tuning}
\label{App_FULLtable}

\begin{table}[htbp]
\centering
\footnotesize
\setlength{\tabcolsep}{4pt}
\captionsetup{font=footnotesize}
\caption{Seven pre-trained generative models before and after fine-tuning with CVaR-GPA, all with the same hyperparameters. Factor: pre-trained error divided by fine-tuned error. Bold: error-reduction factor above $1$ (improvement).}
\label{tab:headline}
\begin{tabular}{l rrr rrr}\toprule
 & \multicolumn{3}{c}{$\mathcal{E}_{L^1}$} & \multicolumn{3}{c}{$\mathcal{E}_{\mathrm{tail}}$}\\ \cmidrule(lr){2-4}\cmidrule(lr){5-7}
pre-trained model & pre-trained & fine-tuned & factor & pre-trained & fine-tuned & factor\\
\midrule
\multicolumn{7}{l}{\textit{2-d Cauchy ($\nu=1$), joint $\|x\|$, $M{=}N{=}5000$}}\\[1pt]
Lip-KL-GPA & 7.13 & 1.08 & \textbf{6.6}$\times$ & 4.24 & 0.063 & \textbf{67.3}$\times$\\
TTF & 17.6 & 7.13 & \textbf{2.5}$\times$ & 0.138 & 0.149 & \textcolor{gray}{0.9$\times$}\\
Tail-GAN & 50.5 & 1.53 & \textbf{33.0}$\times$ & 4.95 & 0.0379 & \textbf{130.6}$\times$\\
t-EDM & 6.28 & 1.22 & \textbf{5.1}$\times$ & 0.294 & 0.0512 & \textbf{5.7}$\times$\\
SHD & 278 & 51.2 & \textbf{5.4}$\times$ & 2.8 & 1.02 & \textbf{2.7}$\times$\\
DLPM & 604 & 17.3 & \textbf{35.0}$\times$ & 0.659 & 0.419 & \textbf{1.6}$\times$\\
mTAF & 67.6 & 14.0 & \textbf{4.8}$\times$ & 1.81 & 0.347 & \textbf{5.2}$\times$\\
\midrule
\multicolumn{7}{l}{\textit{2-d  Neal's funnel, joint $\|x\|$, $M{=}N{=}5000$}}\\[1pt]
Lip-KL-GPA & 0.102 & 0.0229 & \textbf{4.5}$\times$ & 0.0763 & 0.0166 & \textbf{4.6}$\times$\\
TTF & 0.524 & 0.0385 & \textbf{13.6}$\times$ & 0.276 & 0.0172 & \textbf{16.0}$\times$\\
Tail-GAN & 0.839 & 0.025 & \textbf{33.6}$\times$ & 2.45 & 0.0316 & \textbf{77.5}$\times$\\
t-EDM& 0.192 & 0.0246 & \textbf{7.8}$\times$ & 0.0872 & 0.0141 & \textbf{6.2}$\times$\\
SHD & 1.74 & 0.0235 & \textbf{73.9}$\times$ & 0.596 & 0.0313 & \textbf{19.0}$\times$\\
DLPM & 8.15 & 0.0305 & \textbf{267.4}$\times$ & 0.237 & 0.0226 & \textbf{10.5}$\times$\\
mTAF & 0.565 & 0.0228 & \textbf{24.8}$\times$ & 0.429 & 0.0207 & \textbf{20.7}$\times$\\
\midrule
\multicolumn{7}{l}{\textit{Fama--French 25, per-portfolio {geometric} mean, $M{=}N{=}1182$}}\\[1pt]
Lip-KL-GPA & {0.529} & {0.125} & \textbf{4.2}$\times$ & {3.89} & {0.108} & \textbf{{36.1}}$\times$\\
TTF & {5.64} & {0.173} & \textbf{{32.7}}$\times$ & {0.848} & {0.138} & \textbf{{6.2}}$\times$\\
Tail-GAN & {1.33} & {0.138} & \textbf{{9.6}}$\times$ & {2.21} & {0.11} & \textbf{{20.1}}$\times$\\
t-EDM & {0.426} & {0.123} & \textbf{3.5}$\times$ & {0.427} & {0.103} & \textbf{{4.1}}$\times$\\
SHD & {5.16} & {0.152} & \textbf{{34.0}}$\times$ & {2.34} & {0.113} & \textbf{{20.7}}$\times$\\
DLPM & {58.4} & {0.125} & \textbf{{468.2}}$\times$ & 2.92 & {0.116} & \textbf{{25.0}}$\times$\\
mTAF & {1.69} & {0.142} & \textbf{{11.9}}$\times$ & {2.77} & {0.0963} & \textbf{{28.8}}$\times$\\
\midrule
\multicolumn{7}{l}{\textit{Streamflow (64 gauges), per-gauge {geometric} mean, $M{=}N{=}7305$}}\\[1pt]
Lip-KL-GPA & {0.21} & {0.121} & \textbf{{1.7}}$\times$ & {1.74} & {0.151} & \textbf{{11.6}}$\times$\\
TTF & {0.585} & {0.118} & \textbf{{4.9}}$\times$ & {0.226} & {0.146} & \textbf{1.6}$\times$\\
Tail-GAN & {0.694} & {0.238} & \textbf{{2.9}}$\times$ & {0.921} & {0.202} & \textbf{{4.6}}$\times$\\
t-EDM & {0.187} & {0.0848} & \textbf{2.2}$\times$ & {0.267} & {0.0886} & \textbf{{3.0}}$\times$\\
SHD & {18.4} & {0.192} & \textbf{{95.9}}$\times$ & {2.63} & {0.195} & \textbf{{13.5}}$\times$\\
DLPM & {277} & {0.414} & \textbf{667.7}$\times$ & 3.94 & {0.371} & \textbf{{10.6}}$\times$\\
mTAF & {0.209} & {0.107} & \textbf{2.0}$\times$ & {0.704} & {0.179} & \textbf{{3.9}}$\times$\\
\bottomrule\end{tabular}
\end{table}
\Cref{tab:headline} reports the errors underlying \Cref{tab:headline:summary}. For Fama--French 25 and streamflow, the pre-trained and fine-tuned errors are geometric means over the $d$ marginals, so that the factor equals the geometric mean of the per-marginal error-reduction factors (up to rounding of the displayed errors).  The tail parameter of t-EDM is selected by the tail error of the pre-trained model over $\nu\in\{3,5,7\}$ ($\nu=3$ on the Cauchy and Fama--French targets, $\nu=5$ on the streamflow data); on Neal's funnel it uses the per-coordinate setting $\nu=(4,20)$ of its authors.

\subsubsection{Student-$t$ family}
\label{App:ER:Figures}

 \Cref{tab:studentt} reports results for the Student-$t$ targets with tail indices $\nu\in\{1,1.2,1.5,1.8\}$ underlying  \Cref{tab:studentt:factors:pernu}. 
\begin{table}[htbp]
\centering\small\setlength{\tabcolsep}{4.5pt}
\footnotesize
\captionsetup{font=footnotesize}
\caption{{2-dimensional isotropic Student-$t$ targets with $\nu\in\{{1,}1.2,1.5,1.8\}$: seven pre-trained generative models before and after fine-tuning with CVaR-GPA, all with the same hyperparameters. Factor: pre-trained error divided by fine-tuned error. Bold: error-reduction factor above $1$ (improvement).}}
\label{tab:studentt}
\begin{tabular}{l rrr rrr}\toprule
 & \multicolumn{3}{c}{$\mathcal{E}_{L^1}$} & \multicolumn{3}{c}{$\mathcal{E}_{\mathrm{tail}}$}\\ \cmidrule(lr){2-4}\cmidrule(lr){5-7}
pre-trained model & pre-trained & fine-tuned & factor & pre-trained & fine-tuned & factor\\
\midrule
\multicolumn{7}{l}{\textit{2-d Student-$t$ ($\nu=1$, Cauchy), joint $\|x\|$, $M{=}N{=}5000$}}\\[1pt]
Lip-KL-GPA & 7.13 & 1.08 & \textbf{6.6}$\times$ & 4.24 & 0.063 & \textbf{67.3}$\times$\\
TTF & 17.6 & 7.13 & \textbf{2.5}$\times$ & 0.138 & 0.149 & \textcolor{gray}{0.9$\times$}\\
Tail-GAN & 50.5 & 1.53 & \textbf{33.0}$\times$ & 4.95 & 0.0379 & \textbf{130.6}$\times$\\
t-EDM & 6.28 & 1.22 & \textbf{5.1}$\times$ & 0.294 & 0.0512 & \textbf{5.7}$\times$\\
SHD & 278 & 51.2 & \textbf{5.4}$\times$ & 2.8 & 1.02 & \textbf{2.7}$\times$\\
DLPM & 604 & 17.3 & \textbf{35.0}$\times$ & 0.659 & 0.419 & \textbf{1.6}$\times$\\
mTAF & 67.6 & 14.0 & \textbf{4.8}$\times$ & 1.81 & 0.347 & \textbf{5.2}$\times$\\
\midrule
\multicolumn{7}{l}{\textit{2-d Student-$t$ ($\nu=1.2$), joint $\|x\|$, $M{=}N{=}5000$}}\\[1pt]
Lip-KL-GPA & 1.65 & 0.302 & \textbf{5.5}$\times$ & 2.06 & 0.0324 & \textbf{63.6}$\times$\\
TTF & 223 & 0.375 & \textbf{594.9}$\times$ & 0.126 & 0.0661 & \textbf{1.9}$\times$\\
Tail-GAN & 4.29 & 0.309 & \textbf{13.9}$\times$ & 14.3 & 0.0377 & \textbf{378.1}$\times$\\
t-EDM & 0.872 & 0.0945 & \textbf{9.2}$\times$ & 0.137 & 0.0263 & \textbf{5.2}$\times$\\
SHD & 5.74 & 1.91 & \textbf{3.0}$\times$ & 0.352 & 0.186 & \textbf{1.9}$\times$\\
DLPM & 92.3 & 0.523 & \textbf{176.5}$\times$ & 0.321 & 0.0822 & \textbf{3.9}$\times$\\
mTAF & 3.02 & 0.422 & \textbf{7.1}$\times$ & 0.937 & 0.0733 & \textbf{12.8}$\times$\\
\midrule
\multicolumn{7}{l}{\textit{2-d Student-$t$ ($\nu=1.5$), joint $\|x\|$, $M{=}N{=}5000$}}\\[1pt]
Lip-KL-GPA & 0.381 & 0.0184 & \textbf{20.7}$\times$ & 0.981 & 0.017 & \textbf{57.6}$\times$\\
TTF & 10.3 & 0.0711 & \textbf{144.4}$\times$ & 0.145 & 0.0395 & \textbf{3.7}$\times$\\
Tail-GAN & 1.64 & 0.0527 & \textbf{31.1}$\times$ & 19.1 & 0.0431 & \textbf{442.6}$\times$\\
t-EDM& 0.371 & 0.0374 & \textbf{9.9}$\times$ & 0.0856 & 0.0281 & \textbf{3.0}$\times$\\
SHD & 15.0 & 10.4 & \textbf{1.4}$\times$ & 1.32 & 0.985 & \textbf{1.3}$\times$\\
DLPM & 25.6 & 0.0744 & \textbf{343.6}$\times$ & 0.289 & 0.0247 & \textbf{11.7}$\times$\\
mTAF & 0.642 & 0.0614 & \textbf{10.5}$\times$ & 0.483 & 0.0281 & \textbf{17.2}$\times$\\
\midrule
\multicolumn{7}{l}{\textit{2-d Student-$t$ ($\nu=1.8$), joint $\|x\|$, $M{=}N{=}5000$}}\\[1pt]
Lip-KL-GPA & 0.106 & 0.0184 & \textbf{5.8}$\times$ & 0.06 & 0.0303 & \textbf{2.0}$\times$\\
TTF & 0.339 & 0.0251 & \textbf{13.5}$\times$ & 0.221 & 0.0254 & \textbf{8.7}$\times$\\
Tail-GAN & 0.378 & 0.0478 & \textbf{7.9}$\times$ & 2.2 & 0.0369 & \textbf{59.8}$\times$\\
t-EDM& 0.22 & 0.0214 & \textbf{10.3}$\times$ & 0.107 & 0.0242 & \textbf{4.4}$\times$\\
SHD & 1.44 & 0.0278 & \textbf{51.7}$\times$ & 0.477 & 0.0214 & \textbf{22.2}$\times$\\
DLPM & 12.8 & 0.0466 & \textbf{274.7}$\times$ & 0.272 & 0.0381 & \textbf{7.1}$\times$\\
mTAF & 0.257 & 0.0316 & \textbf{8.1}$\times$ & 0.242 & 0.015 & \textbf{16.2}$\times$\\
\bottomrule
\end{tabular}
\end{table}
\Cref{fig:ER_student_t_nu1} in \Cref{Sec:Results} shows the global and tail
errors before and after fine-tuning on the 2-dimensional isotropic Student-$t$
target with $\nu=1$. \Cref{fig:ER_student_t_nu12,fig:ER_student_t_nu15,%
fig:ER_student_t_nu18} show the same comparison for the remaining tail indices
$\nu\in\{1.2,1.5,1.8\}$, with all four targets fine-tuned using the same
hyperparameters and without any estimate of $\nu$ as an input. For these three
targets, both errors are reduced for every pre-trained model; the sole increase
across the entire Student-$t$ family is the tail error of TTF at $\nu=1$
(\Cref{tab:studentt}). The heaviest-tailed target is also the hardest: the
fine-tuned errors at $\nu=1.8$ lie within a factor of three of one another
across the seven models, whereas at $\nu=1$ they still span two orders of
magnitude, with SHD and mTAF remaining the least accurate. SHD is the one model
whose pre-trained samples remain far from the target after fine-tuning at
$\nu=1.2$ and $\nu=1.5$ as well, its global error being reduced by factors of
$3.0\times$ and $1.4\times$ only. That the error reductions nonetheless hold
across the family, at a fixed hyperparameter setting, is a consequence of the
target-adaptive velocity field \eqref{Eq:ramp:vel}: the CVaR discrepancy
$\Delta C_h(Q;\Ptar)$ sets both the magnitude of the correction and the terminal
time of the flow, so a heavier tail is met with a longer and stronger flow rather
than with a retuned algorithm \PropIII.

\begin{figure}[htbp]
\centering
\includegraphics[width=0.9\linewidth]{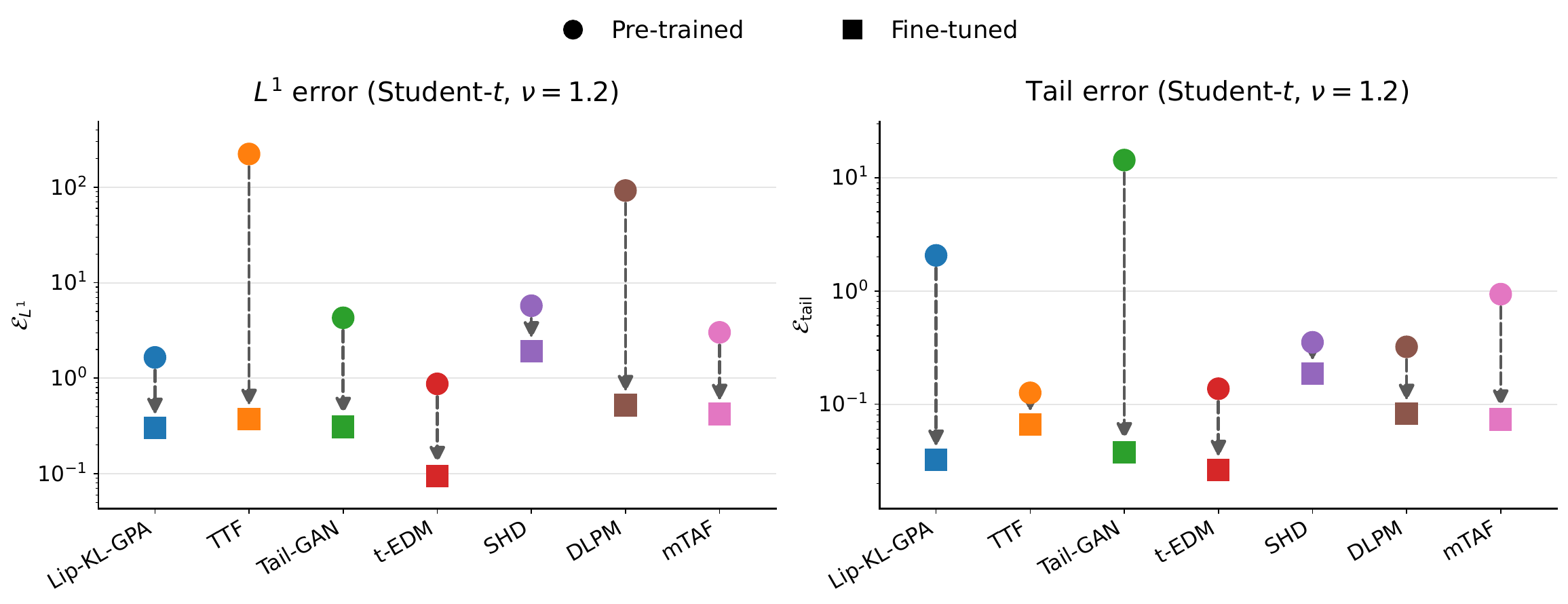}
\caption{2-dimensional Student-$t$ target with $\nu=1.2$: global $L^1$ error
$\mathcal{E}_{L^1}$ (left) and tail error $\mathcal{E}_{\mathrm{tail}}$ (right)
of the seven pre-trained models before (circles) and after (squares)
fine-tuning via CVaR-GPA, on a logarithmic axis. Both errors are reduced for
every pre-trained model. The underlying error values are reported in
\Cref{tab:studentt}.}
\label{fig:ER_student_t_nu12}
\end{figure}
 
\begin{figure}[htbp]
\centering
\includegraphics[width=0.9\linewidth]{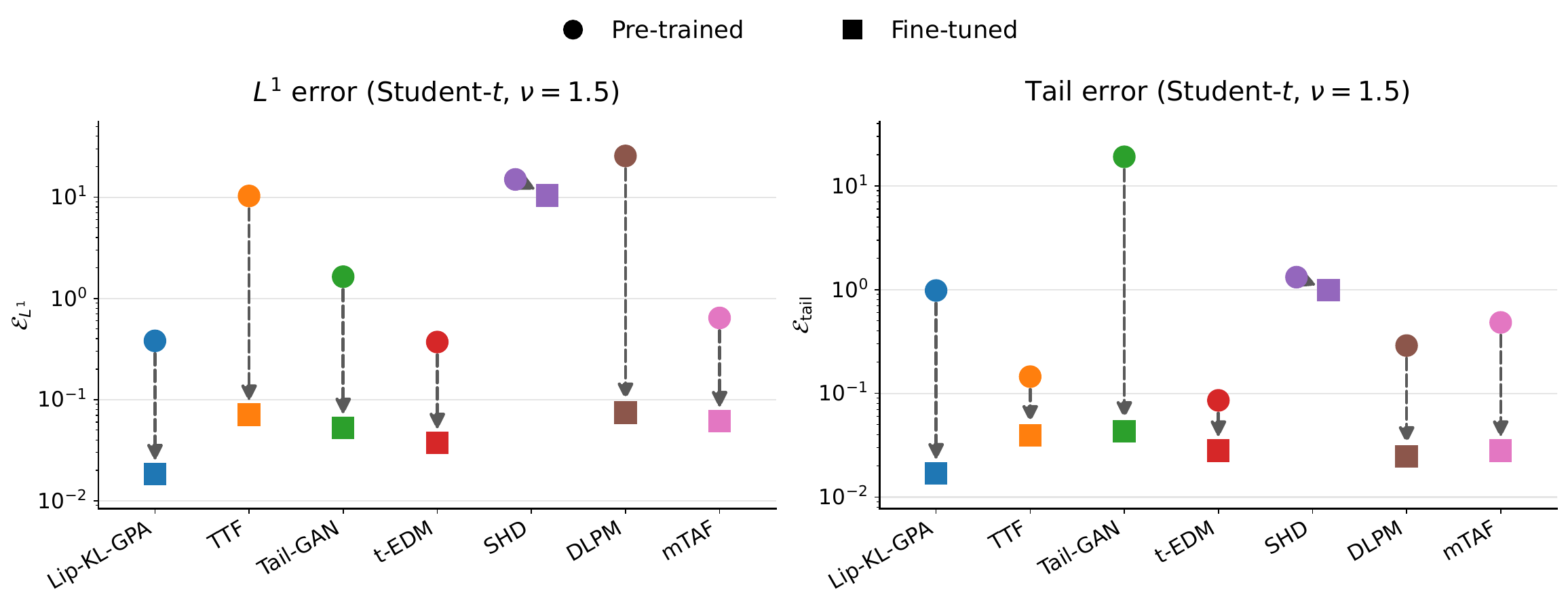}
\caption{2-dimensional Student-$t$ target with $\nu=1.5$: global $L^1$ error
$\mathcal{E}_{L^1}$ (left) and tail error $\mathcal{E}_{\mathrm{tail}}$ (right)
of the seven pre-trained models before (circles) and after (squares)
fine-tuning via CVaR-GPA, on a logarithmic axis. The underlying error values are reported in \Cref{tab:studentt}.}
\label{fig:ER_student_t_nu15}
\end{figure}
 
\begin{figure}[htbp]
\centering
\includegraphics[width=0.9\linewidth]{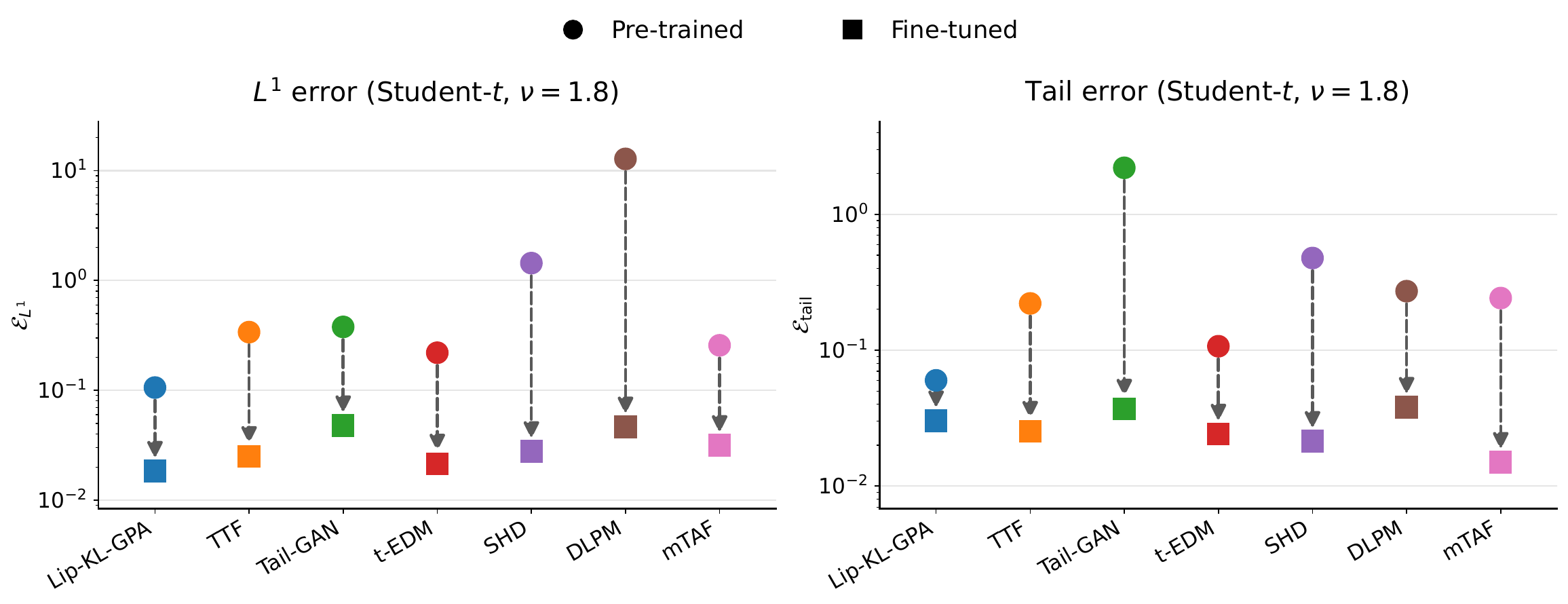}
\caption{2-dimensional Student-$t$ target with $\nu=1.8$: global $L^1$ error
$\mathcal{E}_{L^1}$ (left) and tail error $\mathcal{E}_{\mathrm{tail}}$ (right)
of the seven pre-trained models before (circles) and after (squares)
fine-tuning via CVaR-GPA, on a logarithmic axis. Both errors are reduced for
every pre-trained model. The underlying error values are reported in
\Cref{tab:studentt}.}
\label{fig:ER_student_t_nu18}
\end{figure}
\subsubsection{Neal's funnel distribution}
\label{APP_NEAL}

The per-marginal errors and error-reduction factors are reported in \Cref{tab:neal_marg}. 

\begin{table}[htbp]
\centering\small\setlength{\tabcolsep}{3.5pt}
\caption{Number of marginal distributions whose error decreases after fine-tuning with CVaR-GPA, out of the $d$ marginals of each dataset, for the seven pre-trained models.}
\label{tab:marg_counts}
\begin{tabular}{l rr rr rr rr}\toprule
 & \multicolumn{2}{c}{Cauchy ($d{=}2$)} & \multicolumn{2}{c}{Neal ($d{=}2$)} & \multicolumn{2}{c}{FF25 ($d{=}25$)} & \multicolumn{2}{c}{Streamflow ($d{=}64$)}\\
\cmidrule(lr){2-3}\cmidrule(lr){4-5}\cmidrule(lr){6-7}\cmidrule(lr){8-9}
pre-trained model & $\mathcal{E}_{L^1}$ & $\mathcal{E}_{\mathrm{tail}}$ & $\mathcal{E}_{L^1}$ & $\mathcal{E}_{\mathrm{tail}}$ & $\mathcal{E}_{L^1}$ & $\mathcal{E}_{\mathrm{tail}}$ & $\mathcal{E}_{L^1}$ & $\mathcal{E}_{\mathrm{tail}}$\\ \midrule
Lip-KL-GPA & 2/2 & 2/2 & 2/2 & 1/2 & 25/25 & 25/25 & 58/64 & 64/64\\
TTF & 2/2 & 1/2 & 2/2 & 2/2 & 25/25 & 25/25 & 64/64 & 45/64\\
Tail-GAN & 2/2 & 2/2 & 2/2 & 2/2 & 25/25 & 25/25 & 64/64 & 57/64\\
t-EDM & 2/2 & 2/2 & 2/2 & 2/2 & 25/25 & 25/25 & 64/64 & 61/64\\
SHD & 1/2 & 1/2 & 2/2 & 2/2 & 25/25 & 25/25 & 64/64 & 64/64\\
DLPM & 2/2 & 2/2 & 2/2 & 2/2 & 25/25 & 25/25 & 64/64 & 64/64\\
mTAF & 2/2 & 2/2 & 2/2 & 2/2 & 25/25 & 25/25 & 61/64 & 63/64\\
\bottomrule\end{tabular}
\end{table}

\begin{figure}[htbp]
\includegraphics[width=\linewidth]{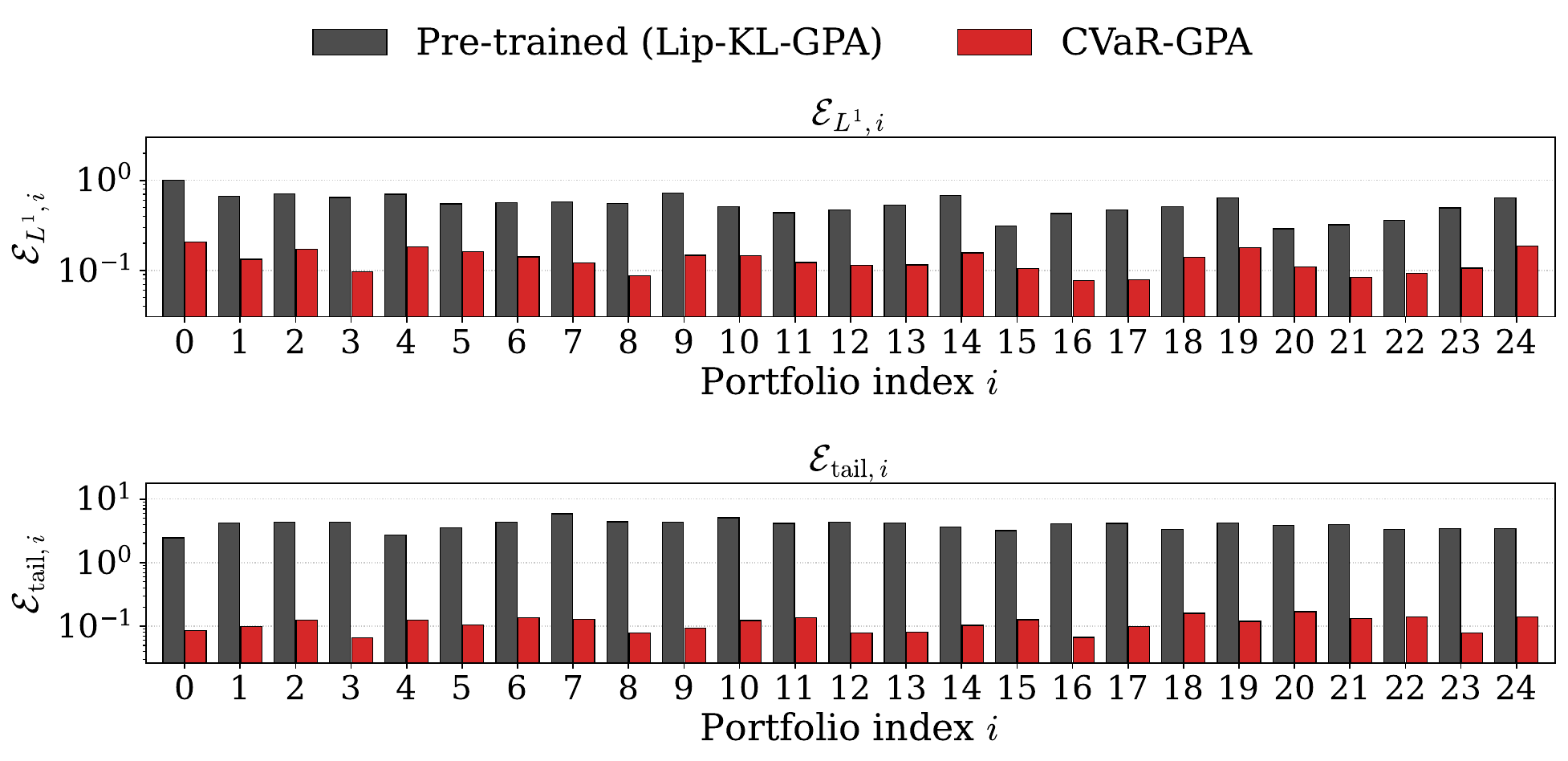}
\caption{Fama--French 25: per-portfolio errors of the Lip-KL-GPA model before and after fine-tuning with CVaR-GPA.}
\label{fig:ff25_perdim}
\end{figure}

\begin{figure}[htbp]
\includegraphics[width=\linewidth]{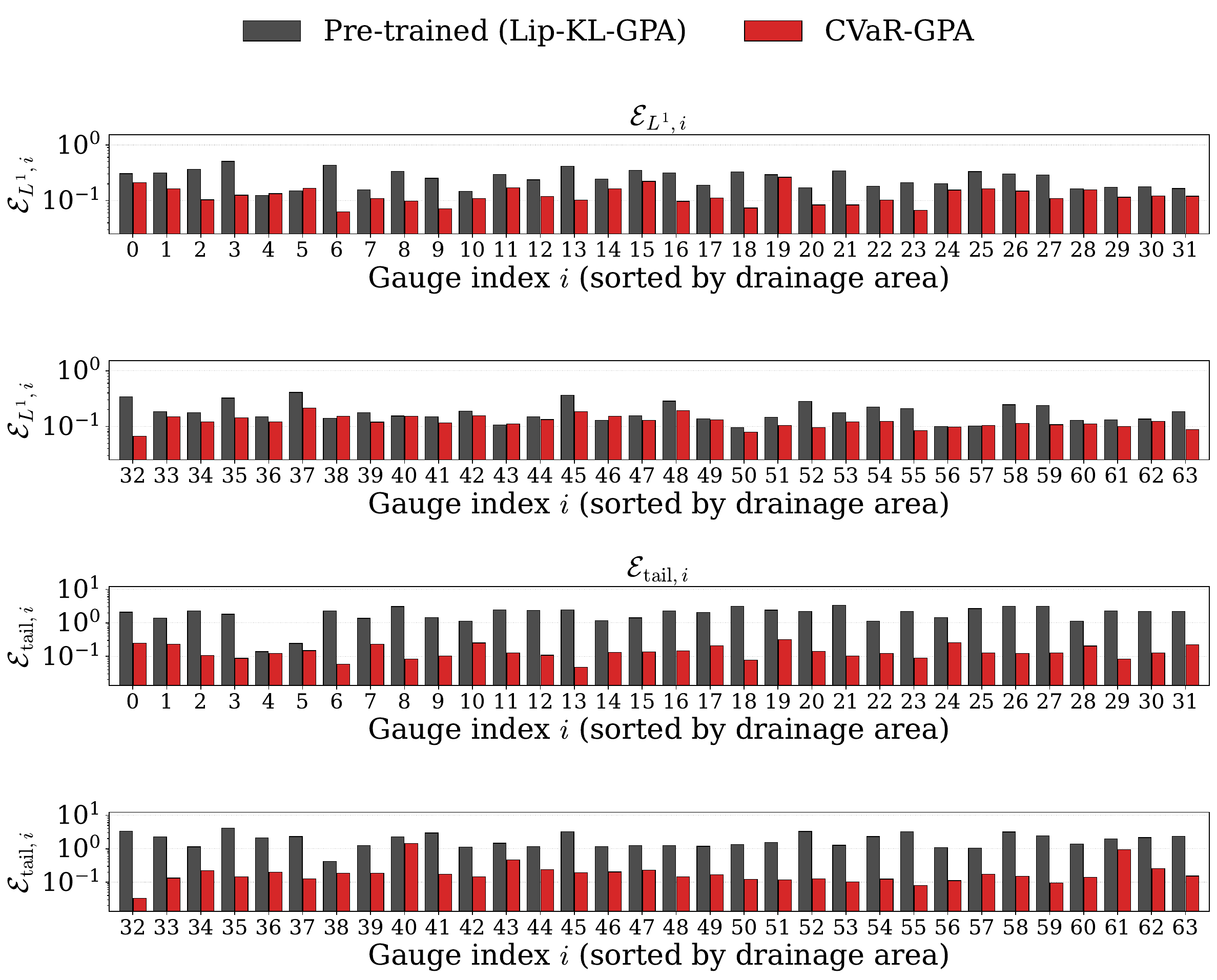}
\caption{Streamflow: per-gauge errors of the Lip-KL-GPA model before and after fine-tuning with CVaR-GPA.}
\label{fig:sf_perdim}
\end{figure}
\begin{table}[htbp]
\centering\small\setlength{\tabcolsep}{3pt}
\caption{2-dimensional Cauchy target: errors of the marginal distributions of $x_1$ and $x_2$ for the seven pre-trained models before and after fine-tuning, and the error-reduction factors (before $/$ after). Non-improvements are in gray; improvements are in bold. The last row reports the geometric mean of the factors over the seven models. t-EDM with $\nu=3$.}
\label{tab:2d_marg}
\resizebox{\linewidth}{!}{\begin{tabular}{l rrr rrr rrr rrr}\toprule
 & \multicolumn{6}{c}{$\mathcal{E}_{L^1}$} & \multicolumn{6}{c}{$\mathcal{E}_{\mathrm{tail}}$}\\ \cmidrule(lr){2-7}\cmidrule(lr){8-13}
 & \multicolumn{3}{c}{$x_1$} & \multicolumn{3}{c}{$x_2$} & \multicolumn{3}{c}{$x_1$} & \multicolumn{3}{c}{$x_2$}\\
\cmidrule(lr){2-4}\cmidrule(lr){5-7}\cmidrule(lr){8-10}\cmidrule(lr){11-13}
pre-trained model & before & after & factor & before & after & factor & before & after & factor & before & after & factor\\ \midrule
Lip-KL-GPA & 3.91 & 0.873 & \textbf{4.5}$\times$ & 5.11 & 0.858 & \textbf{6.0}$\times$ & 3.37 & 0.0578 & \textbf{58.3}$\times$ & 3.34 & 0.0788 & \textbf{42.4}$\times$\\
TTF & 6.66 & 1.98 & \textbf{3.4}$\times$ & 11.4 & 5.94 & \textbf{1.9}$\times$ & 0.159 & 0.0786 & \textbf{2.0}$\times$ & 0.114 & 0.142 & \textcolor{gray}{0.8$\times$}\\
Tail-GAN & 25.0 & 1.33 & \textbf{18.8}$\times$ & 37.7 & 1.49 & \textbf{25.3}$\times$ & 4.63 & 0.129 & \textbf{35.9}$\times$ & 4.17 & 0.055 & \textbf{75.8}$\times$\\
t-EDM & 3.34 & 1.29 & \textbf{2.6}$\times$ & 4.12 & 0.822 & \textbf{5.0}$\times$ & 0.19 & 0.0502 & \textbf{3.8}$\times$ & 0.24 & 0.049 & \textbf{4.9}$\times$\\
SHD & 11.5 & 21.4 & \textcolor{gray}{0.5$\times$} & 275 & 45.5 & \textbf{6.0}$\times$ & 0.889 & 1 & \textcolor{gray}{0.9$\times$} & 2.89 & 0.987 & \textbf{2.9}$\times$\\
DLPM & 482 & 10.2 & \textbf{47.3}$\times$ & 273 & 12.0 & \textbf{22.8}$\times$ & 0.644 & 0.42 & \textbf{1.5}$\times$ & 0.679 & 0.286 & \textbf{2.4}$\times$\\
mTAF & 32.5 & 7.92 & \textbf{4.1}$\times$ & 54.4 & 10.0 & \textbf{5.4}$\times$ & 1.73 & 0.29 & \textbf{6.0}$\times$ & 1.79 & 0.355 & \textbf{5.0}$\times$\\
\midrule
\textbf{geometric mean} & & & \textbf{5.0}$\times$ & & & {\textbf{7.3}$\times$} & & & {\textbf{5.4}$\times$} & & & {\textbf{6.4}$\times$}\\
\bottomrule\end{tabular}}
\end{table}

\begin{table}[htbp]
\centering\small\setlength{\tabcolsep}{3pt}
\caption{Neal's funnel: errors of the marginal distributions of $x_1$ and $x_2$ for the seven pre-trained models before and after fine-tuning, and the error-reduction factors (before $/$ after). The single non-improvement is in gray; improvements are in bold. The last row reports the geometric mean of the factors over the seven models. t-EDM with $\nu=(4,20)$.}
\label{tab:neal_marg}
\resizebox{\linewidth}{!}{\begin{tabular}{l rrr rrr rrr rrr}\toprule
 & \multicolumn{6}{c}{$\mathcal{E}_{L^1}$} & \multicolumn{6}{c}{$\mathcal{E}_{\mathrm{tail}}$}\\ \cmidrule(lr){2-7}\cmidrule(lr){8-13}
 & \multicolumn{3}{c}{$x_1$} & \multicolumn{3}{c}{$x_2$} & \multicolumn{3}{c}{$x_1$} & \multicolumn{3}{c}{$x_2$}\\
\cmidrule(lr){2-4}\cmidrule(lr){5-7}\cmidrule(lr){8-10}\cmidrule(lr){11-13}
pre-trained model & before & after & factor & before & after & factor & before & after & factor & before & after & factor\\ \midrule
Lip-KL-GPA & 0.0164 & 0.0108 & \textbf{1.5}$\times$ & 0.108 & 0.0228 & \textbf{4.7}$\times$ & 0.0394 & 0.0403 & \textcolor{gray}{0.98$\times$} & 0.0773 & 0.0188 & \textbf{4.1}$\times$\\
TTF & 0.106 & 0.037 & \textbf{2.9}$\times$ & 0.47 & 0.0269 & \textbf{17.5}$\times$ & 0.37 & 0.0763 & \textbf{4.8}$\times$ & 0.218 & 0.0133 & \textbf{16.4}$\times$\\
Tail-GAN & 0.352 & 0.0143 & \textbf{24.6}$\times$ & 1.48 & 0.0228 & \textbf{64.9}$\times$ & 1.02 & 0.0375 & \textbf{27.2}$\times$ & 3.73 & 0.0315 & \textbf{118.4}$\times$\\
t-EDM & 0.135 & 0.0131 & \textbf{10.3}$\times$ & 0.102 & 0.021 & \textbf{4.9}$\times$ & 0.447 & 0.0196 & \textbf{22.8}$\times$ & 0.0768 & 0.0125 & \textbf{6.1}$\times$\\
SHD & 0.158 & 0.0112 & \textbf{14.1}$\times$ & 1.71 & 0.0251 & \textbf{68.1}$\times$ & 0.884 & 0.0317 & \textbf{27.9}$\times$ & 0.576 & 0.0195 & \textbf{29.5}$\times$\\
DLPM & 5.69 & 0.0146 & \textbf{389.7}$\times$ & 5.23 & 0.0303 & \textbf{172.6}$\times$ & 0.203 & 0.0312 & \textbf{6.5}$\times$ & 0.232 & 0.0216 & \textbf{10.7}$\times$\\
mTAF & 0.0565 & 0.0093 & \textbf{6.1}$\times$ & 0.513 & 0.0206 & \textbf{24.9}$\times$ & 0.223 & 0.0427 & \textbf{5.2}$\times$ & 0.421 & 0.0193 & \textbf{21.8}$\times$\\
\midrule
\textbf{geometric mean} & & & \textbf{12.1}$\times$ & & & \textbf{25.8}$\times$ & & & \textbf{8.3}$\times$ & & & \textbf{16.5$\times$}\\
\bottomrule\end{tabular}}
\end{table}

\subsubsection{Fama--French 25: per-portfolio results}
\label{App:FF25}
\Cref{tab:ff25_l1,tab:ff25_tail} list the global and the tail error of every portfolio for the seven pre-trained models and the fine-tuned models.
\FloatBarrier
{\footnotesize\setlength{\tabcolsep}{2pt}
\begin{longtable}{l rr rr rr rr rr rr rr}
\caption{Fama--French 25: $\mathcal{E}_{L^1}$ of each portfolio for the seven pre-trained models (pre.) and the fine-tuned models (f.t.).}\label{tab:ff25_l1}\\ \toprule
 & \multicolumn{2}{c}{Lip-KL-GPA} & \multicolumn{2}{c}{TTF} & \multicolumn{2}{c}{Tail-GAN} & \multicolumn{2}{c}{t-EDM} & \multicolumn{2}{c}{SHD} & \multicolumn{2}{c}{DLPM} & \multicolumn{2}{c}{mTAF}\\
\cmidrule(lr){2-3} \cmidrule(lr){4-5} \cmidrule(lr){6-7} \cmidrule(lr){8-9} \cmidrule(lr){10-11} \cmidrule(lr){12-13} \cmidrule(lr){14-15}
portfolio & pre. & f.t. & pre. & f.t. & pre. & f.t. & pre. & f.t. & pre. & f.t. & pre. & f.t. & pre. & f.t.\\ \midrule \endfirsthead
\toprule  & \multicolumn{2}{c}{Lip-KL-GPA} & \multicolumn{2}{c}{TTF} & \multicolumn{2}{c}{Tail-GAN} & \multicolumn{2}{c}{{t-EDM}} & \multicolumn{2}{c}{SHD} & \multicolumn{2}{c}{DLPM} & \multicolumn{2}{c}{mTAF}\\
\cmidrule(lr){2-3} \cmidrule(lr){4-5} \cmidrule(lr){6-7} \cmidrule(lr){8-9} \cmidrule(lr){10-11} \cmidrule(lr){12-13} \cmidrule(lr){14-15}
portfolio & pre. & f.t. & pre. & f.t. & pre. & f.t. & pre. & f.t. & pre. & f.t. & pre. & f.t. & pre. & f.t.\\ \midrule \endhead \bottomrule \endfoot
0 & 1 & {0.206} & 1.5 & {0.217} & 1.02 & {0.215} & 0.335 & {0.203} & 8.97 & {0.23} & 89.6 & {0.221} & 3.34 & {0.231}\\
1 & 0.667 & {0.133} & 6.18 & {0.229} & 0.787 & {0.168} & 0.383 & {0.171} & 7.42 & {0.15} & 59.9 & {0.15} & 2.49 & {0.179}\\
2 & 0.711 & {0.171} & 1.14 & {0.176} & 3.24 & {0.144} & 0.428 & {0.159} & 6.76 & {0.178} & 65.5 & {0.164} & 2.4 & {0.181}\\
3 & 0.645 & {0.0967} & 4.05 & {0.19} & 2.58 & {0.153} & 0.507 & {0.125} & 5.42 & {0.181} & 62.2 & {0.131} & 1.93 & {0.17}\\
4 & 0.701 & {0.183} & 25.1 & {0.202} & 1.27 & {0.176} & 0.47 & {0.167} & 6.48 & {0.203} & 62.8 & {0.148} & 2.26 & {0.151}\\
5 & 0.549 & {0.161} & 30.7 & {0.194} & 0.607 & {0.177} & 0.379 & {0.137} & 4.88 & {0.136} & 67.0 & {0.152} & 1.95 & {0.138}\\
6 & 0.562 & {0.142} & 21.6 & {0.159} & 1.11 & {0.162} & 0.488 & {0.13} & 5.04 & {0.14} & 68.5 & {0.128} & 1.8 & {0.164}\\
7 & 0.574 & {0.121} & 27.6 & {0.198} & 0.453 & {0.125} & 0.496 & {0.109} & 5.35 & {0.123} & 62.7 & {0.113} & 1.54 & {0.154}\\
8 & 0.552 & {0.0874} & 0.714 & {0.164} & 3.86 & {0.122} & 0.571 & {0.127} & 5.21 & {0.182} & 53.2 & {0.106} & 1.59 & {0.116}\\
9 & 0.72 & {0.148} & 2.86 & {0.227} & 0.637 & {0.153} & 0.629 & {0.109} & 5.87 & {0.183} & 74.3 & {0.116} & 2.29 & {0.189}\\
10 & 0.508 & {0.146} & 1.02 & {0.148} & 3.15 & {0.138} & 0.284 & {0.0948} & 4.96 & {0.121} & 62.5 & {0.109} & 1.72 & {0.178}\\
11 & 0.437 & {0.122} & 2.62 & {0.157} & 0.59 & {0.0884} & 0.393 & {0.103} & 4.23 & {0.129} & 53.2 & {0.0972} & 1.53 & {0.122}\\
12 & 0.469 & {0.113} & 42.0 & {0.13} & 1.23 & {0.136} & 0.462 & {0.106} & 4.24 & {0.125} & 48.0 & {0.0883} & 1.2 & {0.143}\\
13 & 0.535 & {0.115} & 41.0 & {0.198} & 1.94 & {0.121} & 0.462 & {0.112} & 4.99 & {0.143} & 46.0 & {0.12} & 1.53 & {0.11}\\
14 & 0.675 & {0.156} & 2.47 & {0.173} & 2.62 & {0.172} & 0.507 & {0.142} & 6.45 & {0.135} & 64.0 & {0.155} & 2.1 & {0.134}\\
15 & 0.31 & {0.106} & 2.59 & {0.0976} & 1.29 & {0.0989} & 0.238 & {0.109} & 4.09 & {0.089} & 40.9 & {0.0943} & 1.43 & {0.117}\\
16 & 0.427 & {0.0766} & 113 & {0.139} & 1.38 & {0.114} & 0.334 & {0.0906} & 4.26 & {0.119} & 63.0 & {0.11} & 1.39 & {0.0935}\\
17 & 0.47 & {0.0787} & 2.16 & {0.131} & 1.1 & {0.105} & 0.478 & {0.114} & 4.71 & {0.164} & 53.9 & {0.101} & 1.36 & {0.115}\\
18 & 0.511 & {0.139} & 1.07 & {0.149} & 0.806 & {0.136} & 0.49 & {0.12} & 5.05 & {0.149} & 54.8 & {0.104} & 1.62 & {0.115}\\
19 & 0.634 & {0.178} & 1.28 & {0.185} & 0.683 & {0.186} & 0.751 & {0.159} & 7.13 & {0.17} & 71.8 & {0.2} & 1.99 & {0.176}\\
20 & 0.289 & {0.11} & 5.38 & {0.129} & 1.19 & {0.125} & 0.262 & {0.106} & 3.2 & {0.152} & 51.4 & {0.101} & 1.25 & {0.0931}\\
21 & 0.322 & {0.0833} & 37.8 & {0.153} & 2 & {0.114} & 0.237 & {0.0722} & 3.14 & {0.124} & 38.6 & {0.148} & 0.896 & {0.112}\\
22 & 0.359 & {0.0928} & 30.9 & {0.198} & 1.59 & {0.112} & 0.373 & {0.123} & 4.38 & {0.158} & 39.1 & {0.077} & 0.963 & {0.111}\\
23 & 0.493 & {0.106} & 1.85 & {0.234} & 1.62 & {0.121} & 0.464 & {0.103} & 4.64 & {0.181} & 71.8 & {0.108} & 1.43 & {0.164}\\
24 & 0.637 & {0.187} & 1.33 & {0.242} & 2.44 & {0.18} & 0.698 & {0.16} & 6.07 & {0.216} & 64.5 & {0.186} & 2.22 & {0.212}\\
\end{longtable}

\begin{longtable}{l rr rr rr rr rr rr rr}
\caption{Fama--French 25: $\mathcal{E}_{\mathrm{tail}}$ of each portfolio for the seven pre-trained models (pre.) and the fine-tuned models (f.t.).}\label{tab:ff25_tail}\\ \toprule
 & \multicolumn{2}{c}{Lip-KL-GPA} & \multicolumn{2}{c}{TTF} & \multicolumn{2}{c}{Tail-GAN} & \multicolumn{2}{c}{t-EDM} & \multicolumn{2}{c}{SHD} & \multicolumn{2}{c}{DLPM} & \multicolumn{2}{c}{mTAF}\\
\cmidrule(lr){2-3} \cmidrule(lr){4-5} \cmidrule(lr){6-7} \cmidrule(lr){8-9} \cmidrule(lr){10-11} \cmidrule(lr){12-13} \cmidrule(lr){14-15}
portfolio & pre. & f.t. & pre. & f.t. & pre. & f.t. & pre. & f.t. & pre. & f.t. & pre. & f.t. & pre. & f.t.\\ \midrule \endfirsthead
\toprule  & \multicolumn{2}{c}{Lip-KL-GPA} & \multicolumn{2}{c}{TTF} & \multicolumn{2}{c}{Tail-GAN} & \multicolumn{2}{c}{t-EDM} & \multicolumn{2}{c}{SHD} & \multicolumn{2}{c}{DLPM} & \multicolumn{2}{c}{mTAF}\\
\cmidrule(lr){2-3} \cmidrule(lr){4-5} \cmidrule(lr){6-7} \cmidrule(lr){8-9} \cmidrule(lr){10-11} \cmidrule(lr){12-13} \cmidrule(lr){14-15}
portfolio & pre. & f.t. & pre. & f.t. & pre. & f.t. & pre. & f.t. & pre. & f.t. & pre. & f.t. & pre. & f.t.\\ \midrule \endhead \bottomrule \endfoot
0 & 2.47 & {0.0855} & 0.713 & {0.102} & 1.97 & {0.115} & 0.934 & {0.0608} & 2.61 & {0.12} & 2.89 & {0.08} & 2.32 & {0.0976}\\
1 & 4.2 & {0.0987} & 0.936 & {0.192} & 3.04 & {0.0794} & 0.192 & {0.121} & 2.58 & {0.0893} & 2.87 & {0.0621} & 1.21 & {0.122}\\
2 & 4.36 & {0.125} & 0.709 & {0.0715} & 1.87 & {0.121} & 0.211 & {0.102} & 3.32 & {0.121} & 2.86 & {0.16} & 3.12 & {0.112}\\
3 & 4.32 & {0.0664} & 0.999 & {0.236} & 1.13 & {0.0915} & 0.335 & {0.0995} & 2.34 & {0.125} & 2.92 & {0.064} & 3.01 & {0.138}\\
4 & 2.73 & {0.123} & 1.08 & {0.176} & 1.31 & {0.129} & 0.287 & {0.0628} & 3.17 & {0.135} & 2.89 & {0.0918} & 1.28 & {0.0903}\\
5 & 3.53 & {0.106} & 0.86 & {0.129} & 0.352 & {0.136} & 0.287 & {0.0955} & 2.3 & {0.138} & 2.94 & {0.114} & 4.2 & {0.0716}\\
6 & 4.3 & {0.136} & 0.835 & {0.141} & 2.23 & {0.0847} & 0.377 & {0.116} & 2.5 & {0.0829} & 2.93 & {0.0808} & 4.22 & {0.107}\\
7 & 5.88 & {0.128} & 0.825 & {0.142} & 1.89 & {0.0781} & 0.349 & {0.0754} & 2.54 & {0.0841} & 2.97 & {0.0584} & 3.26 & {0.079}\\
8 & 4.39 & {0.0776} & 1.26 & {0.12} & 1.46 & {0.0885} & 0.412 & {0.0777} & 2.51 & {0.0598} & 2.93 & {0.0329} & 2.34 & {0.0848}\\
9 & 4.38 & {0.0926} & 0.902 & {0.199} & 1.89 & {0.0944} & 0.391 & {0.0868} & 2.43 & {0.139} & 2.91 & {0.0711} & 2.25 & {0.0729}\\
10 & 5.09 & {0.123} & 0.682 & {0.131} & 1.22 & {0.0453} & 1.17 & {0.0959} & 1.62 & {0.0887} & 2.97 & {0.108} & 4.31 & {0.129}\\
11 & 4.14 & {0.136} & 0.864 & {0.156} & 0.434 & {0.102} & 0.487 & {0.1} & 2.28 & {0.132} & 2.95 & {0.104} & 3.36 & {0.101}\\
12 & 4.3 & {0.0784} & 0.717 & {0.124} & 1.6 & {0.142} & 0.475 & {0.0746} & 2.46 & {0.103} & 2.95 & {0.0643} & 3.91 & {0.06}\\
13 & 4.2 & {0.0803} & 0.797 & {0.15} & 1.86 & {0.0778} & 0.376 & {0.0508} & 2.55 & {0.0666} & 2.91 & {0.102} & 2.3 & {0.0533}\\
14 & 3.63 & {0.103} & 1.02 & {0.117} & 17.5 & {0.144} & 1.15 & {0.129} & 2.43 & {0.0955} & 2.92 & {0.118} & 2.23 & {0.0775}\\
15 & 3.2 & {0.126} & 0.976 & {0.104} & 1.53 & {0.134} & 0.336 & {0.12} & 1.62 & {0.0826} & 2.96 & {0.92} & 3.49 & {0.0821}\\
16 & 4.13 & {0.0666} & 0.644 & {0.0844} & 6.91 & {0.0694} & 0.43 & {0.09} & 1.68 & {0.113} & 2.98 & {0.0705} & 3.21 & {0.0698}\\
17 & 4.16 & {0.0982} & 0.621 & {0.171} & 3.19 & {0.0978} & 0.421 & {0.0741} & 2.57 & {0.117} & 2.93 & {0.101} & 3.2 & {0.069}\\
18 & 3.31 & {0.16} & 0.727 & {0.159} & 4.01 & {0.168} & 1.23 & {0.893} & 2.58 & {0.13} & 2.89 & {0.131} & 3.13 & {0.155}\\
19 & 4.22 & {0.119} & 0.706 & {0.1} & 3.17 & {0.117} & 0.433 & {0.104} & 2.69 & {0.119} & 2.88 & {0.139} & 3.05 & {0.079}\\
20 & 3.84 & {0.17} & 1.11 & {0.164} & 1.14 & {0.179} & 0.414 & {0.167} & 1.48 & {0.231} & 2.93 & {0.12} & 3.88 & {0.152}\\
21 & 3.96 & {0.132} & 1.12 & {0.214} & 15.8 & {0.141} & 0.337 & {0.114} & 2.19 & {0.186} & 2.92 & {0.917} & 3.16 & {0.144}\\
22 & 3.3 & {0.139} & 1.05 & {0.143} & 10.6 & {0.128} & 0.315 & {0.0926} & 2.58 & {0.135} & 2.87 & {0.869} & 2.03 & {0.113}\\
23 & 3.42 & {0.0791} & 0.907 & {0.208} & 1.66 & {0.167} & 0.351 & {0.11} & 2.28 & {0.114} & 2.92 & {0.092} & 3.2 & {0.109}\\
24 & 3.45 & {0.139} & 0.575 & {0.0751} & 1.69 & {0.162} & 0.434 & {0.12} & 2.29 & {0.14} & 2.86 & {0.0849} & 1.42 & {0.15}\\
\end{longtable}
}

\subsubsection{Streamflow: per-gauge results}
\label{App:SF}
The streamflow target consists of the daily specific discharge (mm/day) of the $64$ gauges with the smallest drainage areas ($2.5$ to $204$ km$^2$) among the $312$ gauges of the Ohio River basin with a complete daily record over 2005--2024. Small catchments respond fast to rainfall, which is why their daily discharge has the heaviest tails in the basin. \Cref{tab:sf_l1,tab:sf_tail} list the global and the tail error of every gauge for the seven pre-trained models and the fine-tuned models.

\FloatBarrier
{\scriptsize\setlength{\tabcolsep}{2.5pt}
\begin{longtable}{l r rr rr rr rr rr rr rr}
\captionsetup{font=scriptsize}
\caption{Streamflow: $\mathcal{E}_{L^1}$ of each gauge (USGS site number, sorted by drainage area) with its Hill tail index $\hat\alpha$, for the seven pre-trained models (pre.) and the fine-tuned models (f.t.).}\label{tab:sf_l1}\\ \toprule
 & & \multicolumn{2}{c}{Lip-KL-GPA} & \multicolumn{2}{c}{TTF} & \multicolumn{2}{c}{Tail-GAN} & \multicolumn{2}{c}{t-EDM} & \multicolumn{2}{c}{SHD} & \multicolumn{2}{c}{DLPM} & \multicolumn{2}{c}{mTAF}\\
\cmidrule(lr){3-4} \cmidrule(lr){5-6} \cmidrule(lr){7-8} \cmidrule(lr){9-10} \cmidrule(lr){11-12} \cmidrule(lr){13-14} \cmidrule(lr){15-16}
gauge & $\hat\alpha$ & pre. & f.t. & pre. & f.t. & pre. & f.t. & pre. & f.t. & pre. & f.t. & pre. & f.t. & pre. & f.t.\\ \midrule \endfirsthead
\toprule  & & \multicolumn{2}{c}{Lip-KL-GPA} & \multicolumn{2}{c}{TTF} & \multicolumn{2}{c}{Tail-GAN} & \multicolumn{2}{c}{t-EDM} & \multicolumn{2}{c}{SHD} & \multicolumn{2}{c}{DLPM} & \multicolumn{2}{c}{mTAF}\\
\cmidrule(lr){3-4} \cmidrule(lr){5-6} \cmidrule(lr){7-8} \cmidrule(lr){9-10} \cmidrule(lr){11-12} \cmidrule(lr){13-14} \cmidrule(lr){15-16}
gauge & $\hat\alpha$ & pre. & f.t. & pre. & f.t. & pre. & f.t. & pre. & f.t. & pre. & f.t. & pre. & f.t. & pre. & f.t.\\ \midrule \endhead \bottomrule \endfoot
03284525 & 1.81 & 0.306 & {0.21} & 0.568 & {0.155} & 1.19 & {0.362} & 0.434 & {0.134} & 16.8 & {0.31} & 479 & {0.577} & 0.255 & {0.244}\\
03357350 & 1.22 & 0.316 & {0.164} & 0.304 & {0.247} & 1.19 & {0.422} & 0.23 & {0.157} & 130 & {0.466} & 425 & {0.654} & 0.273 & {0.162}\\
03301900 & 1.50 & 0.364 & {0.104} & 2.75 & {0.0955} & 1.12 & {0.201} & 0.357 & {0.116} & 44.2 & {0.3} & 298 & {0.367} & 0.298 & {0.148}\\
03260100 & 1.42 & 0.508 & {0.126} & 0.48 & {0.233} & 1.8 & {0.417} & 0.301 & {0.21} & 37.2 & {0.198} & 506 & {0.779} & 0.49 & {0.213}\\
03337000 & 1.82 & 0.124 & 0.133 & 0.398 & {0.112} & 0.439 & {0.27} & 0.106 & {0.0443} & 5.5 & {0.173} & 217 & {0.575} & 0.0499 & 0.0564\\
03337100 & 1.86 & 0.15 & 0.167 & 0.644 & {0.154} & 0.539 & {0.336} & 0.141 & {0.0707} & 2.12 & {0.225} & 264 & {0.665} & 0.068 & 0.0796\\
03298135 & 1.58 & 0.436 & {0.0633} & 5.05 & {0.214} & 1.17 & {0.262} & 0.408 & {0.119} & 81.5 & {0.252} & 539 & {0.523} & 0.322 & {0.112}\\
03049800 & 1.84 & 0.156 & {0.109} & 0.103 & {0.0927} & 0.32 & {0.293} & 0.153 & {0.0715} & 2.36 & {0.136} & 274 & {0.416} & 0.161 & {0.123}\\
03298300 & 1.47 & 0.338 & {0.0977} & 1143 & {0.143} & 1.45 & {0.213} & 0.28 & {0.118} & 125 & {0.215} & 285 & {0.411} & 0.299 & {0.133}\\
03353120 & 1.38 & 0.252 & {0.072} & 0.203 & {0.121} & 1.51 & {0.37} & 0.174 & {0.151} & 44.9 & {0.216} & 362 & {0.62} & 0.246 & {0.0798}\\
03026500 & 2.31 & 0.147 & {0.109} & 0.131 & {0.0665} & 0.462 & {0.179} & 0.144 & {0.0608} & 4.87 & {0.137} & 256 & {0.338} & 0.158 & {0.0877}\\
03298550 & 1.39 & 0.296 & {0.172} & 1640 & {0.254} & 1.27 & {0.378} & 0.216 & {0.136} & 136 & {0.41} & 398 & {0.539} & 0.315 & {0.12}\\
03274650 & 1.43 & 0.237 & {0.118} & 0.241 & {0.16} & 1 & {0.301} & 0.128 & {0.111} & 11.0 & {0.0979} & 229 & {0.435} & 0.281 & {0.101}\\
03298150 & 1.45 & 0.418 & {0.102} & 3.31 & {0.155} & 1.31 & {0.185} & 0.372 & {0.0742} & 13.6 & {0.142} & 340 & {0.381} & 0.35 & {0.113}\\
03431300 & 1.67 & 0.245 & {0.162} & 0.593 & {0.226} & 1.17 & {0.328} & 0.285 & {0.0982} & 51.3 & {0.258} & 428 & {0.564} & 0.397 & {0.2}\\
03611260 & 1.16 & 0.352 & {0.222} & 118 & {0.235} & 1.75 & {0.742} & 0.285 & {0.197} & 252 & {0.724} & 612 & {1.12} & 0.55 & {0.298}\\
03353620 & 1.40 & 0.313 & {0.0974} & 7.55 & {0.156} & 1.18 & {0.238} & 0.189 & {0.119} & 35.6 & {0.213} & 291 & {0.451} & 0.257 & {0.102}\\
03351310 & 1.53 & 0.189 & {0.112} & 0.534 & {0.0789} & 0.791 & {0.199} & 0.173 & {0.101} & 7.86 & {0.117} & 245 & {0.366} & 0.204 & {0.123}\\
03293000 & 1.48 & 0.326 & {0.0736} & 1.58 & {0.105} & 1.02 & {0.183} & 0.324 & {0.0747} & 16.4 & {0.168} & 275 & {0.358} & 0.228 & {0.11}\\
03302680 & 1.53 & 0.292 & {0.263} & 0.721 & {0.149} & 1.02 & {0.489} & 0.224 & {0.157} & 69.6 & {0.374} & 302 & {0.673} & 0.286 & {0.154}\\
03289000 & 2.04 & 0.17 & {0.0837} & 0.242 & {0.0881} & 0.311 & {0.146} & 0.292 & {0.0549} & 43.4 & {0.21} & 301 & {0.326} & 0.126 & {0.0757}\\
03262001 & 1.34 & 0.342 & {0.0837} & 0.273 & {0.1} & 1.18 & {0.195} & 0.187 & {0.0895} & 62.7 & {0.191} & 284 & {0.4} & 0.449 & {0.161}\\
03431700 & 1.66 & 0.182 & {0.102} & 0.326 & {0.114} & 0.923 & {0.212} & 0.238 & {0.0825} & 12.5 & {0.117} & 357 & {0.429} & 0.367 & {0.176}\\
03353600 & 1.78 & 0.21 & {0.0667} & 0.148 & {0.0567} & 0.655 & {0.186} & 0.155 & {0.115} & 67.5 & {0.185} & 263 & {0.348} & 0.211 & {0.0908}\\
03140000 & 1.35 & 0.203 & {0.155} & 0.145 & {0.119} & 0.84 & {0.317} & 0.11 & {0.0722} & 2.43 & {0.297} & 217 & {0.48} & 0.192 & {0.166}\\
03291780 & 1.37 & 0.332 & {0.162} & 0.354 & {0.143} & 1.25 & {0.298} & 0.163 & {0.123} & 130 & {0.323} & 346 & {0.491} & 0.382 & {0.151}\\
03289200 & 2.30 & 0.304 & {0.148} & 0.709 & {0.231} & 0.794 & {0.222} & 0.387 & {0.0838} & 16.8 & {0.313} & 369 & {0.373} & 0.205 & {0.106}\\
03254550 & 1.53 & 0.288 & {0.109} & 0.224 & {0.0946} & 0.96 & {0.238} & 0.229 & {0.0919} & 42.4 & {0.218} & 278 & {0.408} & 0.296 & {0.117}\\
03213500 & 1.54 & 0.163 & {0.156} & 1.22 & {0.129} & 0.633 & {0.286} & 0.109 & {0.0832} & 27.7 & {0.189} & 259 & {0.435} & 0.208 & {0.111}\\
03339280 & 1.55 & 0.173 & {0.115} & 0.148 & {0.126} & 0.519 & {0.217} & 0.12 & {0.0688} & 3.46 & {0.0611} & 196 & {0.376} & 0.209 & {0.1}\\
03230450 & 1.79 & 0.178 & {0.122} & 0.253 & {0.137} & 0.517 & {0.197} & 0.156 & {0.0671} & 11.1 & {0.117} & 225 & {0.352} & 0.226 & {0.0868}\\
03266560 & 1.80 & 0.165 & {0.12} & 0.319 & {0.0348} & 0.373 & {0.152} & 0.126 & {0.0475} & 8.5 & {0.14} & 167 & {0.273} & 0.112 & {0.0769}\\
03277075 & 1.49 & 0.342 & {0.0673} & 0.353 & {0.119} & 1.13 & {0.151} & 0.248 & {0.079} & 5.4 & {0.185} & 254 & {0.355} & 0.32 & {0.0951}\\
03223425 & 1.70 & 0.183 & {0.149} & 1.87 & {0.167} & 0.727 & {0.293} & 0.148 & {0.0957} & 6.91 & {0.37} & 269 & {0.447} & 0.219 & {0.108}\\
03011800 & 2.64 & 0.178 & {0.12} & 0.346 & {0.114} & 0.481 & {0.174} & 0.188 & {0.0548} & 74.7 & {0.21} & 253 & {0.31} & 0.199 & {0.113}\\
03238745 & 1.28 & 0.325 & {0.144} & 0.418 & {0.18} & 1.65 & {0.311} & 0.163 & {0.0971} & 112 & {0.324} & 284 & {0.543} & 0.372 & {0.151}\\
03260706 & 1.83 & 0.15 & {0.121} & 7.91 & {0.101} & 0.205 & {0.146} & 0.131 & {0.0534} & 6.88 & {0.0984} & 168 & {0.291} & 0.134 & {0.0678}\\
03384450 & 1.05 & 0.409 & {0.215} & 2.61 & {0.331} & 1.64 & {0.523} & 0.268 & {0.212} & 84.7 & {0.46} & 438 & {0.81} & 0.386 & {0.237}\\
03177710 & 2.22 & 0.141 & 0.153 & 0.145 & {0.0761} & 0.263 & {0.189} & 0.168 & {0.0512} & 2.84 & {0.0797} & 227 & {0.326} & 0.0897 & {0.0706}\\
03076600 & 2.13 & 0.177 & {0.12} & 0.171 & {0.0929} & 0.335 & {0.173} & 0.296 & {0.0697} & 2.92 & {0.198} & 282 & {0.376} & 0.203 & {0.0978}\\
03330241 & 2.92 & 0.154 & {0.152} & 0.105 & {0.089} & 0.22 & {0.203} & 0.0737 & {0.033} & 5.9 & {0.18} & 148 & {0.282} & 0.0506 & {0.0341}\\
03350700 & 1.66 & 0.149 & {0.115} & 0.0931 & {0.0587} & 0.628 & {0.203} & 0.138 & {0.0745} & 52.1 & {0.197} & 192 & {0.322} & 0.165 & {0.0776}\\
03431599 & 1.42 & 0.188 & {0.157} & 0.24 & {0.207} & 0.899 & {0.331} & 0.232 & {0.117} & 16.6 & {0.151} & 364 & {0.583} & 0.253 & {0.145}\\
03106300 & 3.34 & 0.106 & 0.111 & 0.154 & {0.0625} & 0.42 & {0.156} & 0.129 & {0.0509} & 3.94 & {0.179} & 200 & {0.266} & 0.0386 & 0.0405\\
03185000 & 1.80 & 0.149 & {0.133} & 0.208 & {0.0702} & 0.271 & {0.208} & 0.134 & {0.0556} & 30.9 & {0.173} & 269 & {0.332} & 0.122 & {0.1}\\
03244936 & 1.31 & 0.364 & {0.185} & 0.317 & {0.194} & 1.59 & {0.365} & 0.232 & {0.13} & 242 & {0.374} & 367 & {0.563} & 0.431 & {0.181}\\
03210000 & 1.75 & 0.13 & 0.154 & 1.08 & {0.129} & 0.587 & {0.242} & 0.0999 & {0.0682} & 10.6 & {0.11} & 224 & {0.383} & 0.181 & {0.117}\\
03042280 & 2.24 & 0.157 & {0.129} & 1.01 & {0.108} & 0.425 & {0.214} & 0.196 & {0.0726} & 53.4 & {0.21} & 333 & {0.393} & 0.133 & {0.109}\\
03067510 & 2.03 & 0.285 & {0.194} & 0.381 & {0.211} & 0.619 & {0.352} & 0.481 & {0.114} & 29.0 & {0.343} & 519 & {0.593} & 0.296 & {0.193}\\
03280700 & 1.73 & 0.139 & {0.132} & 0.334 & {0.12} & 0.625 & {0.257} & 0.113 & {0.0771} & 5.59 & {0.283} & 266 & {0.386} & 0.188 & {0.0973}\\
03062500 & 2.75 & 0.0959 & {0.0788} & 0.165 & {0.0791} & 0.348 & {0.154} & 0.164 & {0.0679} & 69.1 & {0.142} & 266 & {0.273} & 0.143 & {0.0493}\\
03241500 & 1.74 & 0.146 & {0.105} & 0.14 & {0.0545} & 0.39 & {0.171} & 0.165 & {0.056} & 4.98 & {0.135} & 194 & {0.311} & 0.196 & {0.0981}\\
03302000 & 1.63 & 0.284 & {0.0964} & 0.277 & {0.0591} & 0.829 & {0.202} & 0.196 & {0.0935} & 5.13 & {0.0668} & 314 & {0.315} & 0.236 & {0.0803}\\
03078000 & 2.21 & 0.177 & {0.121} & 0.201 & {0.117} & 0.402 & {0.161} & 0.247 & {0.0521} & 132 & {0.259} & 268 & {0.351} & 0.242 & {0.111}\\
03228750 & 1.43 & 0.223 & {0.123} & 8.51 & {0.178} & 1.01 & {0.363} & 0.18 & {0.117} & 42.0 & {0.197} & 266 & {0.535} & 0.226 & {0.132}\\
03302220 & 1.68 & 0.211 & {0.0839} & 0.273 & {0.116} & 0.817 & {0.186} & 0.193 & {0.0744} & 15.3 & {0.127} & 260 & {0.308} & 0.22 & {0.0627}\\
03277300 & 1.97 & 0.101 & {0.0979} & 0.284 & {0.086} & 0.207 & {0.164} & 0.114 & {0.0422} & 8.89 & {0.139} & 207 & {0.321} & 0.15 & {0.0807}\\
03208950 & 1.83 & 0.102 & 0.105 & 0.35 & {0.0807} & 0.451 & {0.158} & 0.0991 & {0.0454} & 2.42 & {0.164} & 191 & {0.327} & 0.177 & {0.0995}\\
03271000 & 1.32 & 0.247 & {0.114} & 0.162 & {0.0675} & 0.783 & {0.199} & 0.169 & {0.0757} & 6.07 & {0.0599} & 213 & {0.34} & 0.258 & {0.0819}\\
03272700 & 1.44 & 0.237 & {0.108} & 0.133 & {0.0754} & 0.712 & {0.216} & 0.177 & {0.13} & 50.6 & {0.162} & 207 & {0.348} & 0.296 & {0.104}\\
03141870 & 1.73 & 0.13 & {0.111} & 0.238 & {0.145} & 0.573 & {0.252} & 0.193 & {0.0937} & 5.57 & {0.187} & 192 & {0.363} & 0.147 & {0.0874}\\
03074500 & 2.24 & 0.133 & {0.1} & 0.215 & {0.0831} & 0.324 & {0.146} & 0.139 & {0.0845} & 5.49 & {0.212} & 200 & {0.302} & 0.13 & {0.0625}\\
03149500 & 1.51 & 0.136 & {0.124} & 0.333 & {0.146} & 0.541 & {0.211} & 0.164 & {0.0829} & 3.75 & {0.161} & 174 & {0.337} & 0.202 & {0.096}\\
03361850 & 1.75 & 0.186 & {0.088} & 0.497 & {0.114} & 0.902 & {0.211} & 0.17 & {0.0657} & 4.17 & {0.173} & 289 & {0.34} & 0.177 & {0.0652}\\
\end{longtable}

\begin{longtable}{l r rr rr rr rr rr rr rr}
\caption{Streamflow: $\mathcal{E}_{\mathrm{tail}}$ of each gauge (USGS site number, sorted by drainage area) with its Hill tail index $\hat\alpha$, for the seven pre-trained models (pre.) and the fine-tuned models (f.t.).}\label{tab:sf_tail}\\ \toprule
 & & \multicolumn{2}{c}{Lip-KL-GPA} & \multicolumn{2}{c}{TTF} & \multicolumn{2}{c}{Tail-GAN} & \multicolumn{2}{c}{t-EDM} & \multicolumn{2}{c}{SHD} & \multicolumn{2}{c}{DLPM} & \multicolumn{2}{c}{mTAF}\\
\cmidrule(lr){3-4} \cmidrule(lr){5-6} \cmidrule(lr){7-8} \cmidrule(lr){9-10} \cmidrule(lr){11-12} \cmidrule(lr){13-14} \cmidrule(lr){15-16}
gauge & $\hat\alpha$ & pre. & f.t. & pre. & f.t. & pre. & f.t. & pre. & f.t. & pre. & f.t. & pre. & f.t. & pre. & f.t.\\ \midrule \endfirsthead
\toprule  & & \multicolumn{2}{c}{Lip-KL-GPA} & \multicolumn{2}{c}{TTF} & \multicolumn{2}{c}{Tail-GAN} & \multicolumn{2}{c}{t-EDM} & \multicolumn{2}{c}{SHD} & \multicolumn{2}{c}{DLPM} & \multicolumn{2}{c}{mTAF}\\
\cmidrule(lr){3-4} \cmidrule(lr){5-6} \cmidrule(lr){7-8} \cmidrule(lr){9-10} \cmidrule(lr){11-12} \cmidrule(lr){13-14} \cmidrule(lr){15-16}
gauge & $\hat\alpha$ & pre. & f.t. & pre. & f.t. & pre. & f.t. & pre. & f.t. & pre. & f.t. & pre. & f.t. & pre. & f.t.\\ \midrule \endhead \bottomrule \endfoot
03284525 & 1.81 & 2.08 & {0.248} & 0.372 & {0.142} & 1.5 & {0.286} & 0.552 & {0.067} & 3.04 & {0.266} & 3.95 & {0.325} & 0.323 & {0.18}\\
03357350 & 1.22 & 1.41 & {0.231} & 0.626 & {0.27} & 1.76 & {0.296} & 0.264 & {0.0961} & 3.79 & {0.426} & 3.95 & {0.397} & 0.427 & {0.143}\\
03301900 & 1.50 & 2.31 & {0.107} & 0.167 & {0.0385} & 1.28 & {0.227} & 1.25 & {0.132} & 3.17 & {0.304} & 3.93 & {0.258} & 1.16 & {0.152}\\
03260100 & 1.42 & 1.81 & {0.0875} & 0.475 & {0.131} & 1.25 & {0.118} & 0.91 & {0.131} & 2.6 & {0.158} & 3.94 & {0.194} & 1.28 & {0.221}\\
03337000 & 1.82 & 0.138 & {0.122} & 0.29 & {0.168} & 0.364 & {0.162} & 0.205 & {0.0638} & 1.66 & {0.213} & 3.95 & {1.22} & 0.146 & {0.0906}\\
03337100 & 1.86 & 0.242 & {0.148} & 0.308 & {0.208} & 0.437 & {0.164} & 0.225 & {0.126} & 0.606 & {0.101} & 3.95 & {1.26} & 0.17 & {0.128}\\
03298135 & 1.58 & 2.3 & {0.0582} & 0.343 & {0.176} & 0.392 & {0.126} & 1.32 & {0.104} & 3.3 & {0.112} & 3.94 & {0.15} & 1.05 & {0.058}\\
03049800 & 1.84 & 1.36 & {0.229} & 0.16 & {0.121} & 0.365 & 0.903 & 0.198 & {0.174} & 1.06 & {0.218} & 3.95 & {0.345} & 0.397 & 1.01\\
03298300 & 1.47 & 3.08 & {0.0832} & 0.299 & {0.0998} & 1.46 & {0.109} & 0.362 & {0.0831} & 3.81 & {0.137} & 3.93 & {0.169} & 1.14 & {0.118}\\
03353120 & 1.38 & 1.45 & {0.102} & 0.212 & {0.0926} & 1.36 & {0.123} & 0.111 & {0.11} & 3.25 & {0.212} & 3.95 & {0.235} & 0.452 & {0.113}\\
03026500 & 2.31 & 1.13 & {0.252} & 0.194 & {0.0702} & 0.769 & 0.924 & 0.248 & {0.102} & 2.28 & {0.0942} & 3.95 & {1} & 1.1 & {0.245}\\
03298550 & 1.39 & 2.44 & {0.127} & 0.269 & {0.267} & 0.687 & {0.154} & 0.332 & {0.0861} & 3.84 & {0.203} & 3.95 & {0.203} & 0.511 & {0.148}\\
03274650 & 1.43 & 2.36 & {0.108} & 0.529 & {0.126} & 1.33 & {0.12} & 0.135 & {0.0804} & 2.28 & {0.0639} & 3.93 & {0.266} & 1.17 & {0.172}\\
03298150 & 1.45 & 2.45 & {0.0472} & 0.285 & {0.129} & 1.24 & {0.0821} & 0.553 & {0.0655} & 2.37 & {0.116} & 3.93 & {0.0963} & 1.24 & {0.102}\\
03431300 & 1.67 & 1.18 & {0.132} & 0.179 & 0.232 & 1.43 & {0.189} & 0.245 & {0.107} & 3.71 & {0.279} & 3.95 & {0.327} & 0.523 & {0.108}\\
03611260 & 1.16 & 1.42 & {0.135} & 0.6 & {0.117} & 1.44 & {0.353} & 0.414 & {0.113} & 3.91 & {0.293} & 3.95 & {0.468} & 1.46 & {0.361}\\
03353620 & 1.40 & 2.26 & {0.145} & 0.0956 & 0.162 & 1.49 & {0.16} & 0.127 & 0.147 & 3.43 & {0.262} & 3.93 & {0.344} & 1.11 & {0.122}\\
03351310 & 1.53 & 2.06 & {0.205} & 0.168 & {0.0964} & 1.26 & {0.271} & 0.157 & {0.135} & 2.83 & {0.141} & 3.94 & {1.16} & 1.16 & {0.188}\\
03293000 & 1.48 & 3.17 & {0.0778} & 0.179 & {0.109} & 1.37 & {0.105} & 1.25 & {0.0675} & 3.35 & {0.127} & 3.94 & {0.194} & 1.08 & {0.114}\\
03302680 & 1.53 & 2.4 & {0.314} & 0.165 & {0.0877} & 1.43 & {0.4} & 0.151 & {0.0569} & 3.26 & {0.47} & 3.94 & {0.437} & 0.499 & {0.239}\\
03289000 & 2.04 & 2.18 & {0.143} & 0.319 & {0.11} & 0.474 & {0.152} & 0.464 & {0.0594} & 3.17 & {0.304} & 3.95 & {0.33} & 0.26 & {0.105}\\
03262001 & 1.34 & 3.34 & {0.103} & 0.459 & {0.0682} & 1.48 & {0.0891} & 0.973 & {0.0972} & 3.57 & {0.12} & 3.93 & {0.134} & 2.22 & {0.234}\\
03431700 & 1.66 & 1.13 & {0.122} & 0.227 & {0.115} & 1.42 & {0.134} & 0.336 & {0.0797} & 3.17 & {0.0946} & 3.95 & {0.322} & 0.584 & {0.132}\\
03353600 & 1.78 & 2.2 & {0.0888} & 0.157 & {0.0746} & 0.481 & {0.0918} & 0.115 & {0.107} & 3.43 & {0.249} & 3.94 & {0.195} & 1.18 & {0.139}\\
03140000 & 1.35 & 1.44 & {0.258} & 0.209 & {0.167} & 1.54 & {0.332} & 0.115 & {0.0529} & 1.05 & {0.337} & 3.93 & {1.23} & 0.426 & {0.263}\\
03291780 & 1.37 & 2.67 & {0.127} & 0.372 & {0.0998} & 1.56 & {0.156} & 0.0959 & {0.0469} & 3.76 & {0.166} & 3.94 & {0.271} & 1.33 & {0.148}\\
03289200 & 2.30 & 3.11 & {0.121} & 0.426 & {0.333} & 0.678 & {0.0943} & 0.421 & {0.058} & 2.59 & {0.196} & 3.94 & {0.141} & 0.297 & {0.0971}\\
03254550 & 1.53 & 3.16 & {0.128} & 0.361 & {0.112} & 1.36 & {0.193} & 0.25 & {0.0773} & 3.26 & {0.267} & 3.94 & {1.04} & 0.442 & {0.149}\\
03213500 & 1.54 & 1.14 & {0.205} & 0.188 & {0.167} & 1.36 & {0.204} & 0.247 & {0.117} & 3.08 & {0.354} & 3.95 & {0.358} & 1.24 & {0.229}\\
03339280 & 1.55 & 2.28 & {0.0837} & 0.107 & 0.202 & 1.3 & {0.121} & 0.0345 & 0.0448 & 1.38 & {0.0649} & 3.93 & {0.239} & 1.15 & {0.145}\\
03230450 & 1.79 & 2.18 & {0.127} & 0.193 & 0.255 & 0.459 & {0.155} & 0.215 & {0.0634} & 2.62 & {0.136} & 3.94 & {0.956} & 1.19 & {0.945}\\
03266560 & 1.80 & 2.2 & {0.225} & 0.395 & {0.074} & 0.785 & {0.208} & 1.01 & {0.121} & 2.75 & {0.27} & 3.95 & {0.546} & 0.409 & {0.195}\\
03277075 & 1.49 & 3.36 & {0.0335} & 0.504 & {0.113} & 0.46 & {0.0581} & 0.243 & {0.0683} & 1.48 & {0.0582} & 3.93 & {0.16} & 2.09 & {0.118}\\
03223425 & 1.70 & 2.26 & {0.133} & 0.204 & 0.233 & 0.589 & {0.145} & 0.152 & {0.1} & 2.47 & {0.175} & 3.94 & {0.265} & 0.412 & {0.178}\\
03011800 & 2.64 & 1.14 & {0.222} & 0.266 & {0.141} & 0.854 & 0.951 & 0.179 & {0.0679} & 3.79 & {0.317} & 3.95 & {1.01} & 1.16 & {0.232}\\
03238745 & 1.28 & 4.08 & {0.145} & 0.523 & {0.159} & 1.21 & {0.167} & 0.148 & {0.0642} & 3.63 & {0.183} & 3.93 & {0.259} & 1.3 & {0.257}\\
03260706 & 1.83 & 2.14 & {0.202} & 0.0898 & 0.22 & 0.507 & {0.198} & 0.274 & {0.084} & 2.36 & {0.196} & 3.94 & {0.43} & 0.36 & {0.176}\\
03384450 & 1.05 & 2.3 & {0.125} & 0.582 & {0.235} & 1.66 & {0.248} & 0.273 & {0.129} & 3.6 & {0.306} & 3.94 & {0.346} & 1.29 & {0.29}\\
03177710 & 2.22 & 0.421 & {0.188} & 0.0902 & 0.134 & 0.688 & {0.253} & 0.269 & {0.147} & 1.84 & {0.127} & 3.95 & {0.347} & 1.08 & {0.22}\\
03076600 & 2.13 & 1.26 & {0.185} & 0.15 & 0.152 & 0.537 & {0.197} & 0.237 & {0.0841} & 1.06 & {0.199} & 3.94 & {1.01} & 0.353 & {0.181}\\
03330241 & 2.92 & 2.26 & {1.43} & 0.318 & 0.333 & 0.939 & 3.84 & 0.223 & {0.0946} & 2.56 & {0.39} & 3.96 & {0.349} & 1.11 & {1.03}\\
03350700 & 1.66 & 2.97 & {0.174} & 0.082 & {0.0732} & 0.6 & {0.178} & 0.134 & {0.112} & 3.36 & {0.293} & 3.93 & {0.288} & 0.3 & {0.106}\\
03431599 & 1.42 & 1.13 & {0.147} & 0.104 & 0.194 & 1.45 & {0.18} & 0.294 & {0.101} & 3.27 & {0.137} & 3.95 & {0.403} & 0.363 & {0.122}\\
03106300 & 3.34 & 1.46 & {0.47} & 0.339 & {0.213} & 1.24 & {0.707} & 0.372 & {0.171} & 3.03 & {0.527} & 3.96 & {0.384} & 0.249 & {0.21}\\
03185000 & 1.80 & 1.15 & {0.237} & 0.22 & {0.0941} & 0.592 & {0.204} & 0.196 & {0.137} & 3.17 & {0.392} & 3.96 & {0.335} & 1.1 & {0.25}\\
03244936 & 1.31 & 3.23 & {0.193} & 0.242 & {0.166} & 1.47 & {0.177} & 0.336 & {0.0789} & 3.95 & {0.174} & 3.94 & {0.987} & 1.43 & {0.281}\\
03210000 & 1.75 & 1.15 & {0.203} & 0.139 & 0.251 & 0.71 & {0.176} & 0.219 & {0.116} & 2.81 & {0.137} & 3.95 & {0.386} & 1.22 & {0.257}\\
03042280 & 2.24 & 1.24 & {0.232} & 0.261 & {0.184} & 0.843 & 0.961 & 0.162 & {0.0726} & 3.35 & {0.435} & 3.96 & {1.11} & 1.08 & {0.227}\\
03067510 & 2.03 & 1.26 & {0.145} & 0.188 & 0.283 & 0.651 & 1.03 & 1.12 & {0.0986} & 3.02 & {0.316} & 3.95 & {1.17} & 0.425 & {0.241}\\
03280700 & 1.73 & 1.18 & {0.168} & 0.268 & {0.188} & 0.626 & {0.183} & 0.0977 & {0.079} & 2.37 & {0.211} & 3.94 & {0.275} & 0.42 & {0.152}\\
03062500 & 2.75 & 1.35 & {0.12} & 0.229 & {0.19} & 0.702 & {0.192} & 0.166 & {0.0805} & 3.53 & {0.133} & 3.95 & {0.169} & 0.291 & {0.119}\\
03241500 & 1.74 & 1.53 & {0.117} & 0.151 & {0.0984} & 0.527 & {0.102} & 0.239 & {0.0617} & 2.43 & {0.0938} & 3.95 & {1.01} & 0.465 & {0.182}\\
03302000 & 1.63 & 3.29 & {0.127} & 0.13 & {0.0516} & 1.22 & {0.0984} & 1.06 & {0.0809} & 1.66 & {0.0517} & 3.94 & {0.106} & 1.1 & {0.123}\\
03078000 & 2.21 & 1.27 & {0.102} & 0.165 & 0.188 & 0.603 & {0.164} & 0.212 & {0.0523} & 3.93 & {0.211} & 3.94 & {0.282} & 0.42 & {0.129}\\
03228750 & 1.43 & 2.31 & {0.123} & 0.286 & {0.194} & 1.28 & {0.195} & 0.317 & {0.0842} & 3.43 & {0.272} & 3.94 & {1.05} & 0.476 & {0.243}\\
03302220 & 1.68 & 3.26 & {0.0788} & 0.185 & {0.173} & 1.36 & {0.0845} & 1.14 & {0.0581} & 2.92 & {0.149} & 3.94 & {0.119} & 1.19 & {0.0971}\\
03277300 & 1.97 & 1.07 & {0.111} & 0.183 & 0.203 & 1.26 & {0.107} & 0.233 & {0.0489} & 3.01 & {0.203} & 3.94 & {0.358} & 0.446 & {0.174}\\
03208950 & 1.83 & 1.04 & {0.173} & 0.156 & 0.184 & 1.5 & {0.108} & 0.169 & {0.0754} & 1.62 & {0.215} & 3.94 & {0.462} & 0.573 & {0.217}\\
03271000 & 1.32 & 3.18 & {0.149} & 0.345 & {0.0845} & 1.34 & {0.176} & 0.121 & {0.118} & 2.3 & {0.0558} & 3.93 & {0.294} & 1.25 & {0.175}\\
03272700 & 1.44 & 2.46 & {0.0963} & 0.166 & {0.0925} & 1.19 & {0.107} & 0.22 & {0.168} & 3.4 & {0.211} & 3.93 & {0.224} & 1.23 & {0.149}\\
03141870 & 1.73 & 1.38 & {0.138} & 0.181 & 0.283 & 0.476 & {0.153} & 0.413 & {0.0766} & 1.74 & {0.297} & 3.94 & {0.286} & 0.396 & {0.165}\\
03074500 & 2.24 & 1.98 & {0.932} & 0.12 & 0.138 & 0.737 & 0.887 & 0.133 & 0.185 & 2.77 & {0.98} & 3.95 & {0.453} & 1.06 & {0.188}\\
03149500 & 1.51 & 2.14 & {0.258} & 0.108 & 0.224 & 1.3 & {0.18} & 0.355 & {0.089} & 1.99 & {0.188} & 3.93 & {1.03} & 1.18 & {0.214}\\
03361850 & 1.75 & 2.36 & {0.154} & 0.162 & 0.171 & 0.501 & {0.202} & 0.0989 & {0.0628} & 1.5 & {0.23} & 3.94 & {0.219} & 1.03 & {0.122}\\
\end{longtable}
}

\section{CVaR-penalized Wasserstein gradient flows with Clarke's generalized subgradients}
\label{Sec:CVaR:Loss:Clarke}

In this Appendix, we analyze the CVaR-penalized Wasserstein gradient flows defined using  CVaR  (as opposed to smoothed CVaR). All the relevant proofs, and the statements of auxiliary lemmas accompanied by their proofs, are presented in Appendix \ref{App:Clarke:Proofs}. 
We define the loss functional  $\Fg$
\begin{align}
\label{App:Eq:penalized:loss:functional}
\Fg(Q;\Ptar)= \LipKL  \QPtar + \lambda \, \left(\Delta C (Q; \Ptar)\right)^2,
\end{align}
where $\Delta C (Q; \Ptar):=\mathrm{CVaR}_\alpha^{\Ptar,g} -\mathrm{CVaR}_\alpha^{ Q,g}$ denotes the
CVaR discrepancy of $g(X)$ for $X \sim Q$ and $X \sim \Ptar$, for a given risk function $g:\R^d \rightarrow [0,\infty)$.

Throughout the appendix, we use the notation $\Fg_h$ to denote the loss functional penalized with smoothed CVaR, and $\Fg$ to denote the functional penalized with CVaR.

\begin{Definition}
\label{def:VaR}
    Let $g: \R^d  \rightarrow [0,\infty)$ be a risk function, and $Q \in \mathcal{Q}^g$. Given a parameter $0 < \alpha < 1$, the upper and lower $\alpha$-quantiles under $Q$  are given by
\begin{align*}
&\overline{\mathrm{VaR}}^{Q,g}_{\alpha} := \inf\{ c : \Psi^{Q,g}(c) > \alpha \},\\
&\mathrm{VaR}_{\alpha}^{Q, g}  := \inf_{c \in \mathbb{R}}\{ c : \Psi^{Q,g}(c) \geq \alpha \},
\end{align*}
where $\Psi^{Q,g}$ denotes the cumulative
distribution function of $g$ under $Q$.
\end{Definition}
CVaR is defined via the Rockafellar--Uryasev formulation \citep{Rock2002CVaR} {(see also \eqref{eq:cvar_definition_APP})}
\begin{align}
\label{Eq:Rock_App}
\mathrm{CVaR}^{Q,g}_{\alpha}
= \inf_{y\in \mathbb{R}}
F^{Q,g}_{\alpha}(y):= \inf_{y\in \mathbb{R}} {\Big\{} y + \frac{1}{1-\alpha}  \mathbb{E}_{Q} [(g(\cdot)-y)^+]{\Big\}}.
\end{align}
Then the set of minimizers of \eqref{Eq:Rock_App}, denoted by $T(Q)$, is $[ \mathrm{VaR}^{Q,g}_{\alpha},
\overline{\mathrm{VaR}}^{Q,g}_{\alpha}]$ (see \cite[Theorem 10]{Rock2002CVaR} for a proof).

Under Assumption \ref{Assump}, the loss functional $\Fg$ is well defined, and the argument of Proposition \ref{Prop:Div:Property} shows that it is a divergence, so that $\Ptar$ is the unique global minimizer of the optimization problem $\inf_{Q\in\Qg}\Fg(Q;\Ptar)$. Its Wasserstein gradient flow, however, requires a variational derivative, which does not exist in general: $\CVaR^{Q,g}_\alpha$ is Gateaux differentiable at $Q$ if and only if $\VaR^{Q,g}_\alpha=\VaRbar^{Q,g}_\alpha$ (Theorem \ref{Thm:VD:CVaR}(4)), and the same holds for $\Fg$ whenever $\Delta C(Q;\Ptar)\neq0$ (Theorem \ref{Theorem:VD:Loss}(3)). This condition fails whenever $\Psi^{Q,g}=\alpha$ on a nondegenerate interval, and, for an empirical measure, whose CDF $\Psi^{Q,g}$ is a step function, such levels $\alpha$ always exist. We therefore work with variational subgradients, defined via Clarke's generalized gradients \citep{clarke1990optimization}.

\subsection{Clarke's generalized subgradients}
\label{app:clarke}

Let $Q\mapsto J(Q)$ be a locally Lipschitz continuous (with respect to the Wasserstein-$1$ metric $\mathcal{W}_1$ on $\PR$) real-valued functional,  defined on $\PR$. The Clarke generalized directional derivative  \citep{clarke1990optimization} of $J$ at $Q$ for an admissible perturbation $\rho$ {(Definition \ref{def:admissible})} is defined by
\begin{align}
J^\circ(Q;\rho):=\limsup_{Q'\to Q,\ t\downarrow0}\frac{J(Q'+t\rho)-J(Q')}{t}.
\end{align}
The \emph{Clarke subdifferential} of $J$ at $Q$ is defined as
\begin{align}
\partial J(Q):=\Big\{\xi:\ J^\circ(Q;\rho)\ge\textstyle\int\xi\,d\rho\ \ \text{for all admissible perturbations }\rho\Big\},
\end{align}
and its elements are referred to as the first variational subgradients of $J$ at $Q$.

\paragraph{Convex and concave cases}
When $J$ is \emph{convex}, its Clarke subdifferential coincides with the classical \emph{subdifferential} of convex analysis
\begin{align}
\partial J(Q):=\Big\{\xi:\ J(Q')\ge J(Q)+\int \xi\,d(Q'-Q)\ \ \forall Q'\in\PR\Big\}.
\end{align}
When $J$ is \emph{concave}, the Clarke subdifferential instead coincides with the \emph{superdifferential}
\begin{align}
\partial J(Q):=\Big\{\xi:\ J(Q')\le J(Q)+\int \xi\,d(Q'-Q)\ \ \forall Q'\in\PR\Big\},
\end{align}
and its elements are correspondingly called \emph{supergradients} rather than subgradients.

The risk measure $\CVaR_\alpha^{Q,g}$ is an infimum of a family of functions affine in $Q$, by the Rockafellar--Uryasev formulation \eqref{Eq:Rock_App}. Thus $\CVaR_\alpha^{\cdot,g}$ is concave in $Q$. However, in this paper we use the single term \emph{variational subgradients} throughout to refer to the elements of $\partial J$ for any given functional $J$, regardless of whether $J$ is convex, concave, or neither.

\paragraph{Local Lipschitz continuity with respect to the Wasserstein-$1$ metric}
{Clarke's generalized directional derivative is finite when $J$ is locally Lipschitz continuous with respect to a metric on $\PR$ (Lemma \ref{lem:clarke:finite}). Since the perturbation $Q^\epsilon=Q+\epsilon\rho$ displaces probability mass, we use the Wasserstein-$1$ metric, which, unlike metrics insensitive to spatial displacement such as the total variation norm, measures the cost of this displacement. Wasserstein gradient flows are classically formulated in $(\mathcal{P}_2(\R^d),\mathcal{W}_2)$ \citep{jordan1998variational, otto2001geometry}; since $\mathcal{W}_1\le\mathcal{W}_2$, local Lipschitz continuity with respect to $\mathcal{W}_1$ implies local Lipschitz continuity with respect to $\mathcal{W}_2$.}
\begin{Lemma}
\label{lem:clarke:finite}
    Let $J(Q)$ be locally Lipschitz continuous with respect to the Wasserstein-$1$ metric on $\PR$. Then for any admissible perturbation $\rho$, the Clarke generalized directional derivative $J^\circ(Q;\rho)$ exists and is finite.
\end{Lemma}

{Throughout this appendix, the risk function $g$ is Lipschitz continuous. This ensures that $Q\mapsto\Fg(Q;\Ptar)$ is locally Lipschitz continuous with respect to $\mathcal{W}_1$, so that its variational subgradients are well defined.}

\begin{Proposition}[Local Lipschitz continuity]
\label{lem:W1-lipschitz}
Assume that $g$ is $\Lip(g)$-Lipschitz continuous and Assumption \ref{Assump} holds. Then the functional $Q \mapsto \Fg(Q;\Ptar)$ is locally Lipschitz continuous on $\mathcal{Q}^g$ with respect to $\mathcal{W}_1$.
\end{Proposition}

{The variational subgradients of $\CVaR^{Q,g}_\alpha$ and $\Fg(Q;\Ptar)$ are identified through the following lemma.}

\begin{Lemma}
\label{Lem:interval:subgradient}
 Let $J(Q)$ be locally Lipschitz continuous with respect to the Wasserstein-$1$ metric on $\PR$. Define the
 right/left one-sided directional derivatives of $J$ as:
\begin{align}
    D^\pm_\rho  J:=\lim_{\epsilon\to0^\pm} \frac{1}{\epsilon}\big(J(Q^\epsilon)-J(Q)\big).
\end{align}
Assume that for every admissible perturbation $\rho$ at $Q$, the one-sided directional derivatives
$D^{\pm}_\rho J$ exist and $\int \xi\, d\rho$ lies in the closed interval with endpoints $D^-_\rho J$ and
$D^+_\rho J$. Then $\xi$ is a variational subgradient of $J$ at $Q$.
\end{Lemma}

\paragraph{Variational subgradients of \texorpdfstring{\boldmath$\CVaR^{Q,g}_\alpha$}{CVaR\string^{Q,g}_alpha}}
{Since $g$ is non-negative and Lipschitz continuous, $0\le g(x)\le g(0)+\Lip(g)\|x\|$, so that the condition $\int g\,d|\rho|<\infty$ in Definition \ref{def:admissible} follows from $\int\|x\|\,d|\rho|<\infty$:}
\begin{align}
\label{Eq:g:integrability}
\int g \, d|\rho| \;\le\; g(0)\int d|\rho| \;+\; \Lip(g)\int \|x\| \, d|\rho| \;<\; \infty.
\end{align}
{In particular, by \eqref{Eq:g:integrability:h}, $\CVaR_\alpha^{Q^\epsilon,g}$ is well defined for every right- (left-) admissible $\rho$ and all small $\epsilon>0$ ($\epsilon<0$).}

\begin{Theorem}[Directional derivatives and variational subgradients of $\CVaR$]
\label{Thm:VD:CVaR}
Assume that $g$ is Lipschitz continuous and $Q$ has finite first moment (hence $Q \in \Qg$). For $y\in\R$, define
$h^y(x):=\frac{(g(x)-y)^+}{1-\alpha}$.
Then we have the following:
\begin{enumerate}
\item \emph{(Right derivative)} For every right-admissible perturbation $\rho$, the right directional
derivative exists and is given by
\begin{align}\label{eq:dir}
D^+_\rho \CVaR_\alpha^{Q,g}=\lim_{\epsilon\to0^+}\frac{\CVaR_\alpha^{Q^\epsilon,g}-\CVaR_\alpha^{Q,g}}
{\epsilon}=\min_{y\in T(Q)}\int h^y\,d\rho.
\end{align}
\item \emph{(Left derivative)}  For every left-admissible perturbation $\rho$, the left directional derivative
exists and is given by
\begin{align}\label{eq:left}
D^-_\rho \CVaR_\alpha^{Q,g} =\lim_{\epsilon\to0^-}\frac{\CVaR_\alpha^{Q^\epsilon,g}-\CVaR_\alpha^{Q,g}}
{\epsilon}=\max_{y\in T(Q)}\int h^y\,d\rho.
\end{align}
\item \emph{(Variational subgradients)} Let $y \in T(Q)$. Then, for any admissible $\rho$, $\int h^y\,d\rho$
lies in $[D^+_\rho\CVaR_\alpha^{Q,g},D^-_\rho\CVaR_\alpha^{Q,g}]$. Consequently, for each $y\in T(Q)$, $h^y$ is
a variational subgradient of $\CVaR_\alpha^{Q,g}$.
\item \emph{(Two-sided derivative)} The two-sided (Gateaux) derivative for an admissible $\rho$ exists if and
only if $y\mapsto\int h^y\,d\rho$ is constant on $T(Q)$; this holds for every admissible $\rho$ if and only if
$T(Q)$ is a singleton set, i.e.\ $\VaR_\alpha^{Q,g}=\VaRbar_\alpha^{Q,g}$, in which case $\CVaR_\alpha^{Q,g}$
is Gateaux differentiable at $Q$ with the first variation
\begin{equation*}
\frac{\delta\CVaR_\alpha^{Q,g}}{\delta Q}=h^{\VaR_\alpha^{Q,g}}=\frac{(g-\VaR_\alpha^{Q,g})^+}{1-\alpha}.
\end{equation*}
\end{enumerate}
\end{Theorem}

\begin{Remark}
\label{rem:classical:formula}
{For the radial risk function $g(x)=\|x\|$, the hypotheses of Theorem \ref{Thm:VD:CVaR} reduce to $Q\in\Qg$, which is exactly what is needed for $\CVaR_\alpha^{Q,g}$ to be well defined; no hypothesis on the set $T(Q)$ is imposed. Part (4) recovers the formula $\frac{\delta \CVaR_\alpha^{Q,g}}{\delta Q}=\frac{(g -\VaR_\alpha^{Q,g})^+}{1-\alpha}$ employed in \cite{NEURIPS2025_10715deb, wang2026efficient}, where it is derived under a density assumption on the law of $g(X)$, $X\sim Q$, that makes $T(Q)$ a singleton. This assumption fails for empirical measures, for which the variational subgradients of part (3) are required.}
\end{Remark}
We now consider the variational subgradients of the loss functional $\Fg$ itself.

\begin{Theorem}[Directional derivatives and variational subgradients of $\Fg$]
\label{Theorem:VD:Loss}
Assume that $g$ is Lipschitz continuous and Assumption \ref{Assump} holds. Let $\phi^*$ denote the maximizer of
\eqref{equation:Variational_KL_Lip}. For each $y\in T(Q)$, define the potential function
\begin{align}\label{eq:Phi-family}
\Phi_Q^{y}:=\phi^*-2\lambda\,\Delta C(Q; \Ptar)\,h^y
=\phi^*-\frac{2\lambda}{1-\alpha}\,\Delta C(Q; \Ptar)\,(g-y)^+ .
\end{align}
Then the following hold.
\begin{enumerate}
\item  For every right-admissible perturbation $\rho$, the right directional derivative exists and is given by
\begin{align}\label{eq:F-right}
D^+_\rho\Fg(Q;\Ptar)=\int\phi^*\,d\rho-2\lambda\,\Delta C(Q; \Ptar)\,\min_{y\in T(Q)} \int h^y d \rho.
\end{align}
For every left-admissible perturbation $\rho$, the left directional derivative exists and is given by
\begin{align}\label{eq:F-left}
D^-_\rho\Fg(Q;\Ptar)=\int\phi^*\,d\rho-2\lambda\,\Delta C(Q; \Ptar)\,\max_{y\in T(Q)}\int h^y d \rho.
\end{align}
\item For any admissible $\rho$, the range of $\int\Phi_Q^{y} d\rho$ for $y \in T(Q)$ is the compact interval
between $D^-_\rho\Fg(Q;\Ptar)$ and $D^+_\rho\Fg(Q;\Ptar)$. Consequently, the family of potential functions
$\{\Phi_Q^{y}; y \in T(Q) \}$ are subgradients of $\Fg(Q;\Ptar)$.
\item Let $\rho$ be admissible. The two-sided (Gateaux) derivative of $\Fg$ for the perturbation $\rho$ exists
if and only if $\Delta C(Q; \Ptar)=0$ or $y\mapsto\int h^y\,d\rho$ is constant on $T(Q)$. The two-sided
(Gateaux) derivative of $\Fg$ exists for every admissible perturbation if and only if $\Delta C(Q; \Ptar)=0$ or
$T(Q)$ is a singleton set; in that case, $\Fg(\cdot;\Ptar)$ is Gateaux differentiable at $Q$ with first
variational derivative
\begin{equation*}
\frac{\delta\Fg(Q;\Ptar)}{\delta Q}=\Phi_Q^{\VaR_\alpha^{Q,g}}.
\end{equation*}
\end{enumerate}
\end{Theorem}

\subsection{Velocity field of the CVaR-penalized  Wasserstein gradient flow}
\label{Sec:Velocity}
{By Theorem \ref{Theorem:VD:Loss}(3), the variational derivative of $\Fg(Q;\Ptar)$, which defines the velocity field of its Wasserstein gradient flow, does not exist unless $T(Q)$ is a singleton or $\Delta C(Q;\Ptar)=0$. We therefore define the velocity field through a variational subgradient. By Theorem \ref{Theorem:VD:Loss}(2), every $\Phi_Q^{y}$, $y\in T(Q)$, is a variational subgradient, and the choice of $y$ is a genuine degree of freedom in the design of the flow.}

\begin{Theorem}[The velocity field of $\Fg$]\label{thm:vel}
Assume that $g$ is Lipschitz continuous, $g\in C^1$ for $g(x) \neq 0$, and Assumption \ref{Assump} holds. For
each $y\in T(Q)$, choosing the potential function $\Phi_Q^{y}$ induces the velocity field
$v^{y}_Q=-\nabla_x\Phi_Q^{y}$, and we have
\begin{align}\label{eq:vel-cases}
v^{y}_Q(x)=
\begin{cases}
\displaystyle \underbrace{-\nabla_x\phi^*(x)}_{\text{from }\LipKL(Q\|\Ptar)}
+\underbrace{\frac{2\lambda}{1-\alpha}\,\Delta C(Q; \Ptar)\,\nabla_x g(x)}_{\text{from CVaR-penalization}}, &
g(x)>y,\\
\displaystyle \underbrace{-\nabla_x\phi^*(x)}_{\text{from }\LipKL(Q\|\Ptar)}, & g(x)\le y.
\end{cases}
\end{align}
If $T(Q)$ is a singleton set, the velocity field degenerates to $v^{\VaR_\alpha^{Q,g}}_Q$.
\end{Theorem}

\begin{Corollary}[Boundedness of the velocity field]
\label{cor:vel-bound}
Assume that $g$ is Lipschitz continuous, $g\in C^1$ for $g(x) \neq 0$, and Assumption \ref{Assump} holds. For every $y\in
T(Q)$, the velocity field $v^{y}_Q$ in \eqref{eq:vel-cases} satisfies, for almost every $x\in\R^d$,
\begin{align}\label{eq:vel-bound}
\big\|v^{y}_Q(x)\big\|\;\le\; L+\frac{2\lambda}{1-\alpha}\,\Lip(g)\,\big|\Delta C(Q;\Ptar)\big|.
\end{align}
Moreover, the CVaR discrepancy is controlled by the loss $\Fg(Q;\Ptar)$, so we have
\begin{align}\label{eq:vel-bound-energy}
\big\|v^{y}_Q(x)\big\|\;\le\; L+\frac{2\,\Lip(g)}{1-\alpha}\,\sqrt{\lambda\,\Fg(Q;\Ptar)},
\end{align}
for almost every $x\in\R^d$. In particular, the velocity field is bounded, and the CVaR-penalization component vanishes as
$\Fg(Q;\Ptar)\to 0$, i.e.\ as $Q\to\Ptar$.
\end{Corollary}

\begin{Remark}[Bounded but non-Lipschitz continuous]
\label{rem:non-lip}
Although  Corollary \ref{cor:vel-bound} shows that the velocity field $v^{y}_Q$ in \eqref{eq:vel-cases} is bounded, it is not Lipschitz continuous in $x$,
for two reasons. First, $v^{y}_Q$ has a jump of magnitude
$\frac{2\lambda}{1-\alpha}\,|\Delta C(Q;\Ptar)|\,\|\nabla_x g\|$ at the threshold set $\{x:g(x)=y\}$.
Second, $\nabla_x\phi^*$ is bounded due to $\phi^*$ being Lipschitz continuous, but $\nabla_x\phi^*$ is not necessarily continuous.

The bounded yet discontinuous velocity field $v^{y}_Q$ in \eqref{eq:vel-cases} distinguishes the
CVaR-penalized flow from generative models built on Lipschitz transport maps: the latter preserve the tail
behavior of the light-tailed source distribution \citep{jaini2020tails}, whereas a non-Lipschitz transport map
is required to transport the pre-trained distribution $\Ppre$ toward the heavier-tailed target distribution $\Ptar$. We do not claim here a rigorous
quantitative change of the tail behavior; the resulting gain in tail accuracy is demonstrated
empirically in \Cref{Sec:Results}.
\end{Remark}

{Throughout this paper, we adopt the radial risk function $g(x)=\|x\|$, which is $1$-Lipschitz continuous and $C^1$ on $\{x\neq0\}$, so that Theorem \ref{thm:vel} applies and yields the following velocity field.}

\begin{Corollary}
\label{cor:vel:radial:SG}
Let Assumption \ref{Assump} hold and
let $g(x)=\|x\|$ and $\E_Q[g] < \infty$. For each $y\in T(Q)$, choosing the potential function $\Phi_Q^{y}$ induces the velocity field $v^{y}_Q=-\nabla_x\Phi_Q^{y}$, and we have
\begin{align}
\label{Eq:Velocity:SG}
v^{y}_Q(x)=
\begin{cases}
\displaystyle \underbrace{-\nabla_x\phi^*(x)}_{\text{from }\LipKL(Q\|\Ptar)}
+\underbrace{\frac{2\lambda}{1-\alpha}\,\Delta C(Q; \Ptar)\,\frac{x}{\|x\|}}_{\text{from CVaR-penalization}}, & \|x\|>y,\\
\displaystyle \underbrace{-\nabla_x\phi^*(x)}_{\text{from }\LipKL(Q\|\Ptar)}, & \|x\|\le y.
\end{cases}
\end{align}
\end{Corollary}

\subsection{The velocity contribution from CVaR-penalization}
The CVaR-penalized  Wasserstein gradient flow is given by
\begin{align}
\label{Eq:Evol:Eq}
 \partial_t Q_t + \nabla\cdot(Q_t v^y_{Q_t}) = 0,  \quad Q_0 = \Ppre,
\end{align}
where the velocity field $v^y_{Q_t}$ is given by \eqref{Eq:Velocity:SG}. Its component induced by the CVaR-penalization,
\begin{align}
\label{Eq:vhat:SG}
\hat{v}^y_Q(x):=\frac{2\lambda}{1-\alpha}\,\Delta C(Q;\Ptar)\,\frac{x}{\|x\|}\,\ind_{\{\|x\|>y\}},
\end{align}
depends on the target distribution only through $\CVaR_\alpha^{\Ptar,g}$. We show that the finite-sample estimation error of $\CVaR_\alpha^{\Ptar,g}$ does not depend on the dimension $d$, beginning with the following corollary.

\begin{Corollary}
\label{cor:dim:free:SG}
Assume that $\E_{\Ptar}[g^{q}]<\infty$ for some
$q>1$. Then, for a constant $C(q)$ independent of the dimension $d$ of the target distribution $\Ptar$, we have
\begin{align}
\label{Eq:CVaR:Estimation:APP}
  \E_{(\Ptar)^{\otimes N}} \Big|\CVaR_{\alpha}^{(\widehat{\Ptar})_N,g}-\CVaR_{\alpha}^{\Ptar,g}\Big|
 \le \frac{C(q)}{1-\alpha} (N^{-\frac{1}{2}}+N^{-\frac{(q-1)}{q}}).
\end{align}
If $q=2$, the bound holds with $q$ replaced by any $q'\in(1,2)$.
\end{Corollary}

\begin{proof} Almost identical to the proof of Lemma \ref{Cor:CVaR:Estimation:App} and hence is omitted. 
\end{proof}
{By \eqref{Eq:CVaR:Estimation:APP}, the number of target samples required to estimate $\hat{v}^y_Q$ to a prescribed accuracy does not grow with $d$. By \eqref{Eq:vhat:SG}, $\hat{v}^y_Q$ is active only on $\{\|x\|>y\}$, where its magnitude is $\frac{2\lambda}{1-\alpha}|\Delta C(Q;\Ptar)|$; the choice of $y\in T(Q)$ therefore determines the region on which $\hat{v}^y_Q$ is active, not its magnitude. This magnitude is nonzero whenever $\Delta C(Q;\Ptar)\neq0$, irrespective of the number of samples in the tail region, where the component $-\nabla_x\phi^*$ from $\LipKL$ may vanish due to sample scarcity. Since $\hat{v}^y_Q$ is proportional to $\Delta C(Q;\Ptar)$, it is large when the tail is inadequately captured and vanishes as the CVaR discrepancy closes. The flow \eqref{Eq:Evol:Eq} motivates the particle algorithm stated in Algorithm \ref{alg:gpa:NS}.}

\paragraph{Activation region of the velocity contribution from CVaR-penalization}
{The two endpoints $\VaR_\alpha^{\widehat{Q}_k,g}$ and $\VaRbar_\alpha^{\widehat{Q}_k,g}$ of $T(\widehat{Q}_k)$ yield the largest and the smallest activation regions of $\hat{v}^y_Q$, respectively, and are computed exactly from the order statistics $g_{(1)}\le\cdots\le g_{(M)}$ of the risk values:}
\begin{align}\label{eq:orderstats}
\VaR_\alpha^{\widehat{Q}_k,g}=g_{(\lceil\alpha M\rceil)},\qquad\VaRbar_\alpha^{\widehat{Q}_k,g}=g_{(\lfloor\alpha M\rfloor+1)}.
\end{align}
{If $\alpha M\notin\mathbb Z$, the two endpoints coincide. If $\alpha M\in\mathbb Z$, then $\VaR_\alpha^{\widehat{Q}_k,g}\le\VaRbar_\alpha^{\widehat{Q}_k,g}$, with strict inequality precisely when $g_{(\alpha M)}<g_{(\alpha M+1)}$.}

\paragraph{The update scheme}
{Algorithm \ref{alg:gpa:NS} follows the update scheme \eqref{Eq:Update:Scheme} of CVaR-GPA, with the smoothed tail statistics replaced by their unsmoothed counterparts, computed via Algorithm \ref{alg:tail:NS}, and with the velocity component induced by the CVaR-penalization replaced by
\begin{align}
\label{Eq:b:NS}
b^{(i)}_{k}=\frac{2\Delta_k}{1-\alpha}\,\frac{Y^{(i)}_{k}}{\|Y^{(i)}_{k}\|}\,\ind_{\{\|Y^{(i)}_{k}\|>y_k\}},\qquad y_k\in\big\{\VaR_{\alpha}^{\widehat{Q}_k,g},\VaRbar_{\alpha}^{\widehat{Q}_k,g}\big\},\nonumber \\ \Delta_k=\CVaR_\alpha^{\widehat{\Ptar},g}-\CVaR_\alpha^{\widehat{Q}_k,g}.
\end{align}}

\begin{algorithm}[htbp]
\caption{Empirical tail statistics}
\label{alg:tail:NS}
\begin{algorithmic}[1]
\Require samples $\{Z_i\}_{i=1}^{m}$, $g(\cdot)=\|\cdot\|$: the risk function, $\alpha$:  {level}
\State $g_i\gets g(Z_i)$; sort ascending $g_{(1)}\le\cdots\le g_{(m)}$
\State $\VaR\gets g_{(\lceil\alpha m\rceil)}$, \quad $\VaRbar\gets g_{(\lfloor\alpha m\rfloor+1)}$
\State $\CVaR\gets \VaR+\dfrac{1}{(1-\alpha)\,m}\displaystyle\sum_{i=\lceil\alpha m\rceil+1}^{m}\big(g_{(i)}-\VaR\big)$ \Comment{empirical form of \eqref{Eq:Rock_App}}
\State \Return $(\VaR,\ \VaRbar,\ \CVaR)$
\end{algorithmic}
\end{algorithm}

\begin{algorithm}[htbp]
\caption{CVaR-GPA using subgradients}
\label{alg:gpa:NS}
\begin{algorithmic}[1]
\Require {{Hyperparameters $\alpha$, $\lambda$, $L$, $\Delta t$; risk function $g(\cdot)=\|\cdot\|$; number of pre-trained particles $M$; number of target particles $N$; kinetic energy threshold $\epsilon$; neural network approximation $\Gamma^{\mathrm{NN}}_L$; number of network updates per iteration $N_\phi$.}}
\Result Fine-tuned particles $\{Y_{{k}}^{(i)}\}_{i=1}^{M}$ at termination.
\State Sample $\{X^{(j)}\}_{j=1}^{N} \sim \Ptar$, a batch of samples from the target distribution.
\State Sample $\{Y_0^{(i)}\}_{i=1}^{M} \sim \Ppre$, a batch of samples from the pre-trained distribution.
\State Compute the target tail statistic $\mathrm{CVaR}_\alpha^{ \widehat{\Ptar} ,g}$ via Algorithm \ref{alg:tail:NS} using  $\{X^{(j)}\}_{j=1}^{N}$.
\State $k\gets 0$, \ $\mathcal K\gets 2\epsilon$
\While{$\mathcal K>\epsilon$}
  \State $\phi^*_k\gets\displaystyle\arg\max_{\phi\in\Gamma_L^{\mathrm{NN}}}\Big\{\tfrac1M\textstyle\sum_i\phi(Y_k^{(i)})-\log\big(\tfrac1N\sum_j e^{\phi(X^{(j)})}\big)\Big\}$ \Comment{$N_\phi$ steps}
  \State $(\VaR_\alpha^{ \widehat{Q}{_k} ,g},\VaRbar_\alpha^{ \widehat{Q}{_k} ,g},\mathrm{CVaR}_\alpha^{ \widehat{Q}{_k} ,g})$ via Algorithm \ref{alg:tail:NS} using  $\{Y^{(i)}_k\}_{i=1}^{M}$.
  \State $y_k\gets \VaR_\alpha^{ \widehat{Q}{_k} ,g}$ or $\VaRbar_\alpha^{ \widehat{Q}{_k} ,g}$
   \State $\Delta_k\gets \mathrm{CVaR}_\alpha^{ \widehat{\Ptar} ,g}-\mathrm{CVaR}_\alpha^{ \widehat{Q}{_k} ,g}$
  \For{$i=1,\dots,M$}
     \State  $v_k^{(i)}\gets -\nabla_x\phi^*_k(Y_k^{(i)})+\dfrac{2\lambda}{1-\alpha}\,\Delta_k\,\dfrac{Y_k^{(i)}}{\|Y_k^{(i)}\|}\,\ind\{\|Y_k^{(i)}\|>y_k\}$ {\Comment{\eqref{Eq:b:NS}}}
  \EndFor
  \State $Y_{k+1}^{(i)}\gets Y_k^{(i)}+\Delta t\,v_k^{(i)}$;\quad $\mathcal K\gets\tfrac1{2M}\sum_i\|v_k^{(i)}\|^2$
  \State $k\gets k+1$
\EndWhile
\State \Return $\{Y_{k}^{(i)}\}_{i=1}^M$
\end{algorithmic}
\end{algorithm}

{As for CVaR-GPA, the only target-dependent inputs of Algorithm \ref{alg:gpa:NS} are the target samples $\{X^{(j)}\}$, so the algorithm is tail-agnostic.}

\subsection{Proofs for the results concerning CVaR-penalized Wasserstein gradient flows with Clarke's generalized subgradients}
\label{App:Clarke:Proofs}

This subsection collects the proofs of the results stated above, together with the auxiliary lemmas they require.

\begin{restated}{Lemma \ref{lem:clarke:finite}}
    Let $J(Q)$ be locally Lipschitz continuous with respect to the Wasserstein-$1$ metric on $\PR$. Then for any admissible perturbation $\rho$, the Clarke generalized directional derivative $J^\circ(Q;\rho)$ exists and is finite.
\end{restated}
\begin{proof}
Let $J(Q)$ be locally Lipschitz continuous with respect to the Wasserstein-$1$ metric on $\PR$. Let $\Lip(J, Q)$ denote the Lipschitz constant of the functional $J$ in a neighborhood of $Q$. Then for an admissible perturbation $\rho$, we have
    \begin{align*}
  |J^\circ(Q;\rho)|&=\left|\limsup_{Q'\to Q,\ t\downarrow0}\frac{J(Q'+t\rho)-J(Q')}{t}\right| \\
  &\leq \limsup_{Q'\to Q,\ t\downarrow0} \Lip(J,Q) \frac{\mathcal{W}_1(Q'+t\rho,Q')}{t} \\
  &= \limsup_{Q'\to Q,\ t\downarrow0} \Lip(J,Q) \frac{\sup_{\psi \in \Gamma_1} \int \psi\, d(Q'+t\rho-Q') }{t}
  \\
  &= \Lip(J,Q)\sup_{\psi \in \Gamma_1} \int \psi\, d \rho
    \\
  &= \Lip(J,Q)\sup_{\psi \in \Gamma_1} \int [\psi - \psi(0)] d\rho  \quad (\text{since} \int d \rho = 0)
     \\
  &\leq  \Lip(J,Q)\sup_{\psi \in \Gamma_1} \int [\psi - \psi(0)] d|\rho|      \\
  &\leq  \Lip(J,Q) \int \|x\| d|\rho| <\infty,
\end{align*}
where the second  equality follows from the Kantorovich--Rubinstein duality \citep{kantorovich1958space},
$$\sup_{\psi \in \Gamma_1} \int \psi\, d(Q_1-Q_2) = \mathcal{W}_1(Q_1,Q_2).$$
\end{proof}

\begin{restated}{Proposition \ref{lem:W1-lipschitz}.}
Assume that $g$ is $\Lip(g)$-Lipschitz continuous and Assumption \ref{Assump} holds. Then the functional $Q \mapsto \Fg(Q;\Ptar)$ is locally Lipschitz continuous on $\mathcal{Q}^g$ with respect to $\mathcal{W}_1$.
\end{restated}
\begin{proof}
 Fix $y\in\R$. Since $t\mapsto t^+$ is $1$-Lipschitz and $g$ is $\Lip(g)$-Lipschitz, $x\mapsto (g(x)-y)^+$ is $\Lip(g)$-Lipschitz, uniformly in $y$. Thus we have
\begin{align*}
  \big|F^{Q_1,g}_\alpha(y) - F^{Q_2,g}_\alpha(y)\big|
&=  \frac{1}{1-\alpha} \big|\E_{Q_1}[(g-y)^+] - \E_{Q_2}[(g-y)^+]\big| \\
&\leq  \frac{1}{1-\alpha} \sup_{\phi\in\Gamma_{\Lip(g)}} \int \phi\, d(Q_1-Q_2)
 \\
&= \frac{\Lip(g) }{1-\alpha} \sup_{\psi\in\Gamma_{1}} \int \psi\, d(Q_1-Q_2)\\
&= \frac{\Lip(g) }{1-\alpha}  \mathcal{W}_1(Q_1,Q_2),
\end{align*}
for every $y$, uniformly. Here the last equality follows from the Kantorovich--Rubinstein duality \citep{kantorovich1958space}
$$\sup_{\psi \in \Gamma_1} \int \psi\, d(Q_1-Q_2) = \mathcal{W}_1(Q_1,Q_2).$$ Thus it follows that
\begin{align}
\label{Eq:Loc:Lip:CVaR}
\big|\CVaR_\alpha^{Q_1,g} - \CVaR_\alpha^{Q_2,g} \big|
&=\big|\inf_{y \in \R} F^{Q_1,g}_\alpha(y) - \inf_{y \in \R} F^{Q_2,g}_\alpha(y) \big|\nonumber \\ &\leq \sup_{y \in \R} \big|F^{Q_1,g}_\alpha(y) - F^{Q_2,g}_\alpha(y)\big| \nonumber \\
&\leq \frac{\Lip(g) }{1-\alpha}  \mathcal{W}_1(Q_1,Q_2).
\end{align}
Fix $Q_0 \in \mathcal{Q}^g$ and $R>0$, and let $Q_1,Q_2 \in B_R(Q_0)$, so that $\mathcal{W}_1(Q_1,Q_0)\le R$ and $\mathcal{W}_1(Q_2,Q_0)\le R$.
By \eqref{Eq:Loc:Lip:CVaR}, for $i=1,2$
\begin{align}
\big|\Delta C(Q_i;\Ptar) - \Delta C(Q_0;\Ptar)\big| \;\le\; \frac{\Lip(g)}{1-\alpha}\,\mathcal{W}_1(Q_i,Q_0) \;\le\; \frac{\Lip(g)}{1-\alpha}R.
\end{align}
Hence
\begin{align*}
|\Delta C(Q_i;\Ptar)| &\le |\Delta C(Q_0;\Ptar)| + \big|\Delta C(Q_i;\Ptar)-\Delta C(Q_0;\Ptar)\big| \\
&\leq |\Delta C(Q_0;\Ptar)| + \frac{\Lip(g)}{1-\alpha}R
:=M_R
\end{align*}
Then it readily follows that
\begin{align*}
& \big|\left(\Delta C(Q_1;\Ptar)\right)^2 - \left( \Delta C(Q_2;\Ptar)\right)^2 \big| \\ =& \big|\Delta C(Q_1;\Ptar)+\Delta C(Q_2;\Ptar)\big| \, \big|\Delta C(Q_1;\Ptar)-\Delta C(Q_2;\Ptar)\big|\\
\le & \left(\big|\Delta C(Q_1;\Ptar)\big|+\big|\Delta C(Q_2;\Ptar)\big| \right) \, \big|\Delta C(Q_1;\Ptar)-\Delta C(Q_2;\Ptar)\big|
\\ \leq&  2 M_R \, \big|\Delta C(Q_1;\Ptar)-\Delta C(Q_2;\Ptar)\big|
\\ \leq&  2 M_R \dfrac{\Lip(g)}{1-\alpha}\, \mathcal{W}_1(Q_1,Q_2) ; \quad \text{by \eqref{Eq:Loc:Lip:CVaR}}.
\end{align*}

Let $A^{Q_1}(\phi):=\E_{Q_1}[\phi]-\log\E_{\Ptar}[e^\phi]$ and $A^{Q_2}(\phi):=\E_{Q_2}[\phi]-\log\E_{\Ptar}[e^\phi]$ for $\phi\in\Gamma_L$. Then, by the variational representation \eqref{equation:Variational_KL_Lip}, we have
\begin{align*}
\LipKL(Q_1\|\Ptar) - \LipKL(Q_2\|\Ptar) &=   \sup_{\phi \in \Gamma_L} A^{Q_1}(\phi) -  \sup_{\phi \in \Gamma_L} A^{Q_2}(\phi) \\
   &\leq  \sup_{\phi \in \Gamma_L} \big\{  A^{Q_1}(\phi) -   A^{Q_2}(\phi) \big\} \\
&=   \sup_{\phi\in\Gamma_L}\Big\{\E_{Q_1}[\phi]-\E_{Q_2}[\phi]\Big\} \\
&=  \sup_{\phi\in\Gamma_L} \int \phi\, d(Q_1-Q_2)
 \\
&= L \sup_{\psi\in\Gamma_1} \int \psi\, d(Q_1-Q_2) \\
&= L \mathcal{W}_1(Q_1,Q_2),
\end{align*}
where the last equality follows from the Kantorovich--Rubinstein duality. A similar argument yields $\LipKL(Q_2\|\Ptar) - \LipKL(Q_1\|\Ptar) \leq  L \mathcal{W}_1(Q_1,Q_2)$. Hence we obtain
\begin{align}
    \label{Eq:Loc:Lip:KL}
    \big|\LipKL(Q_1\|\Ptar) - \LipKL(Q_2\|\Ptar)\big| \le L\, \mathcal{W}_1(Q_1,Q_2).
\end{align}
Then we have
\begin{align*}
& \big|\Fg(Q_1;\Ptar)-\Fg(Q_2;\Ptar)\big| \\
\le& \big|\LipKL(Q_1\|\Ptar)-\LipKL(Q_2\|\Ptar)\big| + \lambda\big|(\Delta C(Q_1;\Ptar))^2-(\Delta C(Q_2;\Ptar))^2\big| \\
\le& L\,\mathcal{W}_1(Q_1,Q_2) + 2\lambda M_R\frac{\Lip(g)}{1-\alpha}\,\mathcal{W}_1(Q_1,Q_2) \\
=& K_R\,\mathcal{W}_1(Q_1,Q_2)
\end{align*}
where $K_R = L + 2\lambda M_R\frac{\Lip(g)}{1-\alpha}$.
Since $Q_0\in\mathcal{Q}^g$ and $R>0$ were arbitrary, every $Q\in\mathcal{Q}^g$ has a $\mathcal{W}_1$-neighborhood on which $\Fg(\cdot\,;\Ptar)$ is Lipschitz, i.e., $\Fg(\cdot\,;\Ptar)$ is locally Lipschitz continuous with respect to $\mathcal{W}_1$.
\end{proof}

\begin{restated}{Lemma \ref{Lem:interval:subgradient}}
 Let $J(Q)$ be locally Lipschitz continuous with respect to the Wasserstein-$1$ metric on $\PR$. Define the
 right/left one-sided directional derivatives of $J$ as:
\begin{align*}
    D^\pm_\rho  J:=\lim_{\epsilon\to0^\pm} \frac{1}{\epsilon}\big(J(Q^\epsilon)-J(Q)\big).
\end{align*}
Assume that for every admissible perturbation $\rho$ at $Q$, the one-sided directional derivatives
$D^{\pm}_\rho J$ exist and $\int \xi\, d\rho$ lies in the closed interval with endpoints $D^-_\rho J$ and
$D^+_\rho J$. Then $\xi$ is a variational subgradient of $J$ at $Q$.
\end{restated}
\begin{proof}
Let $\rho$ be an admissible perturbation at $Q$. The right/left one-sided derivatives of $J$ at $Q$ are defined by
\begin{align}
D^\pm_\rho J:=\lim_{\epsilon\to0^\pm}\frac{1}{\epsilon}\big(J(Q^\epsilon)-J(Q)\big).
\end{align}
Restricting the limit
superior in the definition of
$J^\circ(Q;\rho)$ to the constant family $Q' = Q$ gives
\begin{equation*}
J^\circ(Q;\rho) \;\geq\; \lim_{t \downarrow 0}
\frac{J(Q + t\rho) - J(Q)}{t} \;=\; D^+_\rho J.
\end{equation*}
Similarly, we may instead take the family
$Q'_t := Q - t\rho$, which lies in $\PR$ for all small $t > 0$ by the
admissibility of $-\rho$ and converges to $Q$ in $\mathcal{W}_1$ as
$t \downarrow 0$. Along this family
\begin{equation*}
\frac{J(Q'_t + t\rho) - J(Q'_t)}{t}
\;=\; \frac{J(Q) - J(Q - t\rho)}{t}
\;=\; \frac{J(Q + s\rho) - J(Q)}{s}\bigg|_{s = -t}
\;\xrightarrow[t \downarrow 0]{}\; D^-_\rho J,
\end{equation*}
whence $J^\circ(Q;\rho) \geq D^-_\rho J$ as well. Therefore
$J^\circ(Q;\rho) \geq \max\{D^+_\rho J,\, D^-_\rho J\} \geq \int \xi\, d\rho$, and
$\xi \in \partial J(Q)$, i.e.\ $\xi$ is a variational subgradient of $J(Q)$.
\end{proof}

We have the following auxiliary lemma.

\begin{Lemma}
\label{lemma1}
Assume that $g$ is Lipschitz continuous and $Q \in \Qg$. Let $\rho$ be right-admissible at
$Q$. Then there exists
$\epsilon_0>0$ such that
\begin{enumerate}
    \item The function $(\epsilon,y) \mapsto F^{Q^{\epsilon},g}_{\alpha}(y)$ is jointly continuous in $[0,
    \epsilon_0] \times  \R$.
    \item The set $ T := \bigcup_{\epsilon \in [0, \epsilon_0]} T(Q^{\epsilon})$ is compact.
\end{enumerate}
\end{Lemma}
\begin{proof}
(1) Denote by $\ell_{\rho}(y):= \frac{1}{1-\alpha} \int (g(x)-y)^+ \, d \rho(x)$ for any $y \in \R$ and $\|\rho\| = \int d|\rho|$. For all $y \in \R$, $|\ell_{\rho}(y)| \leq \frac{1}{1-\alpha} \left[ \int g \,d |\rho| + |y|\, \|\rho\| \right] < \infty$ since $\rho$ is right-admissible and $\int g\,d|\rho|<\infty$ by \eqref{Eq:g:integrability}.  Let $\epsilon_0$ be a given threshold and let $(\epsilon,y) \in [0,\epsilon_0] \times \R$ be arbitrary. Let $\tilde \epsilon > 0$ and $\delta < \frac{\tilde \epsilon}{E_0}$ where $E_0:=\frac{2-\alpha}{1-\alpha}  + |\ell_{\rho}(y)|   + \frac{\epsilon_0   \,\| \rho\|}{1-\alpha}$. For $\epsilon' \in [0,\epsilon_0]$ let $\| (y ,\epsilon) -(y',\epsilon') \| < \delta$. Since the function $t\mapsto t^+$ is $1$-Lipschitz continuous, $|(g(x)-y)^+ -(g(x)-y')^+| \leq |y-y'|$ for all $x \in \R^d$. Hence, for any signed measure $\pi$, we have the following
\begin{align}
\label{Eq:App:lemma}
 \left|   \int \big[(g(x)-y)^+ -(g(x)-y')^+\big] d \pi \right| \leq  \int|y-y'| d |\pi| =  |y-y'| \int d|\pi|
\end{align}
\begin{align*}
& |F_{\alpha}^{ Q^{\epsilon},g}(y)-F_{\alpha}^{ Q^{\epsilon'},g}(y')|  \\
=&  \left| y-y' + \frac{1}{1-\alpha} \int \left[  (g(x)-y)^+ -(g(x)-y')^+\right] d Q + \epsilon \ell_{\rho} (y)- \epsilon' \ell_{\rho} (y')  \right|  \\
  \leq&   |y-y'|+ \frac{1}{1-\alpha}   |y-y'|  + \left| \epsilon   \ell_{\rho}(y )- \epsilon'  \ell_{\rho}(y')  \right|; \,\text{(from \eqref{Eq:App:lemma} with  $\pi=Q$)} \\
   =& \frac{2-\alpha}{1-\alpha}|y-y'| + \left| \epsilon   \ell_{\rho}(y )- \epsilon'  \ell_{\rho}(y')  \right|  \\
 \leq&  \frac{2-\alpha}{1-\alpha}|y-y'| + \left| \epsilon   \ell_{\rho}(y )- \epsilon'  \ell_{\rho}(y)  \right| +  \left| \epsilon '  \ell_{\rho}(y )- \epsilon'  \ell_{\rho}(y')  \right|  \\
 \leq&  \frac{2-\alpha}{1-\alpha}|y-y'| + |\ell_{\rho}(y)| \, \left| \epsilon - \epsilon'    \right| + \epsilon '  \left|   \ell_{\rho}(y )-  \ell_{\rho}(y')  \right|  \\
  \leq&  \frac{2-\alpha}{1-\alpha}|y-y'| + |\ell_{\rho}(y)| \, \left| \epsilon - \epsilon'    \right| + \epsilon_0  \left|   \ell_{\rho}(y )-  \ell_{\rho}(y')  \right|  \\
    \leq& \frac{2-\alpha}{1-\alpha}|y-y'| + |\ell_{\rho}(y)| \, \left| \epsilon - \epsilon'    \right| + \frac{\epsilon_0   \,\| \rho\|}{1-\alpha} \left|y-  y' \right|; \, \text{(from \eqref{Eq:App:lemma} with  $\pi=\rho$)}  \\
      \leq& \left(\frac{2-\alpha}{1-\alpha}  + |\ell_{\rho}(y)|   + \frac{\epsilon_0   \,\| \rho\|}{1-\alpha} \right) \delta = E_0 \delta < \tilde \epsilon.
\end{align*}
Hence, $F^{Q^{\epsilon},g}_{\alpha}(y)$ is jointly continuous in $[0, \epsilon_0] \times  \R$.

(2) For any $c \in \R$, $|\Psi^{Q^\epsilon,g}(c)-\Psi^{Q,g}(c)|\le\epsilon\|\rho\|$. In the trivial case $\|\rho\|=0$, $T=T(Q)$, and the result readily follows. Assume $\|\rho\|>0$. Fix $\zeta\in(0,\tfrac12\min\{\alpha,1-\alpha\})$ and $a<b$ with $\Psi^{Q,g}(a)\le\alpha-2\zeta$, $\Psi^{Q,g}(b)\ge\alpha+2\zeta$; set $\epsilon_0=\min\{\zeta/\|\rho\|,\bar\epsilon\}$, where $\bar\epsilon>0$ is an admissibility bound ($Q^\epsilon\in\PR$ for $\epsilon\in[0,\bar\epsilon]$). Then for all \(\epsilon\in[0,\epsilon_0]\) we have
\begin{align}
    \Psi^{Q^\epsilon,g }(a)\le \alpha-\zeta<\alpha, \qquad
\Psi^{Q^\epsilon,g }(b)\ge \alpha+\zeta>\alpha.
\end{align}
Thus, by definition of $\mathrm{VaR}_\alpha ^{Q^\epsilon,g}$ and $\overline{\mathrm{VaR}}_\alpha ^{Q^\epsilon,g}$ we have that $a \leq \mathrm{VaR}_\alpha ^{Q^\epsilon,g} \leq \overline{\mathrm{VaR}}_\alpha ^{Q^\epsilon,g} \leq b$. Therefore, $T(Q^{\epsilon}) =[\mathrm{VaR}_\alpha ^{Q^\epsilon,g},\overline{\mathrm{VaR}}_\alpha ^{Q^\epsilon,g}] \subseteq [a,b]$ for every $\epsilon\in[0,\epsilon_0]$. Thus $ T := \bigcup_{\epsilon \in [0, \epsilon_0]} T(Q^{\epsilon})$ is a bounded set.

Let $y_n \in T$ be a sequence that converges to $\hat y$. For each $n$, there exists $\epsilon_n \in [0, \epsilon_0]$ such that $y_n \in T(Q^{\epsilon_n})$. By the compactness of the interval $[0, \epsilon_0]$, there exists a subsequence $\epsilon_{n_k}$ that converges to $\epsilon^* \in [0, \epsilon_0]$. Since $(\epsilon,y) \mapsto F_{\alpha}^{Q^{\epsilon},g}(y)$ is jointly continuous in $[0, \epsilon_0] \times \R$ for any threshold $\epsilon_0$, we have the following for any $z \in \R$
\begin{align}
   F_{\alpha}^{Q^{\epsilon^*},g}(\hat{y}) = \lim_{k \to \infty}   F_{\alpha}^{Q^{\epsilon_{n_k}},g}(y_{n_k}) \leq \lim_{k \to \infty}F_{\alpha}^{Q^{\epsilon_{n_k}},g}(z)
   = F_{\alpha}^{Q^{\epsilon^{*}},g}(z) .
\end{align}
That is, $\hat y \in T(Q^{\epsilon^*}) \subseteq T$. Thus, $T$ is closed, which proves the result.
\end{proof}

\begin{restated}{Theorem \ref{Thm:VD:CVaR}}
Assume that $g$ is Lipschitz continuous and $Q$ has finite first moment (hence $Q \in \Qg$). For $y\in\R$, define
$h^y(x):=\frac{(g(x)-y)^+}{1-\alpha}$.
Then we have the following:
\begin{enumerate}
\item \emph{(Right derivative)} For every right-admissible perturbation $\rho$, the right directional
derivative exists and is given by
\begin{align}
D^+_\rho \CVaR_\alpha^{Q,g}=\lim_{\epsilon\to0^+}\frac{\CVaR_\alpha^{Q^\epsilon,g}-\CVaR_\alpha^{Q,g}}
{\epsilon}=\min_{y\in T(Q)}\int h^y\,d\rho.
\tag{\ref{eq:dir}}\end{align}
\item \emph{(Left derivative)}  For every left-admissible perturbation $\rho$, the left directional derivative
exists and is given by
\begin{align}
D^-_\rho \CVaR_\alpha^{Q,g} =\lim_{\epsilon\to0^-}\frac{\CVaR_\alpha^{Q^\epsilon,g}-\CVaR_\alpha^{Q,g}}
{\epsilon}=\max_{y\in T(Q)}\int h^y\,d\rho.
\tag{\ref{eq:left}}\end{align}
\item \emph{(Variational subgradients)} Let $y \in T(Q)$. Then, for any admissible $\rho$, $\int h^y\,d\rho$
lies in $[D^+_\rho\CVaR_\alpha^{Q,g},D^-_\rho\CVaR_\alpha^{Q,g}]$. Consequently, for each $y\in T(Q)$, $h^y$ is
a variational subgradient of $\CVaR_\alpha^{Q,g}$.
\item \emph{(Two-sided derivative)} The two-sided (Gateaux) derivative for an admissible $\rho$ exists if and
only if $y\mapsto\int h^y\,d\rho$ is constant on $T(Q)$; this holds for every admissible $\rho$ if and only if
$T(Q)$ is a singleton set, i.e.\ $\VaR_\alpha^{Q,g}=\VaRbar_\alpha^{Q,g}$, in which case $\CVaR_\alpha^{Q,g}$
is Gateaux differentiable at $Q$ with the first variation
\begin{equation*}
\frac{\delta\CVaR_\alpha^{Q,g}}{\delta Q}=h^{\VaR_\alpha^{Q,g}}=\frac{(g-\VaR_\alpha^{Q,g})^+}{1-\alpha}.
\end{equation*}
\end{enumerate}
\end{restated}
\begin{proof}
For any perturbation $\rho$, we define $ \ell_{\rho}(y):=\tfrac{1}{1-\alpha}\int(g-y)^+d\rho=\int h^y\,d\rho$.

(1)  Let $\rho$ be right-admissible and let $\epsilon_0>0$ be given by Lemma \ref{lemma1}. From $Q^\epsilon=Q+\epsilon\rho$, we have
\begin{align}\label{eq:split}
F^{Q^\epsilon,g}_\alpha(y)=F^{Q,g}_\alpha(y)+\epsilon\,\ell_{\rho}(y).
\end{align}
Thus, it follows that
\begin{align}
\CVaR_\alpha^{Q^\epsilon,g} &\le \inf_{y\in T(Q)}F^{Q^\epsilon,g}_\alpha(y) \nonumber \\
&= \inf_{y\in T(Q)}\{F^{Q,g}_\alpha(y)+\epsilon\ell_{\rho}(y)\}\nonumber \\
&= \CVaR_\alpha^{Q,g} + \epsilon \min_{y\in T(Q)}\ell_{\rho}(y).
\end{align}
Hence, we have
\begin{align}
\label{equation_VD_lim_sup}
    \limsup_{\epsilon\to0^+} \frac{(\CVaR_\alpha^{Q^\epsilon,g}-\CVaR_\alpha^{Q,g})}{\epsilon} \le\min_{y\in T(Q)}\ell_{\rho}(y)=\min_{y\in T(Q)}\int h^y\,d\rho.
\end{align}
Pick $y_\epsilon\in T(Q^\epsilon)$. Using \eqref{eq:split} and
$F^{Q,g}_\alpha(y_\epsilon)\ge\CVaR_\alpha^{Q,g}$, we have
\begin{align}\label{eq:lb}
\frac{\CVaR_\alpha^{Q^\epsilon,g}-\CVaR_\alpha^{Q,g}}{\epsilon}=\frac{F^{Q,g}_\alpha(y_\epsilon)-
\CVaR_\alpha^{Q,g}}{\epsilon}+\ell_{\rho}(y_\epsilon)\ge\ell_{\rho}(y_\epsilon).
\end{align}
Fix $\epsilon_n\downarrow0$; by compactness of $T$ pass to a subsequence with $y_{\epsilon_n}\to\bar y\in T$.
Since $y\mapsto\ell_{\rho}(y)$ is continuous on $\R$, there exists $M>0$ such that  $\sup_{y \in T}|\ell_{\rho}(y)|\le M$. Thus, for any $y^*\in\R$, we have
\begin{align*}
F^{Q,g}_\alpha(y_{\epsilon_n})
&=F^{Q^{\epsilon_n},g}_\alpha(y_{\epsilon_n})-\epsilon_n\ell_{\rho}(y_{\epsilon_n}) \\
&\le F^{Q^{\epsilon_n},g}_\alpha(y^*)-\epsilon_n\ell_{\rho}(y_{\epsilon_n})\\
&=F^{Q,g}_\alpha(y^*)+\epsilon_n \ell_{\rho}(y^*)-\epsilon_n\ell_{\rho}(y_{\epsilon_n})\\
&\le F^{Q,g}_\alpha(y^*) + M \epsilon_n +  | \ell_{\rho}(y^*)|\,  \epsilon_n.
\end{align*}
Letting $\epsilon_n\to0$, and using the continuity of $F^{Q,g}_\alpha(\cdot)$, we obtain that $\lim_{n \to
\infty }F^{Q,g}_\alpha(y_{\epsilon_n}) = F^{Q,g}_\alpha(\bar y) \le F^{Q,g}_\alpha(y^*) $ for all $y^* \in \R$.
Thus $\bar y \in T(Q)$.

Thus, by \eqref{eq:lb}, and the continuity of $\ell_{\rho}(\cdot)$, we have
     \begin{align}
       \label{equation_VD_lim_inf}
 \liminf_{n \to \infty }
\frac{\mathrm{CVaR}_{\alpha}^{Q^{\epsilon_n},g} -\mathrm{CVaR}_{\alpha}^{Q,g}}{\epsilon_n}
\geq  \lim _{n \to \infty }  \ell_{\rho}(y_{\epsilon_n})  =
\ell_{\rho}(\bar y) \geq \min_{y\in T(Q)}\ell_{\rho}(y)=\min_{y\in T(Q)}\int h^y\,d\rho.
     \end{align}
 Thus, we can conclude using \eqref{equation_VD_lim_sup} and \eqref{equation_VD_lim_inf} that for any sequence
 $\epsilon_n \to 0^+$ in the interval  $(0,\epsilon_0]$, there exists a subsequence $ \epsilon_{n_k}$ such
 that
\begin{align}
\lim   _{ \epsilon_{n_k} \to 0^+} \frac{\mathrm{CVaR}_{\alpha}^{Q^{ \epsilon_{n_k} },g} -
\mathrm{CVaR}_{\alpha}^{Q,g}}{ \epsilon_{n_k} }   =\min_{y\in T(Q)}\int h^y\,d\rho.
\end{align}
Hence,
 \begin{align}
\lim_{ \epsilon \to 0^+} \frac{\mathrm{CVaR}_{\alpha}^{Q^\epsilon,g} -\mathrm{CVaR}_{\alpha}^{Q,g}}{\epsilon}
=\min_{y\in T(Q)}\int h^y\,d\rho,
\end{align}
proving \eqref{eq:dir}.

(2) Let $\rho$ be left-admissible. Let $\epsilon=-\delta$, $\delta>0$. Thus $Q^{\epsilon}= Q+\epsilon \rho = Q-
\delta \rho$. We have
 \begin{align}
D^-_\rho \CVaR_{\alpha}^{Q,g}&= \lim_{ \epsilon \to 0^-} \frac{\mathrm{CVaR}_{\alpha}^{Q^\epsilon,g} -\mathrm{CVaR}_{\alpha}^{Q,g}}
{\epsilon} \nonumber \\
&= - \lim_{ \delta \to 0^+} \frac{\mathrm{CVaR}_{\alpha}^{Q-\delta \rho,g} -\mathrm{CVaR}_{\alpha}^{Q,g}}
{\delta} \nonumber \\
&=-\min_{y\in T(Q)}\int h^y\,d(-\rho)
\nonumber \\
&=\max_{y\in T(Q)}\int h^y\,d\rho.
\end{align}

(3) Since $y \mapsto \ell_\rho(y)$ is continuous and $T(Q)$ is a compact interval,  $\int h^y d
\rho=\ell_\rho(y)$ ranges over the interval between $D^+_\rho \mathrm{CVaR}_{\alpha}^{Q,g}$ and $D^-_\rho
\mathrm{CVaR}_{\alpha}^{Q,g}$, attaining both extremes. Thus, by Lemma \ref{Lem:interval:subgradient}, it
follows that the family of functions $\{h^y; y \in T(Q) \}$ are variational subgradients of
$\mathrm{CVaR}_{\alpha}^{Q,g}$.

(4) Let $\rho$ be a perturbation such that $\rho$ is admissible at $Q$. By parts (1) and (2), both one-sided
directional derivatives exist and are given by
\begin{equation*}
D^+_\rho \CVaR_\alpha^{Q,g}=\min_{y\in T(Q)}\ell_{\rho}(y),
\qquad
D^-_\rho \CVaR_\alpha^{Q,g}=\max_{y\in T(Q)}\ell_{\rho}(y).
\end{equation*}
The two-sided (Gateaux) derivative in direction $\rho$ exists if and only if these one-sided derivatives
coincide, that is, if and only if
\begin{equation*}
\min_{y\in T(Q)}\ell_{\rho}(y)=\max_{y\in T(Q)}\ell_{\rho}(y).
\end{equation*}

That is, the two-sided (Gateaux) derivative in direction $\rho$ exists if and only if  $y \mapsto \ell_{\rho}(y)$ is constant on $T(Q)$.

Assume that $T(Q)$ is a singleton set, i.e.\ $\VaR_\alpha^{Q,g}=\VaRbar_\alpha^{Q,g}$, so that
$T(Q)=\{\VaR_\alpha^{Q,g}\}$. Then $y \mapsto \ell_{\rho}(y)$ is constant on $T(Q)$ for every admissible $\rho$, and
\begin{equation*}
D^+_\rho \CVaR_\alpha^{Q,g}=D^-_\rho \CVaR_\alpha^{Q,g}=\ell_{\rho}\big(\VaR_\alpha^{Q,g}\big)=\int
h^{\VaR_\alpha^{Q,g}}\,d\rho.
\end{equation*}
The two-sided derivative therefore exists for every admissible perturbation; that is, $\CVaR_\alpha^{Q,g}$ is
Gateaux differentiable at $Q$ with first variation
\begin{equation*}
\frac{\delta\CVaR_\alpha^{Q,g}}{\delta Q}=h^{\VaR_\alpha^{Q,g}}=\frac{(g-\VaR_\alpha^{Q,g})^+}{1-\alpha}.
\end{equation*}
Conversely, assume that $\VaR_\alpha^{Q,g}<\VaRbar_\alpha^{Q,g}$. Define  $k(x)=\big(g(x)-
\VaR_\alpha^{Q,g}\big)^+-\big(g(x)-\VaRbar_\alpha^{Q,g}\big)^+$. Define $\varphi_k:=k-\E_Q[k]$. Since $0\le
k\le\VaRbar_\alpha^{Q,g}-\VaR_\alpha^{Q,g}$, both $k$ and $\varphi_k$ are bounded. Define the perturbation
measure $\rho:=\varphi_k\, Q$.  We verify that $\rho$ is admissible at $Q$. First,
$\int d \rho=\int\varphi_k\,dQ=\E_Q[k]-\E_Q[k]=0$ and $\int d|\rho|=\int|\varphi_k|\,dQ\le\|\varphi_k\|_{\infty}<\infty$. Second, for $|\epsilon|\le\|\varphi_k\|_{\infty}^{-1}$ we
have $1+\epsilon\varphi_k\ge0$ point-wise, so $Q^{\epsilon}=Q+\epsilon\rho=(1+\epsilon\varphi_k)\,Q$
is a non-negative measure of total mass $1$, i.e.\ $Q^{\epsilon}\in\PR$ for all
$|\epsilon|\le\|\varphi_k\|_{\infty}^{-1}$.
Third, we have
\begin{align}
\int \|x\|\,d|\rho|
= \int \|x\|\,|\varphi_k|\,dQ \leq \|\varphi_k\|_{\infty} \int \|x\|\,dQ(x),
\end{align}
and since $\varphi_k$ is bounded and $Q$ has a finite first moment, we have $\int \|x\|\,d|\rho|  < \infty$.
Hence,  $\rho$ is admissible.
 Since $\rho=\varphi_k Q$,
\begin{align*}
\ell_{\rho}\big(\VaR_\alpha^{Q,g}\big)-\ell_{\rho}\big(\VaRbar_\alpha^{Q,g}\big)
&=\frac{1}{1-\alpha}\int\Big[\big(g-\VaR_\alpha^{Q,g}\big)^+-\big(g-\VaRbar_\alpha^{Q,g}\big)^+\Big]\,d\rho \\
&=\frac{1}{1-\alpha}\int k\,\varphi_k\,dQ\\
&=\frac{1}{1-\alpha}\,\E_Q\big[k\,(k-\E_Q[k])\big]\\
&=\frac{\operatorname{Var}_Q(k)}{1-\alpha},
\end{align*}
where  $\operatorname{Var}_Q(k)$ denotes the variance of $k$ under the probability measure $Q$.
Since $\VaR_\alpha^{Q,g}<\VaRbar_\alpha^{Q,g}$, the CDF $\Psi^{Q,g}$ satisfies $\Psi^{Q,g}(c)=\alpha$ for all
$c\in[\VaR_\alpha^{Q,g},\VaRbar_\alpha^{Q,g})$. Consequently, $Q(\{k=0\})= Q(g\le\VaR_\alpha^{Q,g})=\Psi^{Q,g}
(\VaR_\alpha^{Q,g})=\alpha$ and $Q(\{k>0\})=1-\alpha>0$. Since the probability of $k>0$ and the probability of
$k=0$ are strictly positive under the probability measure $Q$, $\operatorname{Var}_Q(k) >0$. Thus we have
\begin{align}
\ell_{\rho}\big(\VaR_\alpha^{Q,g}\big)-\ell_{\rho}\big(\VaRbar_\alpha^{Q,g}\big)
=\;\frac{\operatorname{Var}_Q(k)}{1-\alpha}>0.
\end{align}
Thus $\ell_{\rho}$ is non-constant on $T(Q)$, and hence
\begin{equation*}
D^+_\rho \CVaR_\alpha^{Q,g}=\min_{y\in T(Q)}\ell_{\rho}(y)<\max_{y\in T(Q)}\ell_{\rho}(y)=D^-_\rho \CVaR_\alpha^{Q,g}.
\end{equation*}

The two-sided derivative fails to exist for the perturbation $\rho$. We conclude that the two-sided derivative
exists for every admissible perturbation if and only if $T(Q)$ is a singleton set.
\end{proof}

\begin{restated}{Theorem \ref{Theorem:VD:Loss}}
Assume that $g$ is Lipschitz continuous and Assumption \ref{Assump} holds. Let $\phi^*$ denote the maximizer of
\eqref{equation:Variational_KL_Lip}. For each $y\in T(Q)$, define the potential function
\begin{align}
\Phi_Q^{y}:=\phi^*-2\lambda\,\Delta C(Q; \Ptar)\,h^y
=\phi^*-\frac{2\lambda}{1-\alpha}\,\Delta C(Q; \Ptar)\,(g-y)^+ .
\tag{\ref{eq:Phi-family}}\end{align}
Then the following hold.
\begin{enumerate}
\item  For every right-admissible perturbation $\rho$, the right directional derivative exists and is given by
\begin{align}
D^+_\rho\Fg(Q;\Ptar)=\int\phi^*\,d\rho-2\lambda\,\Delta C(Q; \Ptar)\,\min_{y\in T(Q)} \int h^y d \rho.
\tag{\ref{eq:F-right}}\end{align}
For every left-admissible perturbation $\rho$, the left directional derivative exists and is given by
\begin{align}
D^-_\rho\Fg(Q;\Ptar)=\int\phi^*\,d\rho-2\lambda\,\Delta C(Q; \Ptar)\,\max_{y\in T(Q)}\int h^y d \rho.
\tag{\ref{eq:F-left}}\end{align}
\item For any admissible $\rho$, the range of $\int\Phi_Q^{y} d\rho$ for $y \in T(Q)$ is the compact interval
between $D^-_\rho\Fg(Q;\Ptar)$ and $D^+_\rho\Fg(Q;\Ptar)$. Consequently, the family of potential functions
$\{\Phi_Q^{y}; y \in T(Q) \}$ are subgradients of $\Fg(Q;\Ptar)$.
\item Let $\rho$ be admissible. The two-sided (Gateaux) derivative of $\Fg$ for the perturbation $\rho$ exists
if and only if $\Delta C(Q; \Ptar)=0$ or $y\mapsto\int h^y\,d\rho$ is constant on $T(Q)$. The two-sided
(Gateaux) derivative of $\Fg$ exists for every admissible perturbation if and only if $\Delta C(Q; \Ptar)=0$ or
$T(Q)$ is a singleton set; in that case, $\Fg(\cdot;\Ptar)$ is Gateaux differentiable at $Q$ with first
variational derivative
\begin{equation*}
\frac{\delta\Fg(Q;\Ptar)}{\delta Q}=\Phi_Q^{\VaR_\alpha^{Q,g}}.
\end{equation*}
\end{enumerate}
\end{restated}
\begin{proof}

(1) For a right-admissible perturbation $\rho$, we have that
\begin{align}
    \label{eq:RD:CVAR}
D^+_\rho (\Delta C(Q; \Ptar))^2 &=\lim_{\epsilon\to0^+}\frac{(\Delta C(Q^\epsilon; \Ptar ))^2-(\Delta C(Q;
\Ptar))^2}{\epsilon} \nonumber \\
 &=\lim_{\epsilon\to0^+}   \left[\CVaR_\alpha^{Q^\epsilon,g}+\CVaR_\alpha^{Q,g}-2
 \CVaR_\alpha^{\Ptar,g}\right] \frac{\CVaR_\alpha^{Q^\epsilon,g}-\CVaR_\alpha^{Q,g}}{\epsilon} \nonumber \\
 &= -2 \Delta C(Q; \Ptar) \,D^+_\rho \CVaR_\alpha^{Q,g}   \nonumber \\
 &= - 2 \Delta C(Q; \Ptar) \min_{y\in T(Q)}\int h^y\,d\rho.
\end{align}

We recall that the variational derivative of the divergence $\LipKL(Q\|\Ptar)$ is given by
\begin{align}
  \frac{\delta \LipKL(Q\|\Ptar)}{\delta Q}(x)
=\phi^*(x),
  \label{eq:first-variation_dup}
\end{align}
where $\phi^*$ is the maximizer of \eqref{equation:Variational_KL_Lip}. That is, for any right-admissible
$\rho$, $D^+_\rho \LipKL(Q\|\Ptar) = \int \phi^* d \rho$. Since $D^+_\rho\Fg(Q;\Ptar)= D^+_\rho
\LipKL(Q\|\Ptar) + \lambda D^+_\rho (\Delta C(Q; \Ptar))^2$, we readily obtain \eqref{eq:F-right}.

Now let $\rho$ be a left-admissible perturbation. Then we have $D^-_\rho \LipKL(Q\|\Ptar) = \int \phi^* d \rho$
and $D^-_\rho (\Delta C(Q; \Ptar))^2 = -2 \Delta C(Q; \Ptar) \, \max_{y\in T(Q)} \int h^y\,d\rho$ (similar to
the derivation of \eqref{eq:RD:CVAR}). Since $D^-_\rho\Fg(Q;\Ptar)= D^-_\rho \LipKL(Q\|\Ptar) + \lambda D^-_\rho (\Delta C(Q; \Ptar))^2$, we readily obtain \eqref{eq:F-left}.

(2) Since $y \mapsto \int \Phi^y_Q \,d\rho$ is continuous and $T(Q)$ is a compact interval,
$\int\Phi_Q^y\,d\rho$ ranges over the interval between $D^+_\rho \Fg (Q; \Ptar)$ and $D^-_\rho \Fg (Q; \Ptar)$,
attaining both extremes. Thus, by Lemma \ref{Lem:interval:subgradient}, it follows that the family of potential
functions $\{\Phi_Q^{y}; y \in T(Q) \}$ are subgradients of $\Fg(Q;\Ptar)$.

(3) For any admissible perturbation $\rho$, we have
\begin{equation*}
D^-_\rho\Fg(Q;\Ptar)-D^+_\rho\Fg(Q;\Ptar)
=2\lambda\,\Delta C(Q; \Ptar)\Big(\min_{y\in T(Q)}\int h^y\,d\rho-\max_{y\in T(Q)}\int h^y\,d\rho\Big).
\end{equation*}

The two-sided derivative for perturbation $\rho$ exists if and only if $\Delta C(Q; \Ptar)=0$ or $\int h^y
d\rho$ is constant on $T(Q)$. As proven in Theorem \ref{Thm:VD:CVaR}, $\int h^y d\rho$ is constant on $T(Q)$  for
every admissible $\rho$, if and only if $T(Q)=\{\VaR_\alpha^{Q,g}\}$ is a singleton set. If $T(Q)$ is the
singleton set $\{\VaR_\alpha^{Q,g}\}$, we have
\begin{equation*}
D^-_\rho\Fg(Q;\Ptar)=D^+_\rho\Fg(Q;\Ptar)= \int \phi^* d \rho -2\lambda\,\Delta C(Q; \Ptar)\,\int h^{\VaR_\alpha^{Q,g}} d\rho,
\end{equation*}
for every admissible $\rho$.
That is, $\Fg(\cdot;\Ptar)$ is Gateaux differentiable at $Q$ with first variational derivative
$$\frac{\delta\Fg(Q;\Ptar)}{\delta Q}=\phi^* -2\lambda\,\Delta C(Q; \Ptar)\,
h^{\VaR_\alpha^{Q,g}}=\Phi_Q^{\VaR_\alpha^{Q,g}}.$$
\end{proof}

\begin{restated}{Theorem \ref{thm:vel}}
Assume that $g$ is Lipschitz continuous, $g\in C^1$ for $g(x) \neq 0$, and Assumption \ref{Assump} holds. For
each $y\in T(Q)$, choosing the potential function $\Phi_Q^{y}$ induces the velocity field
$v^{y}_Q=-\nabla_x\Phi_Q^{y}$, and we have
\begin{align}
v^{y}_Q(x)=
\begin{cases}
\displaystyle \underbrace{-\nabla_x\phi^*(x)}_{\text{from }\LipKL(Q\|\Ptar)}
+\underbrace{\frac{2\lambda}{1-\alpha}\,\Delta C(Q; \Ptar)\,\nabla_x g(x)}_{\text{from CVaR-penalization}}, &
g(x)>y,\\
\displaystyle \underbrace{-\nabla_x\phi^*(x)}_{\text{from }\LipKL(Q\|\Ptar)}, & g(x)\le y.
\end{cases}
\tag{\ref{eq:vel-cases}}\end{align}
If $T(Q)$ is a singleton set, the velocity field degenerates to $v^{\VaR_\alpha^{Q,g}}_Q$.
\end{restated}
\begin{proof}[Proof of Theorem \ref{thm:vel}]
The result readily follows from taking the gradient of $\Phi_Q^{y}(x)$ with respect to $x$ for $g(x) > y \geq 0$. Since
$(g(x)-  y)^+$ is not differentiable at $g(x)=y$, the velocity field is not defined at $g(x)=y$. One can adopt the
convention that $v^y(x) \in \{-\nabla_x \phi^*(x)+ s \left( \frac{2 \lambda }{(1-\alpha) } \Delta C (Q; \Ptar)  \nabla_x
g(x) \right) ; s\in [0,1]\}$ for $g(x)=y$. In our work, we define $v^y(x) = -\nabla_x \phi^*(x)$ for $g(x)=y$, yielding
\eqref{eq:vel-cases}.
\end{proof}

\begin{restated}{Corollary \ref{cor:vel-bound}}
Assume that $g$ is Lipschitz continuous, $g\in C^1$ for $g(x) \neq 0$, and Assumption \ref{Assump} holds. For every $y\in
T(Q)$, the velocity field $v^{y}_Q$ in \eqref{eq:vel-cases} satisfies, for almost every $x\in\R^d$,
\begin{align}
\big\|v^{y}_Q(x)\big\|\;\le\; L+\frac{2\lambda}{1-\alpha}\,\Lip(g)\,\big|\Delta C(Q;\Ptar)\big|.
\tag{\ref{eq:vel-bound}}\end{align}
Moreover, the CVaR discrepancy is controlled by the loss $\Fg(Q;\Ptar)$, so we have
\begin{align}
\big\|v^{y}_Q(x)\big\|\;\le\; L+\frac{2\,\Lip(g)}{1-\alpha}\,\sqrt{\lambda\,\Fg(Q;\Ptar)},
\tag{\ref{eq:vel-bound-energy}}\end{align}
for almost every $x\in\R^d$. In particular, the velocity field is bounded, and the CVaR-penalization component vanishes as
$\Fg(Q;\Ptar)\to 0$, i.e.\ as $Q\to\Ptar$.
\end{restated}
\begin{proof}[Proof of Corollary \ref{cor:vel-bound}]
By \eqref{eq:vel-cases}, for almost every $x$, we have
\begin{equation*}
  \big\|v^{y}_Q(x)\big\|\le\big\|\nabla_x\phi^*(x)\big\|
+\frac{2\lambda}{1-\alpha}\,\big|\Delta C(Q;\Ptar)\big|\,\big\|\nabla_x g(x)\big\|.
\end{equation*}
Since $\phi^*\in\Gamma_L$ is $L$-Lipschitz, $\|\nabla_x\phi^*\|\le L$ almost everywhere; since $g$ is
$\Lip(g)$-Lipschitz, $\|\nabla_x g\|\le\Lip(g)$. This proves \eqref{eq:vel-bound}.

Since $\LipKL(Q\|\Ptar)\ge 0$ we have
$$\lambda\,(\Delta C(Q;\Ptar))^{2}\le \LipKL(Q\|\Ptar)+\lambda\,(\Delta C(Q;\Ptar))^{2}=\Fg(Q;\Ptar)\, .$$
That is $|\Delta C(Q;\Ptar)|\le\sqrt{\Fg(Q;\Ptar)/\lambda}$; this combined with  \eqref{eq:vel-bound}, proves \eqref{eq:vel-bound-energy}.
\end{proof}

\begin{restated}{Corollary \ref{cor:vel:radial:SG}}
Let Assumption \ref{Assump} hold and
let $g(x)=\|x\|$ and $\E_Q[g] < \infty$. For each $y\in T(Q)$, choosing the potential function $\Phi_Q^{y}$ induces the velocity field $v^{y}_Q=-\nabla_x\Phi_Q^{y}$, and we have
\begin{align}
v^{y}_Q(x)=
\begin{cases}
\displaystyle \underbrace{-\nabla_x\phi^*(x)}_{\text{from }\LipKL(Q\|\Ptar)}
+\underbrace{\frac{2\lambda}{1-\alpha}\,\Delta C(Q; \Ptar)\,\frac{x}{\|x\|}}_{\text{from CVaR-penalization}}, & \|x\|>y,\\
\displaystyle \underbrace{-\nabla_x\phi^*(x)}_{\text{from }\LipKL(Q\|\Ptar)}, & \|x\|\le y.
\end{cases}
\tag{\ref{Eq:Velocity:SG}}\end{align}
\end{restated}
\begin{proof}[Proof of Corollary \ref{cor:vel:radial:SG}]
Immediate from  Theorem \ref{thm:vel} with $g(x)=\|x\|$ and $\nabla_x\|x\|=x/\|x\|$ for $x\neq0$.
\end{proof}

\subsection{Using CVaR vs smoothed CVaR for penalization}
\label{App:compare}

{Penalization via the unsmoothed CVaR is well posed: $\Fg$ is a divergence, its variational subgradients are characterized by Theorem \ref{Theorem:VD:Loss}, and each subgradient $\Phi^y_Q$, $y\in T(Q)$, defines a velocity field (Theorem \ref{thm:vel}). The resulting flow, however, has three drawbacks, all of which stem from the indicator $\ind_{\{g>y\}}$ in the subgradients $h^y=\frac{(g-y)^+}{1-\alpha}$.
\begin{enumerate}
\item \emph{Subgradient selection.} Whenever $T(Q)$ is not a singleton and $\Delta C(Q;\Ptar)\neq0$, the loss functional does not determine the velocity field: a level $y\in T(Q)$ must be selected, and different selections activate the CVaR velocity component on different regions. For empirical measures with $\alpha M\in\mathbb{Z}$, this is the generic situation (see \eqref{eq:orderstats}).
\item \emph{Discontinuous velocity field.} The velocity field $v^y_Q$ has a jump of magnitude $$\frac{2\lambda}{1-\alpha}|\Delta C(Q;\Ptar)|\,\|\nabla_x g\|$$ across the threshold set $\{x:g(x)=y\}$ (Remark \ref{rem:non-lip}).
\item \emph{Unstable learning.} In Algorithm \ref{alg:gpa:NS}, the velocity component $\lambda b^{(i)}_k$, with $b^{(i)}_k$ as in \eqref{Eq:b:NS}, assigns the full displacement $\frac{2\lambda}{1-\alpha}|\Delta_k|$ to the particles above $y_k$ and none to those below it. Since both $y_k$ and the particles move from one iteration to the next, a particle close to the threshold may cross it back and forth, so that the displacement it receives alternates between zero and its full magnitude.
\end{enumerate}}

{Smoothing removes all three drawbacks while preserving the focus of the penalty on the tail. Replacing $u^+$ by the softplus $r_h(u)$ replaces the indicator $\ind_{\{u>0\}}$, the derivative of $u^+$ for $u\neq0$, by its smooth surrogate $r_h'(u)=\sigma(u/h)$, which converges to $\ind_{\{u>0\}}$ as $h\downarrow0$ for every $u\neq0$. With $g=\|\cdot\|$, the two velocity components induced by the penalization are
\begin{align}
\label{Eq:compare:b}
b^{(i)}_{k}=\frac{2\Delta_k}{1-\alpha}\,\ind_{\{\|Y^{(i)}_k\|>y_k\}}\,\frac{Y^{(i)}_k}{\|Y^{(i)}_k\|},
\qquad
b^{(i)}_{k,h}=\frac{2\Delta_{k,h}}{1-\alpha}\,\sigma\Big(\frac{\|Y^{(i)}_k\|-q_{\alpha,h}^{\widehat{Q}_k,g}}{h}\Big)\,\frac{Y^{(i)}_k}{\|Y^{(i)}_k\|},
\end{align}
where $\Delta_{k,h}=\CVaR_{\alpha,h}^{\widehat{\Ptar},g}-\CVaR_{\alpha,h}^{\widehat{Q}_k,g}$. Consequently:
\begin{enumerate}
\item $F^{Q,g}_{\alpha,h}$ is strictly convex with a unique minimizer $q_{\alpha,h}^{Q,g}$ (Proposition \ref{prop:ramp}), and $\F$ has a variational derivative at every $Q\in\Qg$ (Theorem \ref{thm:vel:h}); no selection is required.
\item The jump of the CVaR velocity component across $\{x:g(x)=y\}$ is replaced by a smooth transition of width $O(h)$ around $\{x:g(x)=q_{\alpha,h}^{Q,g}\}$, and the weights $\sigma\big((\|Y^{(i)}_k\|-q_{\alpha,h}^{\widehat{Q}_k,g})/h\big)$ depend continuously on the particle positions.
\item The penalty remains focused on the tail. The weight $\sigma\big((g-q_{\alpha,h}^{Q,g})/h\big)$ is close to $1$ above the level $q_{\alpha,h}^{Q,g}$ and close to $0$ below it, and its total mass equals $1-\alpha$ by \eqref{Eq:ramp:root}. Moreover, since $r_h$ is convex and nondecreasing, $\CVaR^{Q,g}_{\alpha,h}$, viewed as a functional of the random variable $g(X)$, $X\sim Q$, is monotone, translation equivariant and convex, that is, a convex risk measure \citep{kouri2016risk}. Finally, since $0<r_h(u)-u^+=h\log(1+e^{-|u|/h})\le h\log2$, evaluating $F^{Q,g}_{\alpha,h}$ at any $y\in T(Q)$ and $F^{Q,g}_{\alpha}$ at $q_{\alpha,h}^{Q,g}$ gives
\begin{align}
\label{Eq:h:gap}
0<\CVaR^{Q,g}_{\alpha,h}-\CVaR^{Q,g}_{\alpha}\le\frac{h\log2}{1-\alpha},
\end{align}
so the smoothed CVaR converges to the CVaR as $h\downarrow0$, uniformly in $Q\in\Qg$.
\end{enumerate}}

\end{document}